\documentclass[12pt]{colt2019_arxiv}

\usepackage{microtype}
\usepackage{graphicx}
\usepackage{booktabs}

\usepackage{fancyhdr}

\usepackage{amsfonts}
\usepackage{afterpage}
\usepackage{amssymb, mathtools}
\usepackage[T1]{fontenc}
\usepackage{epsf}
\usepackage{amsfonts,amsmath}
\usepackage{psfrag,xspace}
\usepackage{color,etoolbox}
\usepackage{float}
\usepackage{algorithmic}
\usepackage{algorithm}
\usepackage{wrapfig,caption}

\usepackage{mathabx}
\usepackage[colorinlistoftodos,textwidth=1in]{todonotes}
\usepackage{hyperref}

\usepackage{times}

\DeclareMathOperator*{\argmin}{arg\,min}

\newtheorem{defn}{Definition}[section]

\newcommand{\alg}{AdaCRR\xspace}
\newcommand{\adaht}{AdaHT\xspace}
\newcommand{\tor}{\textsc{Torrent}\xspace}
\newcommand{\crr}{\textsc{CRR}\xspace}
\newcommand{\rob}{RobGrad\xspace}
\newcommand{\indep}{\rotatebox[origin=c]{90}{$\models$}}

    \newcommand{\intl}{I}
    \newcommand{\intlub}{{I}}

    \newcommand{\x}{\mathbf{x}}
    \newcommand{\tx}{\tilde{\mathbf{x}}}
    \newcommand{\tX}{\tilde{X}}
    \newcommand{\uu}{\mathbf{u}}
    \newcommand{\tT}{t'}
    \newcommand{\w}{\mathbf{w}}
    \newcommand{\dw}{\Delta\mathbf{w}}
    
    \newcommand{\tn}{\Tilde{n}}
    \newcommand{\bb}{\mathbf{b}}
    
    \newcommand{\vb}{\mathbf{v}}
    \newcommand{\rb}{\mathbf{r}}
    \newcommand{\y}{\mathbf{ y}}
    \newcommand{\z}{\mathbf{z}}
    \newcommand{\E}{\mathbb{E}}
    \newcommand{\epb}{\pmb{\epsilon}}

    \newcommand{\lb}{\pmb{\lambda}}

    \newcommand{\T}{\mathbf{t}}
    \newcommand{\A}{\mathbf{a}}

	\renewcommand{\a}{\mathbf{a}}
    \newcommand{\nocontentsline}[3]{}
    \newcommand{\tocless}[2]{\bgroup\let\addcontentsline=\nocontentsline#1{#2}\egroup}

\newcommand{\adacrr}{\alg\xspace}

\newcommand{\tO}{\widetilde{O}}

\title[Near-optimal Consistent Robust Regression]{Adaptive Hard Thresholding for \\ Near-optimal Consistent Robust Regression}

\coltauthor{%
 \Name{Arun Sai Suggala}\thanks{Part of the work done while interning at Microsoft Research, India.} \Email{asuggala@cs.cmu.edu}\\
 \addr Carnegie Mellon University
 \AND
 \Name{Kush Bhatia} \Email{kushbhatia@berkeley.edu}\\
 \addr University of California, Berkeley
 \AND
  \Name{Pradeep Ravikumar} \Email{pradeepr@cs.cmu.edu}\\
 \addr Carnegie Mellon University
 \AND
 \Name{Prateek Jain} \Email{prajain@microsoft.com}\\
 \addr Microsoft Research, India%
}

\begin{document}
\maketitle
\vspace{-20pt}
\begin{abstract}
We study the problem of robust linear regression with response variable corruptions. We consider the oblivious adversary model, where the adversary corrupts a fraction of the responses in complete ignorance of the data. We provide a nearly linear time estimator which consistently estimates the true regression vector, even with $1-o(1)$ fraction of corruptions. Existing results in this setting either don't guarantee consistent estimates or can only handle a small fraction of corruptions. We also extend our estimator to robust sparse linear regression and show that similar guarantees hold in this setting. Finally, we apply our estimator to the problem of linear regression with heavy-tailed noise and show that our estimator consistently estimates the regression vector even when the noise has unbounded variance (e.g., Cauchy distribution), for which most existing results don't even apply. Our estimator is based on a novel variant of outlier removal via hard thresholding in which the threshold is chosen adaptively and crucially relies on randomness to escape bad fixed points of the non-convex hard thresholding operation.
\end{abstract}

\tocless

\section{Introduction}
\label{sec:intro}
We study robust least squares regression, where the goal is to robustly estimate a linear predictor from data which is potentially corrupted by an adversary. We focus on the setting where response variables are corrupted via an oblivious adversary. Such a setting has numerous applications such as click-fraud in a typical ads system, ratings-fraud in recommendation systems, as well as the less obvious application of regression with heavy tailed noise.
%

For the problem of oblivious adversarial corruptions, our goal is to design an estimator 
that satisfies three key criteria: (a) (\textbf{statistical efficiency}) estimates the optimal solution {\em consistently} with nearly optimal statistical rates, (b) (\textbf{robustness efficiency}) allows a high amount of corruption, \emph{i.e.,} fraction of corruptions is $1-o(1)$, (c) (\textbf{computational efficiency}) has the same or nearly the same computational complexity as the standard ordinary least squares (OLS) estimator. Most existing techniques do not even provide consistent estimates in this adversary model \citep{bhatia2015robust,nasrabadi2011robust,nguyen2013exact,prasad2018robust,diakonikolas2018sever,wright2010dense}. \cite{bhatia2017consistent} provides statistically consistent and computationally efficient estimator, but requires the fraction of corruptions to be less than a small constant ($\leq 1/100$). \cite{tsakonas2014convergence} study Huber-loss based regression to provide nearly optimal statistical rate with nearly optimal fraction of corruptions. But their sample complexity is sub-optimal, and more critically, the algorithm has super-linear computational complexity (in terms of number of points) and is significantly slower than the standard least squares estimator.

So the following is still an open question: {\em ``Can we design a linear time consistent estimator for  robust regression that allows almost all responses to be corrupted by an oblivious adversary?''}

We answer this question in affirmative, \emph{i.e.,} we design a novel outlier removal technique that can ensure consistent estimation at nearly optimal statistical rates, assuming Gaussian data and sub-Gaussian noise. Our results hold as long as the number of points $n$ is larger than the input dimensionality $p$ by logarithmic factors, i.e., $n\geq p\log^2 p$, and allows $n-\frac{n}{\log \log n}$ responses to be corrupted; the number of corrupted responses can be increased to $n-\frac{n}{\log n}$ with a slightly worse generalization error rate.

\begin{figure}[t]
    \centering
    \includegraphics[scale=0.21]{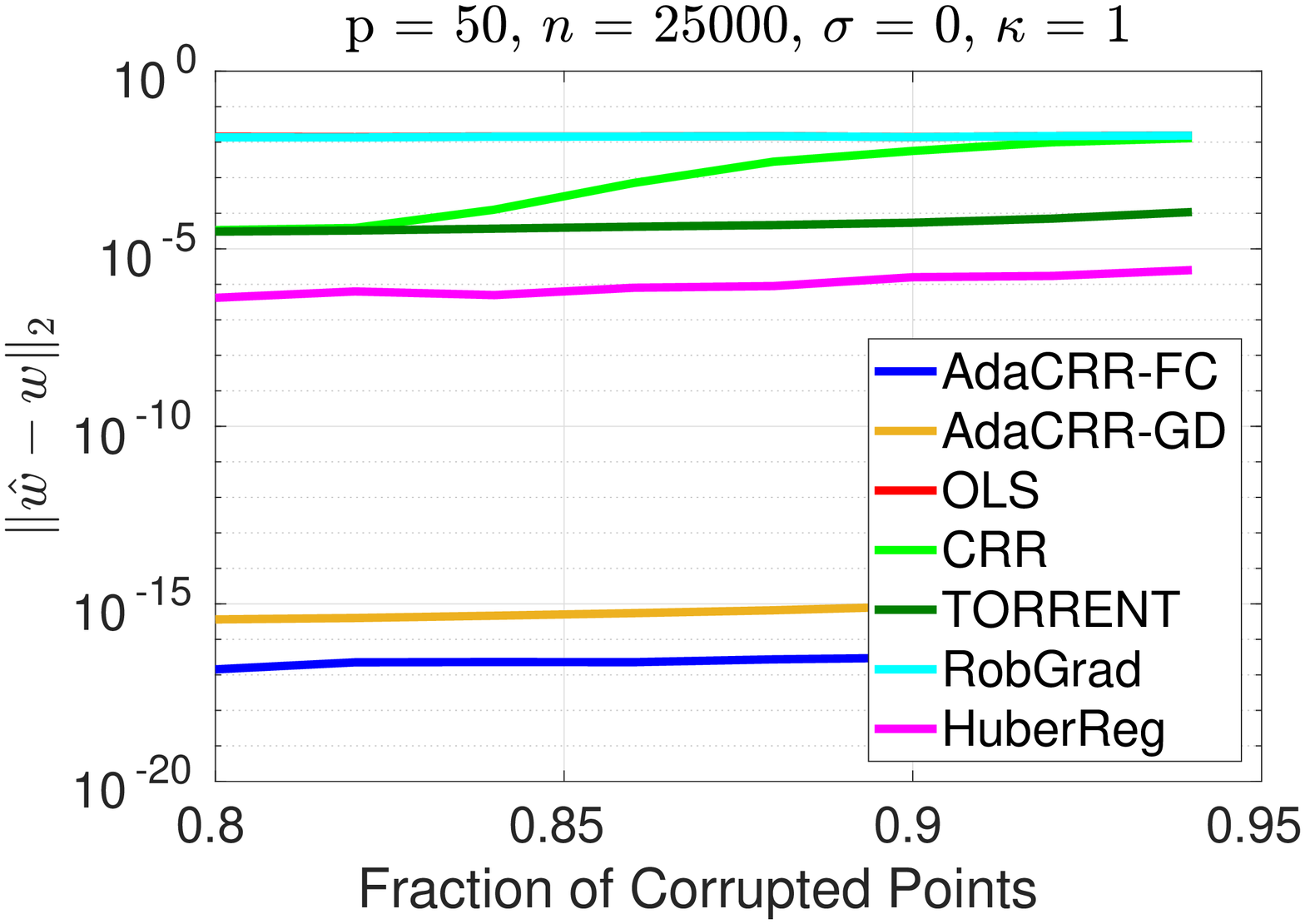}
    \includegraphics[scale=0.21]{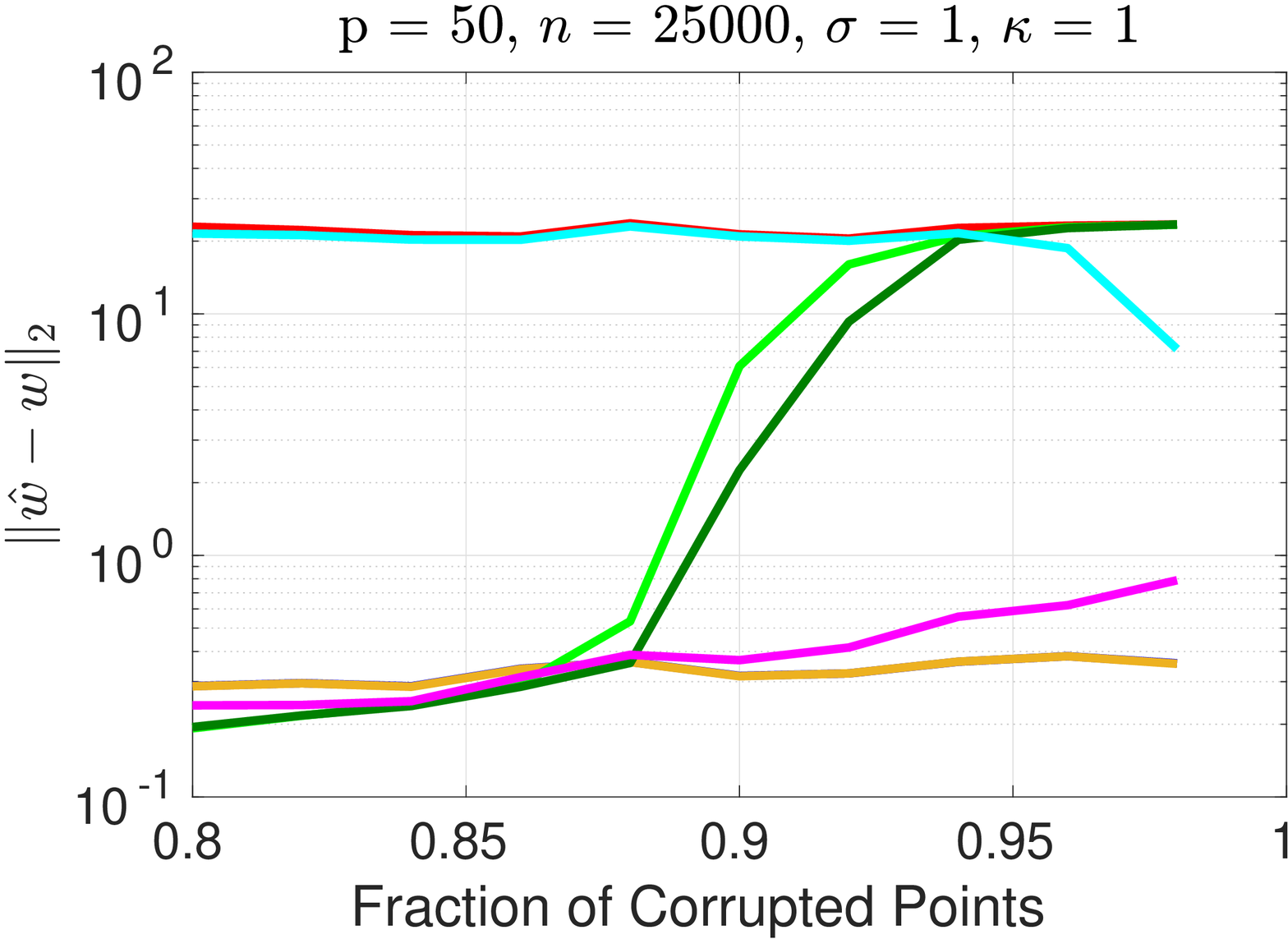}
    \includegraphics[scale=0.21]{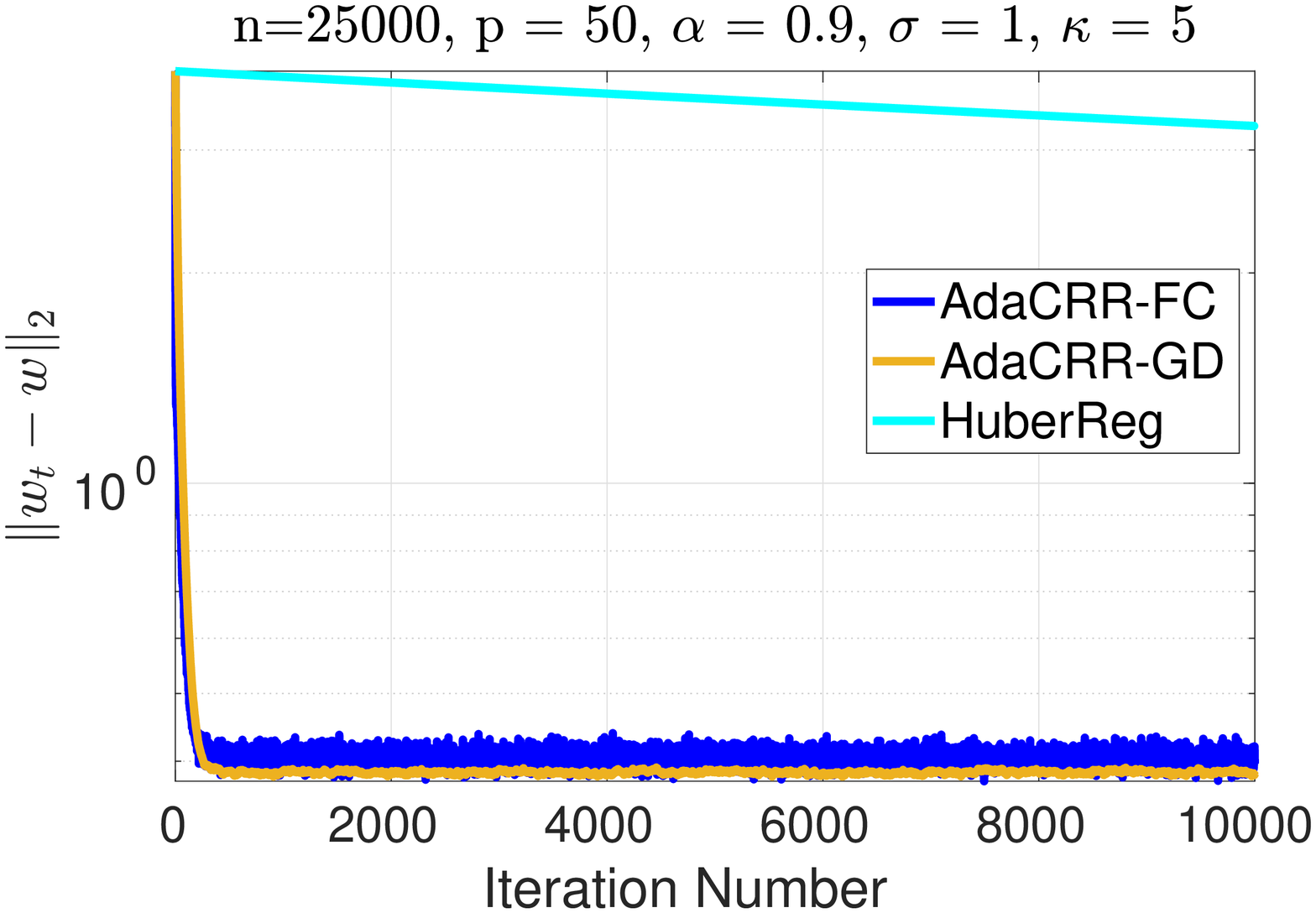}\vspace*{-3pt}
    \caption{\small{The first two plots show the parameter error (y-axis) of various estimators as we vary fraction of corruptions $\alpha$ in the robust regression setting (x-axis); noise variance is $0$ for the first plot and $1$ for the second. Plots indicate that \alg is able to tolerate significantly higher fraction of outliers than most existing methods. The last plot shows parameter error over number of iterations for robust regression, indicating \alg can be upto 100x faster as compared to Huber regression.}}\vspace{-3ex}
    \label{fig:main}
\end{figure}

Our algorithm, which we refer to as \adacrr\footnote{To be more precise, \alg is a framework and we study two algorithms instantiated from this framework, namely \alg-FC, \alg-GD which differ in how they update $\w$.}, uses a similar technique as \citet{bhatia2015robust, bhatia2017consistent}, where we threshold out points that we estimate as outliers in each iteration. However, we show that fixed thresholding operators as in \citet{bhatia2015robust, bhatia2017consistent} can get stuck at poor fixed-points in presence of a large number of outliers (see Section~\ref{sec:algorithm}). Instead, we rely on an adaptive thresholding operator that uses noise in each iteration to avoid such sub-optimal fixed-points. Similar to \citet{bhatia2015robust,bhatia2017consistent}, \alg-FC solves a standard OLS problem in each iteration, so the overall complexity is $O(T \cdot T_{OLS})$ where $T$ is the number of iterations and $T_{OLS}$ is the time-complexity of an OLS solver. We show that $T=O(\log 1/\epsilon)$ iterations are enough to obtain $\epsilon$-optimal solution, i.e., the algorithm is almost as efficient as the standard OLS solvers. Our simulations also demonstrate our claim, \emph{i.e.,} we observe that \alg-FC is significantly more efficient than Huber-loss based approaches \citep{tsakonas2014convergence}  while still ensuring consistency in presence of a large number of corruptions unlike existing thresholding techniques \citep{bhatia2015robust, bhatia2017consistent} (see Figure~\ref{fig:main}).

The above result requires $n\geq p\log^2 p$ which is prohibitively large for high-dimensional problems. Instead, we study the  problem with sparsity structure on the regression vector \citep{wainwright2009sharp}. That is, we study the problem of sparse linear regression  with oblivious response corruptions. We provide {\em first} (to the best of our knowledge) consistent estimator for the problem under standard RSC assumptions. Similar to the low-d case, we allow $1-o(1)$ fraction of points to be corrupted, but the sample complexity requirement is only $n\geq k^*\log^2 p$,  where $k^*$ is the number of non-zero entries in the optimal sparse regression vector. Existing Huber-loss based estimators \citep{tsakonas2014convergence} would be difficult to extend to this setting due to the additional non-smooth $L_1$ regularization of the regression vector. Existing hard-thresholding based consistent estimators \citep{bhatia2017consistent} marginalize out the regression vector, which is possible only in low-d due to the closed form representation of the least squares solution, and hence, do not trivially extend to sparse regression.

Finally, we enhance and apply our technique to the problem of regression with heavy-tailed noise. By treating the tail as oblivious adversarial corruptions, we obtain consistent estimators for a large class of heavy-tailed noise distributions that might not even have well-defined first or second moments. Despite being a well-studied problem, to the best of our knowledge, this is the first such result in this domain of learning with heavy tailed noise. For example, our results provide consistent estimators  with Cauchy noise, for which even the mean is not well defined, with rates which are very similar to that of standard sub-Gaussian distributions.  In contrast, most existing results~\citep{sun2018adaptive,hsu2016loss} do not even hold for Cauchy noise as they require the variance of the noise to be bounded. Furthermore, existing results mostly rely on {\em median of means} technique~\citep{hsu2016loss, lecue2017robust,prasad2018robust}, while we present a novel but natural viewpoint of modeling the tail of noise as adversarial but oblivious corruptions.

\paragraph{ Paper Organization.} Next section presents the problem setup and our main results. Section~\ref{sec:related_work} discusses some of the related works. Section~\ref{sec:algorithm} presents our algorithm and discusses why adaptive thresholding is necessary. Our extension to sparse linear regression is presented in Section~\ref{sec:sparse}. Section~\ref{sec:heavy} presents our results for the regression with heavy tailed noise problem. We conclude with Section~\ref{sec:conc}. Due to the lack of space, most proofs  and experiments are presented in the appendix.

\tocless
\section{Problem Setup and Main Results}
\label{sec:setup}
We are given $n$ independent data points $\x_1, \dots, \x_n \sim D$ sampled from a Gaussian distribution $D=\mathcal{N}(0, \Sigma)$ and their corrupted responses $y_1, \dots, y_n$, where,
\begin{equation}\label{eq:main_model}
y_i = \x_i^T\w^* + \epsilon_i + b_i^*,
\end{equation}
$\w^*$ is the true regression vector, $\epsilon_i$  - the white noise - is independent of $\x_i$ and is sampled from a sub-Gaussian distribution with parameter $\sigma$, and $b_i^*$ is the corruption in the response of $\x_i$. $\{b_i^*\}_{i =1}^n$ is a sparse corruption set, i.e., $\|b^*\|_0=|\{i,\ s.t.,\ b^*_i\neq 0\}|\leq \alpha \cdot n$ where $\alpha<1$. Also, $\{b_i^*\}_{i =1}^n$ is {\em  independent} of $\{\x_i, \epsilon_i\}_{i = 1}^n$. Apart from this independence we do not impose any restrictions on the values of corruptions added by the adversary. Our goal is to robustly estimate $\w^*$ from the corrupted data $\{\x_i,y_i\}_{i = 1}^n$. In particular, following are the key criteria in evaluating an estimator's performance:
%
\begin{itemize}
\item \textbf{Breakdown point:} It is the maximum fraction of corruption, $\alpha$, above which the estimator is not guaranteed to recover $\w^*$ with small error, even as $n \to \infty$ \citep{hampel1971general}.
\item \textbf{Statistical rates and sample complexity:} We are interested in the generalization error ($\mathbb{E}_{x\sim D}[(\langle \x, \w\rangle - \langle \x, \w^*\rangle)^2]$) of the estimator and its scaling with  problem dependent quantities like $n$, $p$, noise variance $\sigma^2$ as well as the fraction of corruption $\alpha$.
\item \textbf{Computational complexity:} The number of computational steps taken to compute the estimator. The goal is to obtain nearly linear time estimators similar to the standard OLS solvers.
\end{itemize}
As discussed later in the section, our \adacrr estimator is near optimal with respect to all three criteria above.

\paragraph{Heavy-tailed Regression.} We also study the heavy-tailed regression problem where $y_i = \x_i^T\w^* + \epsilon_i$ for all $\x_i \sim D$ and $i\in [n]$. Noise $\epsilon_i \stackrel{i.i.d.}{\sim} \mathcal{E}$ where $\mathcal{E}$ is a heavy-tailed distribution, such as the Cauchy distribution which does not even have bounded first moment. The goal is to design an efficient estimator that provides nearly optimal statistical rates.


\paragraph{Notation.} Let $X = [\x_1, \x_2, \dots \x_n]^T$ be the matrix whose $i^{th}$ row is equal to $\x_i\in \mathbb{R}^p$. Let $\y = [y_1, y_2 \dots y_n]^T, \epb = [\epsilon_1, \dots \epsilon_n]^T$, and $\bb^* = [b_1^*, \dots b_n^*]^T$. For any matrix $X \in \mathbb{R}^{n \times p}$ and subset $S \subseteq [n]$,  we use $X_{S}$ to denote the submatrix of $X$ obtained by selecting the rows corresponding to $S$. Throughout the paper, we denote vectors by bold-faced letters ($\a$), and matrices by capital letters ($A$). $\|\a\|_\Sigma^2:=\a^T \Sigma \a$ for a positive definite matrix $\Sigma$. $\|\a\|_0$ denotes the $L_0$ norm of $\a$, i.e., the number of non-zero elements in $\a$. $b=\tO(a)$ implies, $b\leq C a\log a$ for a large enough constant $C>0$ independent of $a$. We use $SG(\sigma^2)$ to denote the set of random variables whose Moment Generating Function (MGF) is less than the MGF of $\mathcal{N}(0, \sigma^2)$.
\vspace{-1ex}
\subsection{Main Results}
{\bf Robust Regression}: For robust regression with oblivious response variable corruptions, we propose the first efficient consistent estimator with break-down point of $1$. That is,
\begin{theorem}[Robust Regression]
	\label{thm:rr}
	Let $\{\x_i, y_i\}_{i = 1}^n$ be $n$ observations generated from the oblivious adversary model, i.e., $\y=X\w^*+\epb+\bb^*$ where $\epsilon_i \in SG(\sigma^2)$, $\x_i\sim \mathcal{N}(0,\Sigma)$, $\|\bb^*\|_0\leq \alpha\cdot n$ and $\bb^*$ is selected independently of $X, \epb$. Suppose \alg-FC  is run for $T$ iterations with appropariate choice of hyperparameters. Then with probability at least $1-T/n^6$, the $T$-th iterate $\w_T$ produced by the \alg-FC algorithm satisfies:
	\begin{equation*}
	\displaystyle    \|\w_{T}-\w^*\|_{\Sigma} \leq \tO\left(\frac{\sigma}{1-\alpha}\sqrt{\frac{p\log^2{n}+ (\log{n})^3}{n}}\right),
\end{equation*}\vspace{-1ex}
	for any $\alpha\leq 1-\frac{\Theta(1)}{\log\log n}$, where the number of iterations $T=\tO\left(\log \left(\frac{n}{p}\cdot \frac{\|\w_0-\w^*\|_\Sigma}{\sigma}\right)\right)$.
\end{theorem}
{\bf Remarks:} a) \alg-FC solves an OLS problem in each iteration and the number of iterations is $\approx \log n$, so the overall time complexity of the algorithm is still nearly linear in $n$. In contrast, standard Huber-loss or $L_1$ loss based methods \citep{tsakonas2014convergence,nasrabadi2011robust} have iteration complexity of $1/\sqrt{\epsilon}$ for $\epsilon$-suboptimality and require $\epsilon\approx 1/\sqrt{n}$, which implies super-linear $O(n^{1.25})$ time complexity. Our experiments (Section~\ref{sec:exp_rr}) also agree with this observation.\\[3pt]
b) Break-down point $\alpha$ of \alg-FC satisfies:  $\alpha\rightarrow 1$ for $n\rightarrow \infty$. In contrast, similar consistent estimator by \cite{bhatia2017consistent} requires $\alpha<1/100$. In fact, Proposition~\ref{prop:torrent_negative} shows that fixed hard thresholding operators like the ones used by \citep{bhatia2015robust,bhatia2017consistent} {\em cannot} provide consistent estimator for $\alpha\rightarrow 1$; instead, we propose and analyze a randomized and adaptive thresholding operator (Algorithm~\ref{alg:ht_oracle}) to avoid sub-optimal fixed-points.\\[3pt]
c) Generalization error of \alg-FC is $O(\sigma^2 \cdot p\log^2 n/n)$, which is information theoretically optimal up to $\log^2 n$ factors. In contrast, most of the existing analysis for $L_1$-loss do not guarantee such consistent estimators \citep{nasrabadi2011robust,wright2010dense,nguyen2013exact}.\\[3pt]
d) Our result is presented for Gaussian covariates and sub-Gaussian response noise. However, the technique is significantly more general and can apply to a large class of sub-Gaussian data distributions. Furthermore, we can relax the assumption on independence of  $\epsilon_i, \x_i$. It suffices to have $\mathbb{E}[\epsilon_i|\x_i] = 0$.\\[3pt]
e) Sample complexity of \alg-FC is nearly optimal $n=O(p\log^2 p)$ and can be improved to $n=O(k^*\log^2 p)$ for $k^*$-sparse estimators with the data that satisfies RSC/RSS  (Theorem~\ref{thm:real_hd}). \\[3pt]
See Table~\ref{table:comparison} for a detailed comparison with the existing works. \\
\begin{table}[t]
	\centering
	\resizebox{\textwidth}{!}{
		\begin{tabular}{|c c c c c|}
			\hline
			Paper & Breadkdown Point & Consistent & \begin{tabular}{c} Optimal \\  Sample Comlexity\end{tabular} & Computational Rates\\ [0.5ex]
			\hline
			\citet{wright2010dense} & $\alpha \to 1$ & No & Yes & $O(1/\sqrt{\epsilon})$\\
			\citet{nasrabadi2011robust} & $\alpha \to 1$ & No & Yes & $O(1/\sqrt{\epsilon})$ \\
			\citet{tsakonas2014convergence} & $\alpha \to 1$ & Yes & No & $O(1/\sqrt{\epsilon})$\\
			\citet{bhatia2017consistent} & $\alpha = \Theta(1)$ & Yes & Yes & $O(\log(1/{\epsilon}))$\\
			\textbf{This paper} & $\alpha \to 1$ & Yes & Yes & $O(\log(1/{\epsilon}))$\\
			\hline
		\end{tabular}
	}
	\caption{Comparison of various approaches for regression under oblivious adversary model. The computational rates represents the time taken by estimator to compute an $\epsilon$-approximate solution.}\vspace*{-20pt}
	\label{table:comparison}
\end{table}

\noindent{\bf Regression with Heavy-tailed Noise:} We present our result for regression with heavy-tailed noise.
\begin{theorem}[Heavy-tailed Regression]
	\label{thm:ht}
	Let $\{\x_i, y_i\}_{i = 1}^n$ be $n$ observations generated from the linear model, i.e., $y_i =\x_i^T \w^*+\epsilon_i$ where $\x_i\sim \mathcal{N}(0,\Sigma)$, $\epsilon_i$'s are sampled i.i.d. from a distribution s.t. $\mathbb{E}[|\epsilon|^{\delta}] \leq C$ for a constant $\delta>0$ and are independent of $\x_i$. Then, for \mbox{$T=\tO\left(\log \left(\frac{n}{p}\cdot \frac{\|\w_0-\w^*\|_\Sigma}{\sigma}\right)\right)$}, the $\w_{T}$-th iterate of \alg-FC guarantees the following with probability $\geq 1-T/n^6$:
	\begin{equation*}\|\w_{T}-\w^*\|_{\Sigma} \leq O\left(C^{1/\delta}\sqrt{\frac{p\log{n} + \log^2{n}}{n}}\right).
	\end{equation*}
%
\end{theorem}
{\bf Remarks}: a) Note that our technique does not even require the first moment to exist. In contrast, existing results hold only when the  variance is bounded~\citep{hsu2016loss}. In fact, the general requirement on distribution of $\epsilon$ is significantly weaker and holds for almost every distribution whose parameters are independent of $n$. Also, we present a similar result for mean estimation with symmetric noise $\epsilon$. \\[3pt]
b) For Cauchy noise \citep{johnson2005univariate} with location parameter $0$, and scale parameter $\sigma$, we can guarantee error rate of $\approx \sigma \sqrt{\frac{p \log^2 n}{n}}$, i.e., we can obtain almost same rate as sub-Gaussian noise despite unbounded variance which precludes most of the existing results. Our empirical results also agree with the theoretical claims, i.e., they show small generalization error for \alg-FC while almost trivial error for several  heavy-tailed regression algorithms (see Figure~\ref{fig:cauchy_noise}).\\[3pt]
c) Similar to robust regression, the estimator is nearly linear in $n$, $p$. Moreover, we can extend our analysis to  sparse linear regression with heavy-tailed response noise.

\tocless

\section{Related Work}
\label{sec:related_work}
The problems of robust regression and heavy tailed regression have been extensively studied in the fields of robust statistics and statistical learning theory. We now review some of the relevant works in the literature and discuss their applicability to our setup.

\paragraph{Robust Regression.} The problem of response corrupted robust regression can be written as the following equivalent optimization problems:
\begin{equation}
\label{eqn:torrent}
\min_{\begin{subarray}{c} \w, S\\ S\subset [n], |S|= (1-\alpha) n\end{subarray}} \sum_{i \in S} \left(y_i -\left\langle \x_i, \w\right\rangle \right)^2 \Leftrightarrow \quad \min_{\begin{subarray}{c} \w, \bb\\ \|\bb\|_0 \leq \alpha n\end{subarray}} \|\y -X\w -\bb\|^2_2.
\end{equation}
The problem is NP-hard in general 	due to it's combinatorial nature~\citep{studer2012recovery}. 
\citet{rousseeuw1984least} introduced the Least Trimmed Squares (LTS) estimator which computes OLS estimator over all subsets of points and selects the best estimator. Naturally, the estimator's computational complexity is exponential in $n$ and is not practical. There are some practical variants like RANSAC \citep{ransac} but they are mostly heuristics and do not come with strong guarantees. 

A number of approaches have been proposed which relax \eqref{eqn:torrent} with $L_1$ loss \citep{wright2010dense} or Huber loss \citep{huber1973robust}. \citet{tsakonas2014convergence} analyze Huber regression estimator under the oblivious adversary model and show that it tolerates any constant fraction of corruptions, while being consistent. However, their analysis requires $\Tilde{\Omega}(p^2)$ samples.  \citet{wright2010dense, nasrabadi2011robust} also study convex relaxations of~\eqref{eqn:torrent}, albeit in the sparse regression setting. While their estimators tolerate any constant fraction of corruptions, they do not guarantee consistency in presence of white noise.  Statistical properties aside, a major drawback of Huber's M-estimator and other convex relaxation based approaches is that they are computationally expensive due to  sublinear convergence rates to the global optimum.   Another class of approaches use greedy or local search heuristics to approximately solve the $\ell_0$ constrained objectives. For example, the estimator of ~\citet{bhatia2017consistent} uses alternating minimization to optimize objective~\eqref{eqn:torrent}.  While this estimator is consistent and converges linearly to the optimal solution, it only tolerates a small fraction of corruptions and breaks down when $\alpha$ is greater than a small constant.

Another active line of research on robust regression has focused on handling more challenging adversary models. One such popular  model is the malicious adversary model, where the adversary looks at the data before adding corruptions. Recently there has been a flurry of research on designing robust estimators that are both  computationally and statistically efficient in this setting \citep{bhatia2015robust, prasad2018robust, diakonikolas2018sever, klivans2018efficient}. While the approach by \cite{bhatia2015robust} is based on an alternating minimization procedure, \cite{prasad2018robust} and \cite{diakonikolas2018sever} derive robust regression estimators based on  robust mean estimation \citep{lai2016agnostic,diakonikolas2016robust}. However, for such an  adaptive adversary, we cannot expect to achieve consistent estimator. In fact, it is easy to show that we cannot expect to obtain generalization error better than $O(\alpha \sigma)$ where $\alpha$ is the fraction of corruptions and $\sigma$ is the noise variance. Furthermore, as we show in our experiments, these techniques fail to recover the parameter vector in the oblivious adversary model when the fraction of corruption is $1-o(1)$.

\paragraph{Heavy-tailed Regression.} Robustness to heavy-tailed noise distribution is another regression setting that is actively studied in the statistics community. The objective here is to construct estimators which work without the sub-Gaussian distributional assumptions that are typically imposed on the data distribution, and allow it to be a heavy tailed distribution. For the setting where the noise $\epsilon$ is heavy-tailed with bounded variance, Huber's estimator is known to achieve sub-Gaussian style rates~\citep{fan2017estimation, sun2018adaptive}.  Several other non-convex losses such as Tukey's biweight and Cauchy loss have also been proposed for this setting~\citep[see][and references therein]{loh2017statistical}. 
 For the case where both the covariates and noise are heavy-tailed, several recent works have proposed computationally efficient estimators that achieve sub-Gaussian style rates~\citep{hsu2016loss, lecue2017robust, prasad2018robust}.
As noted earlier, all of these results require bounded variance. Moreover, many of the Huber-loss style estimators typically do not have linear time computational complexity. In contrast, our result holds even if the $\delta$-th moment of noise is bounded where $\delta>0$ is any arbitrary small constant. Furthermore, the estimation algorithm is nearly linear in the number of data points as well as data dimensionality.

\tocless
\section{\alg Algorithm}
\label{sec:algorithm}

In this section we describe our algorithm \alg (see Algorithm~\ref{alg:torrent}), for estimating the regression vector in the oblivious adversary model. At a high level, \alg uses alternating minimization to optimize objective~\eqref{eqn:torrent}. That is, \alg maintains an estimate of the coefficient vector $\w_t$ and the set of corrupted responses $S_t$, and alternatively updates them at every iteration.
\noindent \paragraph{Updating $\w_t$.}
Given any subset $S_t$, $\w_{t}$ is updated using the points in $S_t$. We study two variants of \alg which differ in how we update $\w_t$.  In \alg-FC (Algorithm~\ref{alg:torrent_fc}) we perform a fully corrective linear regression step on points from $S_t$. In \alg-GD (Algorithm~\ref{alg:torrent_gd}) we take a gradient descent step to update $\w_t$. While these two variants have similar statistical properties, the GD variant is computationally more efficient, especially for large $n$ and $p$.
\noindent \paragraph{Updating $S_t$.} For any given $\w_t$, \alg updates $S_t$ using a novel hard thresholding procedure, which adds all the points whose absolute residual is larger than an adaptively chosen threshold, to the set $S_{t+1}$. Hard thresholding based algorithms for robust regression have been explored in the literature \citep{bhatia2017consistent,bhatia2015robust} but they use thresholding with a fixed threshold or at a fixed level and are unable to guarantee $\alpha =1-o(1)$ break-down point. In fact, as we show in Proposition~\ref{prop:torrent_negative}, such fixed hard thresholding operators cannot in general tolerate such large fraction of corruption.

In contrast, our hard thresholding routine (detailed in Section~\ref{sec:adaht}) selects the threshold adaptively and adds randomness to escape bad fixed points.  While randomness has proven to be useful in escaping first and second order saddle points in unconstrained optimization \citep{ge2015escaping,jin2017accelerated}, to the best of our knowledge, our result is the first such result for a constrained optimization problem with randomness in the projection step.

Before we proceed, note that Algorithm~\ref{alg:torrent} relies on a new set of samples for each iteration. This ensures independence of the current iterate $\w_{t-1}$ from the samples and is done mainly for theoretical convenience. We believe this can be eliminated using more complex arguments in the analysis.

\subsection{\adaht: Adaptive Hard Thresholding Operator}
\label{sec:adaht}
In this section we describe our hard thresholding operator \adaht.
There are two key steps involved in \adaht, which we describe below.  Consider the call to \adaht in $t^{th}$ iteration of Algorithm~\ref{alg:torrent}.
\noindent\paragraph{Interval Selection.} In the first step we find an interval on positive real line which acts as a ``crude'' threshold for our hard thresholding operator. We partition the positive real line into intervals of width $\intl_t$. We then  place points in these intervals based on the magnitude of their residuals. Finally, we pick the smallest $j$ such that the $j^{th}$ interval has fewer than  $\frac{\gamma\tn}{j\log{\tn}}$ elements in it, for some $\gamma$ such that $1 < \gamma < \log{\tn}$. Let $j_t$ be the chosen interval. This interval acts as a crude threshold. All the points to the left of $j_t^{th}$ interval are considered as un-corrupted points and added to $S_t$ (line 7-9, Algorithm~\ref{alg:ht_oracle}); all the points to the right of $j_t^{th}$ interval are considered as corrupted points.
The goal of such interval selection is to ensure: a) all the true un-corrupted points lie to the left of the interval and are included in $S_t$, and, b) not many points fall in interval $j_t$ so that a large fraction of the points that are in set $S_t$ remain independent of each other. This independence allow us to exploit sub-Gaussian concentration results rather than employing a worse case-bounds and helps achieve optimal consistent rates.

Let $\beta\in (0,1)$ be a constant. Then, we select interval length as:
\begin{equation}
\label{eqn:dist_ub}
\intlub_t = 18\sqrt{(2\hat{\sigma}^2 + 2\beta^{2(t-1)}\hat{d}_0^2)\log{\tn}},
\end{equation}
where, $\hat{\sigma}$ and $\hat{d}_0$ are approximate upper bounds of $\sigma$ and $\|\dw_0\|_2=\|\w_0-\w^*\|_{\Sigma}$:
\begin{equation}
\label{eqn:sigma_ub}
\sigma \leq \hat{\sigma} \leq \mu \sigma \qquad \text{and} \qquad \|\dw_0\|_2 \leq \hat{d}_0 \leq \nu \|\dw_0\|_2,
\end{equation}
where $\mu\geq 1$, $\nu\geq 1$. In Section~\ref{sec:analysis} we show that for appropriate choice of $\beta$, all the true un-corrupted points lie to the left of $j_{t}^{th}$ interval. In Appendix~\ref{sec:aux_additional} we present techniques to estimate $\hat{d}_0$ with a constant $\nu$. Estimating the noise variance $\sigma^2$ (and $\hat{\sigma}^2$) is significantly more tricky and it is not clear if it is even possible apriori. So, in practice one can either use prior knowledge or treat $\hat{\sigma}$ as a hyper-parameter that is selected using cross-validation.

\newcommand{\upd}{\textsc{Update}}
\hspace*{-15pt}
\begin{minipage}[t]{0.5\textwidth}
\begin{algorithm}[H]
\caption{\alg}
\label{alg:torrent}
\begin{algorithmic}[1]
  \small
  \STATE \textbf{Input:} Training data $(X, \y)$, iterations $T$, Update-routine: \upd
  \STATE Randomly split $(X, \y)$ into $T$ sets $\{(X_t, \y_t)\}_{t = 0}^T$ of size $\tn = \lfloor\frac{n}{T+1}\rfloor$ each
 \STATE $\w_0 \leftarrow (X_{0}^TX_{0})^{-1}X_{0}^T\y_{0}$
  \STATE $t \leftarrow 1$ 
  \WHILE{$t \leq T$}
        \STATE \textbf{Get new set of samples} $(X_{t}, \y_{t})$
        \STATE $S_{t} \leftarrow \text{AdaHT}\left(\y_{t} - X_{t}\w_{t-1}\right)$
        \STATE $\w_t \leftarrow \upd \left(X_{t,S_t}, \y_{t,S_t}, \w_{t-1}\right)$
        \STATE $t \leftarrow t + 1$
  \ENDWHILE
\end{algorithmic}
\end{algorithm}\vspace*{-20pt}
\begin{algorithm}[H]
\caption{\alg-FC (Fully Corrective)}
\label{alg:torrent_fc}
\begin{algorithmic}[1]
  \small
  \STATE \textbf{Input:}  Training data $(X, \y)$, iterations $T$
  \STATE \textsf{OLS}($X, \y$)$:=\argmin_{\w}\|\y-X\w\|_2^2$
  \STATE \textbf{Return:} \alg($X,\ \y,\ T$, \textsf{OLS}$(\cdot,\cdot)$)
\end{algorithmic}
\end{algorithm}\vspace*{-20pt}
\begin{algorithm}[H]
\caption{\alg-GD (Gradient Descent)}
\label{alg:torrent_gd}
\begin{algorithmic}[1]
  \small
  \STATE \textbf{Input:}  Train data $(X, \y)$, $T$, step length $\eta$
  \STATE GD($X, \y, \w$)$:=w-\eta X^T (Xw-y)$
  \STATE \textbf{Return:} \alg($X,\ \y,\ T$, GD$(\cdot,\cdot,\cdot)$)
\end{algorithmic}
\end{algorithm}
\end{minipage}
\hspace*{10pt}
\begin{minipage}[t]{.48\textwidth}
\begin{algorithm}[H]
	\caption{Adaptive Hard Thresholding (\adaht)}
	\label{alg:ht_oracle}
	\begin{algorithmic}[1]
		\small
		\STATE \textbf{Input:} Residual $\rb_{t} = \y_{t} - X_{t}\w_{t-1}\in \mathbb{R}^{\tn}$
		\STATE \textbf{Output:} Estimated Subset $S_t$
		\STATE \textbf{Hyperparameter:} $\gamma$, decay factor $\beta$, noise variance estimate: $\hat{\sigma}^2$, error estimate: $\hat{d}_0$, randomness range $[-a,a]$, interval length: $\intl_t$
		\STATE Divide real line into intervals of width $\intl_t$ 
		\STATE Place each point in the appropriate interval based on its residual $|\rb_{t}(i)|$
		\STATE Find smallest $j_t$ such that the $j_t^{th}$ interval has fewer than $\frac{\gamma \tn}{j_t\log{\tn}}$ points. Denote its midpoint by $\tau_t = (j_t -\frac{1}{2})\intl_t$
		\FOR{each element i to left of $j_t^{th}$ interval}
		\STATE $S_t \leftarrow S_t \cup \{i\}$
		\ENDFOR
		\FOR{each element i in $j_t^{th}$ interval}
		\STATE Sample $\eta_{i,t}$ uniformly from $[-a,a]$
		\IF{$|\rb_{t}(i)| < \tau_t + \eta_{i, t} \intl_t$}
		\STATE $S_t \leftarrow S_t \cup \{i\}$
		\ENDIF
		\ENDFOR
		\STATE \textbf{return} $S_t$
	\end{algorithmic}
\end{algorithm}
\end{minipage}

\noindent\paragraph{Points in Selected Interval.} This step decides inclusion in $S_t$ of points which fall in the selected interval $j_t$. Let $\tau_t = (j_t -\frac{1}{2})\intl_t$ be the mid-point of this interval. For each point in this interval we sample $\eta$ uniformly from $[-a,a]$, for some universal constant $a \in (0, 0.1]$. If the magnitude of its residual is smaller than $\tau_t + \eta \intl_t$ we consider it as un-corrupted and add it to $S_t$ (see line 10-15, Algorithm~\ref{alg:ht_oracle}). As we show in the proof of Theorem~\ref{thm:rr}, this additional randomness is critical in avoiding poor fixed points and in obtaining  the desired statistical rates for the problem.

\subsection{Fixed Hard Thresholding doesn't work}
In this section we show that algorithms of  \citet{bhatia2015robust,bhatia2017consistent} that rely on \emph{fixed} hard thresholding operators pruning out a fixed number of elements, need not recover the true parameter $\w^*$ when $\alpha \to 1$. We prove this for \tor~\citep{bhatia2015robust}; proof for \crr~\citep{bhatia2017consistent} can be similarly worked out.

\tor is based on a similar alternating minimization procedure as \alg, but differs from it in the subset selection routine: instead of adaptive hard thresholding, \tor always chooses the smallest $(1-\alpha)n$ elements from the residual vector $(\y_t-X_t\w_{t-1})$. The following proposition provides an example where \tor fails to recover the underlying estimate for $\alpha = 0.8$.

\begin{wrapfigure}{r}{0.45\textwidth}
    \centering
    \vspace{-0.2in}
    \includegraphics[scale=0.22, trim={0 0 0 1cm},clip]{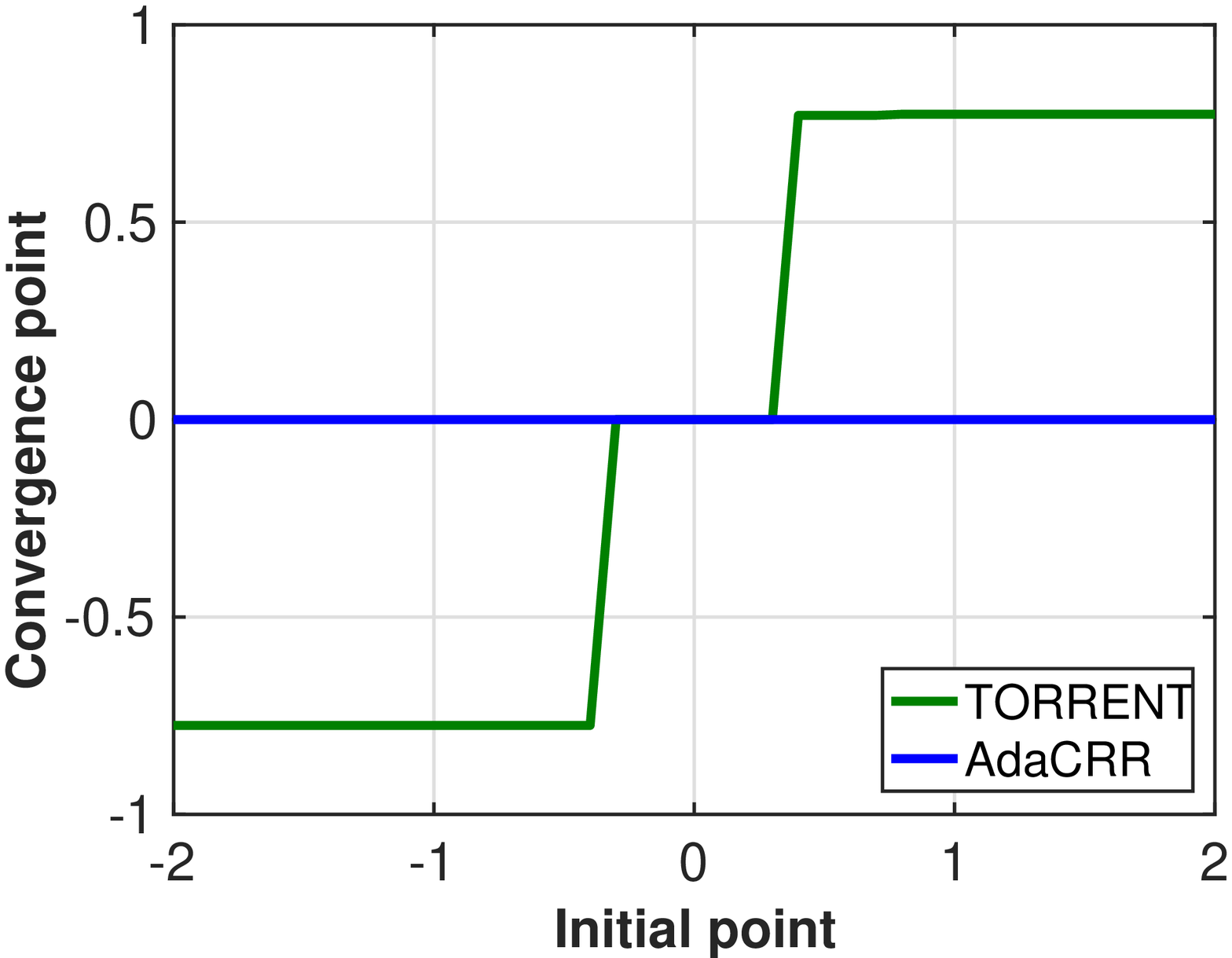}
\end{wrapfigure}
\begin{proposition}[Lower Bound for \tor]
\label{prop:torrent_negative}
Let $y_i = x_i w^* + b_i^*$, $i\in [n]$,  where $w^*=0$, $x_i \stackrel{i.i.d}{\sim} \mathcal{N}(0,1)$. Let $b_i^*=1$ for $1\leq i\leq \alpha\cdot n$  and $0$ otherwise. Consider the limit as $n \to \infty$ and suppose $\alpha = 0.8$. Then $\exists \w\in \mathbb{R}$ which is far from $\w^*$ (\textit{i.e.,} $|\w-\w^*| = \Omega(1)$) such that if \tor is initialized at $\w$, it remains at $\w$ even after infinite iterations.
\end{proposition}

See Appendix~\ref{sec:aux_torr_neg} for a detailed proof of the proposition. Figure on the right shows the performance of \tor on the $1$-d regression problem described in Proposition~\ref{prop:torrent_negative}. The x-axis denotes the initial point while the y-axis denotes the point of convergence of \tor. Clearly TORRENT fails with several initializations despite $10^6$ samples.

\tocless

\section{Analysis}
\label{sec:analysis}
In this section we provide an outline of the proof of our main result stated in Theorem~\ref{thm:rr}. We prove a more general result in Theorem~\ref{thm:real} from which Theorem~\ref{thm:rr} follows readily.
\begin{theorem}[\alg-FC for Robust Regression]
\label{thm:real}
Consider the setting of Theorem~\ref{thm:rr}. Set $\tn = \frac{n}{T+1}$, $a = 1/18$, $\gamma \in (1, \log{\tn})$, $\beta\geq\frac{c\gamma}{(1-\alpha)\log{\tn}}$ for a universal constant $c>0$. Let $\hat{\sigma}, \hat{d}_0$ be given s.t. $\hat{\sigma}/\sigma\in [1,\mu]$ and $\hat{d}_0/\|\dw_0\|_2 \in [1,\nu]$ with $\|\dw_0\|_2=\|\w_0-\w^*\|_{\Sigma}$. Set $\intlub_t = 18\sqrt{(2\hat{\sigma}^2 + 2\beta^{2(t-1)}\hat{d}_0^2)\log{\tn}}$, $t\in [T]$. Then the  iterates $\{\w_{t}\}_{t = 1}^T$ of \alg-FC (Algorithm~\ref{alg:torrent_fc}) executed with the above given hyperparameters, satisfy the following (w.p. $\geq 1-T/\tn^6$):

\begin{equation}
    \label{eqn:descent}
         \displaystyle \|\w_{t}-\w^*\|_{\Sigma} \leq \beta^{t}\|\w_0-\w^*\|_{\Sigma} + O\left(\frac{\mu\sigma \tn^{1/\gamma}}{(1-\beta)(1-\alpha)}\sqrt{\frac{p\log{\tn} + \log^2{\tn}}{\tn}}\right)
\end{equation}
where break-down point $\alpha < 1 - \frac{c\gamma}{\log{\tn}}$ and $\tn^{1-2/\gamma} \geq c_1\max\left\lbrace\frac{\mu^2}{(1-\beta)^2(1-\alpha)^2}, \frac{\nu^2\log^2{\tn}}{\gamma^2}\right\rbrace\left(p\log{\tn} + \log^2{\tn}\right)$.
\end{theorem}
\begin{proof} (Sketch)
Consider the $(t+1)^{th}$ iteration of \alg-FC. We first divide $(X_{t+1}, \y_{t+1})$ into the following mutually exclusive sets:
\begin{equation}
\begin{array}{c}
Q_1 = \left\lbrace i: |\bb_{t+1}^{*}(i)| \geq  \tau_{t+1} + \frac{5}{18}\intlub_{t+1}\right\rbrace,\quad Q_2 = \left\lbrace i: |\bb_{t+1}^{*}(i)| <  \tau_{t+1} -  \frac{5}{18}\intlub_{t+1}\right\rbrace, \vspace{0.05in}\\
Q_3 = \left\lbrace i: |\bb_{t+1}^{*}(i) -\tau_{t+1}| \leq \frac{5}{18}\intlub_{t+1}, \text{ and } |\y_{t+1}(i) - \left\langle \x_{t+1,i}, \w_{t}\right\rangle | \geq \tau_{t+1} + \eta_{i,t+1}\intlub_{t+1}\right\rbrace,\vspace{0.05in}\\
Q_4 = \left\lbrace i:
|\bb_{t+1}^{*}(i) -\tau_{t+1}| \leq \frac{5}{18}\intlub_{t+1}, \text{ and } |\y_{t+1}(i) - \left\langle \x_{t+1,i}, \w_{t}\right\rangle | < \tau_{t+1} + \eta_{i,t+1}\intlub_{t+1}\right\rbrace,
\end{array}\label{eq:setsQ}
\end{equation}
where $\tau_{t+1}:=(j_t-1){I}_{t+1}$ is as defined in Line 6, Algorithm~\ref{alg:ht_oracle}. Note that $Q_1$ contains the egregious outliers in $(X_{t+1}, \y_{t+1})$ and $Q_2$ contains all the ``true'' uncorrupted points. Our proof first shows that $\intlub_{t+1}$ satisfies the properties described in Section~\ref{sec:adaht}. Specifically, the output $S_{t+1}$ of \adaht satisfies: a) $Q_2 \subseteq S_{t+1}$, b) $Q_1 \cap S_{t+1} = \{\}$ and, c) $S_{t+1} = Q_2 \cup Q_4$. Next, we show that $\w_{t+1}-\w^*$ can be written in terms of $Q_2, Q_4$:{\small
\begin{equation*}
        \w_{t+1}-\w^* = -(X_{t+1,S_{t+1}}^TX_{t+1,S_{t+1}})^{-1}\left(\sum_{i \in Q_2\cup Q_4}(\bb_{t+1}^*(i)+\epb_{t+1}(i)) X_{t+1, i}	\right).
\end{equation*}}
The rest of the proof focuses on bounding the two terms in the RHS of the above equation. To bound the first term involving $Q_2$ we use the observation that $\bb_{t+1}^*, \epb_{t+1}$ are independent of $X_{t+1}$ and rely on concentration properties of sub-Gaussian random variables. To bound the other term involving $Q_4$,
we rely on a crucial property of our algorithm which guarantees $|Q_4| \leq \frac{\gamma\tn}{j_{t+1}\log{\tn}}$ and perform a worst case analysis to bound $Q_4$. See Appendix~\ref{sec:aux_thm_real} for a detailed proof.
\end{proof}
{\bf Discussion:} Theorem~\ref{thm:real} characterizes both the computational as well as statistical guarantees of \alg-FC. More specifically, consider setting $\gamma = \frac{2\log{\tn}}{\log\log{\tn}}$. Then, if the number of samples $\tn = \tilde{\Omega}\left(\max\left\lbrace \mu^2, \nu^2\right\rbrace p\right)$, \alg-FC after $T = O\left(\log\left(\frac{\tn}{p}\frac{\|\w_0-\w^*\|_{\Sigma}}{\sigma}\right)\right)$ iterations (and hence nearly linear computation time) produces an iterate $\w_T$ satisfying,
\begin{equation*}
  \|\w_T-\w^*\|_{\Sigma} = O\left(\mu\sigma \sqrt{\frac{p\log^2{\tn} + \log^3{\tn}}{\tn}}\right),
\end{equation*}
for  $\alpha < 1-\frac{c}{\log\log{\tn}}$ where $c>0$ is a universal constant. 
This shows that constant-factor estimates of $\sigma$, $\|\dw_0\|$ suffices to achieve information theoretically optimal rates, up to $\log^2{n}$ factors, even with $1-\frac{\Theta(1)}{\log\log{\tn}}$ fraction of corruptions.
In fact, \alg-FC can tolerate~$1-\frac{\Theta(1)}{\log{\tn}}$ fraction of corruptions by setting $\gamma = O(1)$, although with a slightly worse parameter estimation error.

\tocless
\section{Consistent Robust Sparse Regression}
\label{sec:sparse}
In this section, we extend our algorithm to the problem of sparse regression with oblivious response variable corruptions. In this setting the dimension of the data $p$ is allowed to exceed the sample size $n$.  When $p > n$, the linear regression model is unidentifiable. Consequently, to make the model identifiable, certain structural assumptions need to be imposed on the parameter vector $\w^*$. Following \citet{wainwright2009sharp}, this work assumes that $\w^*$ is $k^*$-sparse i.e. has at most $k^*$ non-zero entries. Our objective now is to recover a sparse $\hat{\w}$ with small generalization error.  In this setting, we modify the update step of $\w$ in Algorithm~\ref{alg:torrent} as follows \vspace{-0.05in}
\begin{equation}\label{eq:sparse}
\w_t \leftarrow \argmin_{\w:\|\w\|_0 \leq k} \|\y_{t,S_t} - X_{t,S_t}\w\|_2^2,
\vspace{-0.05in}
\end{equation}
and start the algorithm at $\w_0 = 0$. We refer to this modified algorithm as \alg-HD.
A huge number of techniques have been proposed to solve the above optimization problem efficiently. Under certain properties of the design matrix (Restricted Eigenvalue property), these techniques estimate $\w^*$ up to statistical precision, using just $O(k^*\log{p})$ samples.  In this work we use the Iterative Hard Thresholding (IHT) technique of \citet{jain2014iterative} to solve the above problem. More details about the IHT Algorithm can be found in Appendix~\ref{sec:aux_proof_hd}.


\begin{theorem}[\alg-HD for Sparse Robust Regression]
\label{thm:real_hd}
Consider the setting of Theorem~\ref{thm:rr}. In addition assume $\|\w^*\|_0 \leq k^*$. Use IHT (Algorithm~\ref{alg:iht}) to solve \eqref{eq:sparse} in each iteration of \alg-HD with parameter $k= \frac{\Omega(1)}{(1-\alpha)^4}\frac{\lambda_{max}(\Sigma)^2}{\lambda_{min}(\Sigma)^2}k^*$. Set $\gamma = \frac{2\log{\tn}}{\log{\log{\tn}}}, \intlub_t = 18\sqrt{(2\hat{\sigma}^2 + 2\beta^{2(t-1)}\hat{d}_0^2)\log{p}}$, and all the other parameters as in Theorem \ref{thm:real}. 
  Then the $T^{th}$ iterate produced by \alg-HD, for $T = O\left(\log\left(\frac{\tn}{k}\frac{\|\w_0-\w^*\|_{\Sigma}}{\sigma}\right)\right)$ satisfies the following bound with probability at least $1-T/p^{6}$: 
\begin{equation*}
         \displaystyle \|\w_{T}-\w^*\|_{\Sigma} = O\left(\frac{\mu\sigma }{(1-\alpha-2c(\log\log{\tn})^{-1})} \sqrt{\frac{k\log{\tn}\log^2{p}}{\tn}}\right).
\end{equation*}
where $\tn=n/(T+1)$,  $\tn = \tilde{\Omega}\left(\max\left\lbrace\mu^2,\nu^2\right\rbrace k\log{\tn}\log^2{p}\right)$, and $\alpha < 1 - \frac{2c}{\log\log{\tn}}$. 
\end{theorem}
We would like to highlight nearly linear sample complexity in $k^*$ for well-conditioned covariates. Furthermore, the total time complexity of the algorithm is still nearly linear in $n$ and $p$. Finally, to the best of our knowledge, this is the first result for the sparse regression setting with oblivious response corruptions and break-down point $\alpha\rightarrow 1$. 
\tocless
\section{Regression with Heavy-tailed Noise}
\label{sec:heavy}
In this section we consider the problem of linear regression with heavy-tailed noise. We consider the heavy-tailed model from Section~\ref{sec:setup} where we observe $n$ i.i.d samples from the linear model: $y_i = \left\langle \x_i, \w^*\right\rangle + \epsilon_i,$
where $\epsilon_i$ is sampled from a heavy-tailed distribution. We now show that our estimator from Section~\ref{sec:algorithm} can be adapted to this setting to estimate $\w^*$ with sub-Gaussian error rates, even when the noise lacks the first moment.

In this setting, although there is no adversary corrupting the data, we consider any point with noise greater than a threshold $\rho$ as a ``corrupted'' point, and try not to use these points to estimate $\w^*$. That is, we decompose $\epsilon_i=\bar{\epsilon}_i+b_i^*$ where $|\bar{\epsilon}_i|\leq \rho$. Note that this implies dependence between $\bar{\epsilon}_i$ and $b_i^*$, but as we show later in Appendix~\ref{sec:app_heavy}, our proof still goes through with minor modifications, and in fact, provides similar rates as  the case where $\epsilon$ is sampled from a Gaussian distribution. Below, we provide a more general result than Theorem~\ref{thm:ht}, from which Theorem~\ref{thm:ht} follows by appropriate choice of $\rho$. Define $\alpha_{\rho}$, the tail probability of $\epsilon$ as $\alpha_{\rho}\coloneqq \mathbb{P}(|\epsilon| > \rho)$.
\begin{theorem}[\alg-FC for Heavy-tailed Noise]
\label{thm:heavy_tails}
Consider the setting of Theorem~\ref{thm:ht}. Let $\rho>0$ be any threshold and $\tn = \frac{n}{T+1}$. Set $a = 1/18,  \gamma =\frac{2\log{\tn}}{\log\log{\tn}}$, estimate $\hat{d}_0$ that satisfies \eqref{eqn:sigma_ub}, set $\beta = \frac{\Omega(1)}{(1-\alpha_{\rho})\log\log{\tn}}$, $\intl_t=18\left(\frac{\rho}{\sqrt{8}} + \beta^{t-1}\hat{d}_0\sqrt{\log{\tn}}\right). $ 
Then, for any $\rho$ such that $\alpha_{\rho} < 1 - \frac{2c}{\log\log{\tn}}$, and $\tn = \tilde{\Omega}\left(\nu^2(p\log^2{\tn} + \log^3{\tn})\right)$, the $T^{th}$ iterate produced by \alg-FC executed with above parameters and $T = O\left(\log\left(\frac{\tn}{p}\frac{\|\w_0-\w^*\|_{\Sigma}}{\rho}\right)\right)$ satisfies the following  w.p. $\geq 1-T/\tn^{6}$:
\begin{equation*}
         \|\w_{T}-\w^*\|_{\Sigma} = O\left(\frac{\rho}{(1-\alpha_{\rho}-2c(\log\log{\tn})^{-1})} \sqrt{\frac{p\log{\tn} + \log^2{\tn}}{\tn}}\right).
\end{equation*}
\end{theorem}
Note that if the distribution of $\epsilon$ is independent of $n$,  we should always be able to find constants $\rho$ and $\alpha_\rho$ to obtain nearly optimal rates. We instantiate this claim for the popular Cauchy noise, for which the existing results do not even apply due to unbounded variance.
\begin{corollary}[Cauchy noise]
\label{cor:heavy_tails_cauchy}
Consider the similar setting as in Theorem~\ref{thm:ht}. Suppose the noise follows a Cauchy distribution with location parameter $0$ and scale parameter $\sigma$. Then, the $T^{th}$ iterate of \alg-FC, for $T = O\left(\log\left(\frac{\tn \|\w_0-\w^*\|_{\Sigma}}{p}\right)\right)$ satisfies the following, w.p.  $\geq 1-T/\tn^{6}$:
\begin{equation*}
        \|\w_{T}-\w^*\|_{\Sigma} = O(\sigma)\cdot \sqrt{\frac{p\log{\tn} + \log^2{\tn}}{\tn}}.
\end{equation*}
\end{corollary}
We would like to note that despite sub-Gaussian style rates for Cauchy noise, the sample and time complexity of the algorithm is still nearly optimal. \\
{\bf Mean estimation:} Although our result holds for regression, we can extend our result to solve the mean estimation problem as well. That is, suppose $\y_i=\w^*+\epb_i\in \mathbb{R}^p$ where $i\in [n]$, $\w^*$ is the mean of a distribution and $\epb_i$ is a zero mean random variable which follows a heavy-tailed distribution. Then by using a simple symmetrization reduction, we can show that we can compute $\w_T$ such that $\|\w_T-\w^*\|\leq C^{\frac{1}{1+\delta}}\sqrt{\frac{p\log^2{n}}{n}}$, if $\E[|\epb_{i}(j)|^{1+\delta}|]\leq C$ and $\epb_{i}(j)$ is a symmetric random variable, i.e., $P(\epb_{i}(j))=P(-\epb_{i}(j))$, $\forall j\in [p]$.

This result seems to be counter-intuitive as \citet{devroye2016sub} derive lower bounds for heavy tailed mean estimation and show that over the set of all $(1+\delta)^{th}$ moment bounded distributions, no estimator can achieve faster rates than $O\left(n^{-\min\{\delta/(1+\delta), 1/2\}}\right)$ while we can obtain $O(n^{-1/2})$ rates. However, we additionally require noise distribution to be symmetric, while the lower bound construction uses asymmetric noise distribution. We further discuss this problem in Appendix~\ref{sec:app_mean}. Similarly, our result avoids regression lower-bound by \citet{sun2018adaptive}, as we do not estimate the bias term in our regression model.

\tocless
\section{Conclusion}
\label{sec:conc}
In this paper, we studied the problem of response robust regression with oblivious adversary. For this problem, we presented a simple outlier removal based algorithm that uses a novel randomized and adaptive thresholding algorithm. We proved that our algorithm provides a consistent estimator with break-down point (fraction of   corruptions) of $1-o(1)$ while still ensuring a nearly linear-time computational complexity. Empirical results on synthetic data agrees with our results and show computational advantage of our algorithm over Huber-loss based algorithms \citep{tsakonas2014convergence} as well as better break-down point than thresholding techniques \citep{bhatia2015robust,bhatia2017consistent}. We also provided an extension of our approach to the high-dimensional setting. Finally, our technique extends to the problem of linear regression with heavy-tailed noise, where we provide nearly optimal rates for a general class of noise distributions that need not have a well-defined first moment.

The finite sample break-down point of our method is $1-O(1/\log n)$ which is still sub-optimal compared to the information theoretic limit of $1-\Omega(d/n)$. Obtaining efficient estimators for nearly optimal break-down point is an interesting open question. Furthermore, our algorithm requires an approximate estimate of noise variance $\sigma^2$ which can sometimes be difficult to select in practice. A completely parameter-free algorithm for robust regression (similar to OLS) is an interesting research direction that should have significant impact in practice as well.

\clearpage

\bibliography{local}

\clearpage
\appendix
\onecolumn
\setcounter{tocdepth}{2}
\tableofcontents
\newpage

\section{Proof of Proposition~\ref{prop:torrent_negative}}
\label{sec:aux_torr_neg}
Let $\{\x_i,y_i\}_{i=1}^n$  be the $n$ points we observe, out of which at most $\alpha n$ points are corrupted. Note that the true linear model is such that $\w^* = 0$, $\sigma = 0$. Based on this model, we have: 
\[
\y = \bb^*, \quad \text{where } \bb^*(i) = \begin{cases} 1,& \quad \text{if } $i$ \text{ is corrupted}\\ 0,&\quad \text{otherwise}\end{cases}.
\]
Let's suppose we start the TORRENT algorithm at $\w$. Given $\w$, TORRENT computes its estimate of the un-corrupted  points as:
\begin{equation}
    \label{eqn:torrent_neg_intd1}
    S = HT_{(1-\alpha)n}(\y-X\w) = HT_{(1-\alpha)n}(\bb^*-X\w),
\end{equation}
where $HT_{(1-\alpha)n}(\vb)$ returns the $(1-\alpha)n$ points in $\vb$ with smallest magnitude. Given $S$, TORRENT updates its estimate of parameter vector as:
 \[
 \w^+ = \left(X_S^TX_S\right)^{-1}X_S^T\y_S' = \left(X_S^TX_S\right)^{-1}X_S^T\bb_S^* = \frac{\left\langle X_S, \bb_S^*\right\rangle}{\|X_S\|^2_2}.
 \]
Note that if $\w^+ = \w$, then TORRENT will be stuck at $\w$ and will not make any progress. We now show that for large $\alpha$ there in fact exists a $\w > 0$ such that $\w^+ = \w$.

Let $\tau_{\w}$ be the threshold used in the hard thresholding operator to compute $S$ in Equation~\eqref{eqn:torrent_neg_intd1}; that is, $\tau_{\w}$ is such that the magnitude of residuals of all the points in $S$ is less than $\tau_{\w}$ and magnitude of residuals of all the points in $S^c$ is greater than $\tau_{\w}$.  Note that there are $(1-\alpha)$ fraction of points with residuals less than $\tau_{\w}$. Since we are working in the $n \to \infty$ setting, this implies
\[
\mathbb{P}_{x\sim\mathcal{N}(0,1),b^*}(|b^*-x\w| < \tau_{\w}) = (1-\alpha).
\]
Rewriting the LHS of the above expression, we get:{\small
\begin{equation*}
    \begin{array}{lll}
         \displaystyle\mathbb{P}_{x\sim\mathcal{N}(0,1),b^*}(|b^*-x\w| < \tau_{\w}) &=&  \displaystyle\mathbb{P}(b^* = 0)\mathbb{P}\left(|b^* - x\w| < \tau_{\w}|b^* = 0\right) + \mathbb{P}(b^* = 1)\mathbb{P}\left(|b^*-x\w| < \tau_{\w}|b^* = 1\right)\\
         & = &\displaystyle \mathbb{P}(b^* = 0)\mathbb{P}\left(|x\w| < \tau_{\w}\right) + \mathbb{P}(b^* = 1)\mathbb{P}\left(|1-x\w| < \tau_{\w}\right)\\
         &=& (1-\alpha)\left(\Phi\left(\frac{\tau_{\w}}{\w}\right) - \Phi\left(-\frac{\tau_{\w}}{\w}\right)\right) + \alpha \left(\Phi\left(\frac{1+\tau_{\w}}{\w}\right) - \Phi\left(\frac{1-\tau_{\w}}{\w}\right)\right).
    \end{array}
\end{equation*}}
Combining the above two equations, we get
\begin{equation}
    \label{eqn:torrent_neg_eq1}
    (1-\alpha)\left(\Phi\left(\frac{\tau_{\w}}{\w}\right) - \Phi\left(-\frac{\tau_{\w}}{\w}\right)\right) + \alpha \left(\Phi\left(\frac{1+\tau_{\w}}{\w}\right) - \Phi\left(\frac{1-\tau_{\w}}{\w}\right)\right) = 1-\alpha.
\end{equation}
For TORRENT to be stuck at $\w$, we require $\w = \w^+$, i.e.,  $\w=\frac{\left\langle X_S, \bb_S^*\right\rangle}{\|X_S\|^2_2}=\frac{\mathbb{E}\left[b^*x\Big||b^*-x\w| < \tau_{\w}\right]}{\mathbb{E}\left[x^2\Big||b^*-x\w| < \tau_{\w}\right]}$. As $b^*=1$ uniformly at random with probability $\alpha$, the final term reduces to: 
\begin{equation}
\w=\frac{\alpha \mathbb{E}\left[x\Big||1-x\w| < \tau_{\w}\right]}{(1-\alpha)\mathbb{E}\left[x^2\Big||x\w| < \tau_{\w}\right] + \alpha \mathbb{E}\left[x^2\Big||1-x\w| < \tau_{\w}\right]}.
\label{eqn:torrent_neg_eq2}
\end{equation}
This shows that TORRENT will be stuck at $\w$ iff there exists a $\tau_{\w} > 0$ such that Equations~\eqref{eqn:torrent_neg_eq1},~\eqref{eqn:torrent_neg_eq2} hold. The two are essentially system of linear equations in $\alpha$. And it is easy to verify feasibility of this system for various $\tau_w$. For example, for $\alpha = 0.8$ the equations are feasible and $\w = 0.79$, $\tau_{\w} = 0.354$ are approximate feasible points. 



\section{Proof of Theorem~\ref{thm:real}}
\label{sec:proof_oracle}
Before we present the proof of the Theorem, we introduce some notation and present useful intermediate results which we require in our proof.
The proofs of all the Lemmas in this section can be found in Appendix~\ref{sec:aux_thm_real}.
\paragraph{Notation}
Recall that $(X_t, \y_t)$ are the new points obtained in $t^{th}$ iteration of Algorithm~\ref{alg:torrent}. Let $\bb_t^*$ be the corruption vector added to these points and $\epb_t$ be the noise vector. Let $\tX_{t}$ be obtained from $X_{t}$ by applying the whitening transformation:
\[
\tX_{t} := X_{t}\Sigma^{-1/2},\ \dw_{t} \coloneqq \Sigma^{1/2}(\w^* - \w_{t}).
\]
 Let $S_t^*$ be the set of un-corrupted points in $(X_t, \y_t)$. Let $S_t$ be the output of \adaht in the $t^{th}$ iteration of \alg-FC and  $j_t$ be the interval chosen. For any $S \subseteq [\tn]$, let $X_{t,S}$  be the $|S|\times p$ matrix with $\{\x_{t,i}, i \in S\}$ as rows.
  Finally, let us define $\zeta := \frac{c\gamma}{(1-\alpha)\log{\tn}}$.
 \subsection{Intermediate Results}
 \label{sec:aux_proof_rr_main_itd}
\begin{lemma}
\label{lem:update_rewrite}
The input $\rb_{t} = \y_{t} - X_{t}\w_{t-1}$ to \adaht can be written in terms of $\dw_{t-1}$ as
\[
\rb_{t} = \bb^{*}_{t} + \tX_{t}\dw_{t-1} + \epb_{t},
\]
where $\bb^*_t$ is the corruption vector of points $(X_t, \y_t)$.
\end{lemma}
The following Lemma obtains a bound on $j_t$, the interval number, chosen by Algorithm~\ref{alg:ht_oracle}.
\begin{lemma}[Interval Number]
\label{lem:bucket_number}
Let $j_t$ be the interval chosen by \adaht in the $t^{th}$ iteration of \alg-FC.  Then $j_t \leq \tn^{1/\gamma}$.
\end{lemma}
The following Lemma presents a condition on $\intl_t$ which ensures that all the uncorrupted points fall to the left of $j_{t}^{th}$ interval.
\begin{lemma}[Interval Length]
\label{prop:HT_properties}
Consider the $t^{th}$ iteration of \alg-FC. Suppose \adaht is run with the interval length $\intl_t$ such that: $\intl_t \geq 18\sqrt{(\sigma^2+\|\dw_{t-1}\|_2^2) \log{\tn}},$ and $a = 1/18$ and $\gamma \in (1, \log{\tn})$.
Define sets $Q_1, Q_2$, which are subsets of points in $(X_{t}, \y_{t})$, as follows: 
\begin{equation*}
Q_1 = \left\lbrace i: |\bb_{t}^{*}(i)| >  (j_t - 2/9)\intl_t \right\rbrace \quad \text{and} \quad
Q_2 = \left\lbrace i: |\bb_{t}^{*}(i)| < (j_t - 7/9)\intl_t  \right\rbrace.
\end{equation*}
 Then the following statements hold with probability at least $1-1/\tn^7$: 
 \begin{equation*}
 Q_1 \cap S_t = \{\}, \quad S_t^* \subseteq Q_2 \subseteq S_t.
\end{equation*}
 Moreover, all the points in $(Q_1\cup Q_2)^c$ fall in the $j_t^{th}$ interval.
\end{lemma}

\subsection{Main Argument}
We first prove the following Lemma, which obtains a bound on the progress made by \alg-FC in each iteration, assuming  $\intl_t \geq 18\sqrt{(\sigma^2 + \|\dw_{t-1}\|^2)\log{\tn}}$. In Section~\ref{sec:aux_real_induction} we use this Lemma to prove Theorem~\ref{thm:real}.
\begin{lemma}\label{lem:real_intd_lemma}
	Consider the setting of Theorem~\ref{thm:real}. Let $\intl_t  \geq  18\sqrt{(\sigma^2 + \|\dw_{t-1}\|^2)\log{\tn}}.$ Then, $\forall t \in [T]$, w.p. $\geq 1-1/\tn^6$: 
	 $$\|\dw_t\|_2 = O\left(\frac{\gamma}{(1-\alpha)\log{\tn}}\right)\|\dw_{t-1}\|_2 + O\left(\frac{\tn^{1/\gamma}}{1-\alpha} \sqrt{\frac{p + \log{\tn}}{\tn}}\right)\intl_t + O\left(\frac{\sigma}{1-\alpha}\sqrt{\frac{\alpha p\log{\tn}}{\tn}}\right),$$
	 where $\dw_t=\Sigma^{1/2}(\w_t-\w^*)$. 
\end{lemma}

\begin{proof} 
Consider the $t^{th}$ iteration of \alg-FC.
We divide the $\tn$ points in $(X_{t}, \y_{t})$ into the following mutually exclusive sets $Q_1, Q_2, Q_3, Q_4$
$$Q_1 = \left\lbrace i: |\bb_{t}^{*}(i)| >  \tau_{t} + \frac{5}{18}\intl_{t}\right\rbrace, \quad Q_2 = \left\lbrace i: |\bb_{t}^{*}(i)| <  \tau_{t} - \frac{5}{18}\intl_{t}\right\rbrace,$$ $$Q_3 = \left\lbrace i: |\bb_{t}^{*}(i) -\tau_{t}| \leq \frac{5}{18}\intl_{t}, \text{ and } |\y_{t}(i) - \left\langle \x_{t,i}, \w_{t-1}\right\rangle | \geq \tau_{t} + \eta_{i,t}\intl_{t}\right\rbrace,$$ $$Q_4 = \left\lbrace i:
|\bb_{t}^{*}(i) -\tau_{t}| \leq \frac{5}{18}\intl_{t}, \text{ and } |\y_{t}(i) - \left\langle \x_{t,i}, \w_{t-1}\right\rangle | < \tau_{t} + \eta_{i,t}\intl_{t}\right\rbrace,$$
where $\tau_t=(j_t-0.5)I_t$ is as defined by Line 6 of Algorithm~\ref{alg:ht_oracle}. 
We now highlight some key properties of the sets $Q_1, Q_2, Q_3, Q_4$, which follow from Lemma~\ref{prop:HT_properties} and hold with probability at least $1-1/\tn^7$.
 \begin{enumerate}
     \item  Since $\intl_t \geq 18\sqrt{(\sigma^2 + \|\dw_{t-1}\|^2)\log{\tn}}$,  from Lemma~\ref{prop:HT_properties} we have
 \[
 Q_1 \cap S_{t}^* = \{\}, \quad Q_1 \cap S_{t} = \{\}, \quad S_{t}^* \subseteq Q_2 \subseteq S_{t}.
 \]
 \item $S_{t} = Q_2 \cup Q_4$.
 \item Since $|S_{t}^*| \geq (1-\alpha)\tn$, we have $ |S_{t}| \geq |Q_2| \geq (1-\alpha)\tn$.
 \item Since any point in $Q_3\cup Q_4$ lies in the $j_{t}^{th}$ interval (see Lemma~\ref{prop:HT_properties}), we have $|Q_3 \cup Q_4| \leq \frac{\gamma \tn}{j_{t}\log{\tn}}$.
 \item For a given set of points $(X_{t}, \y_{t})$, there are at most $\tn^{1/\gamma}$ possible choices for $Q_1,Q_2$ and $Q_3 \cup Q_4$; one for each possible choice of $j_{t}$.
 \end{enumerate}
 We often use the above properties in the proof. Using sets $Q_1, Q_2, Q_3, Q_4$, we rewrite $\dw_{t}$ as:
\begin{equation*}
        \dw_{t} =  -(\tX_{t,S_{t}}^T\tX_{t,S_{t}})^{-1}\left(\underbrace{\left(\sum_{i \in Q_2}\bb_{t}^*(i) \tx_{t, i}\right)}_{T_1}
         -\underbrace{\left(\sum_{i \in Q_2}\epb_{t}(i) \tx_{t, i}\right)}_{T_2}
         -\underbrace{\left(\sum_{i \in Q_4}\left(\bb_{t}^*(i) + \epb_{t}(i) \right)\tx_{t, i}\right)}_{T_3}\right).
\end{equation*}
So we have: 
\begin{equation}\label{eq:rr_1}
\|\dw_t\|_2 \leq \frac{1}{\lb_{min}\left(\tX_{t,S_{t}}^T\tX_{t,S_{t}}\right)}\left(\|T_1\|_2 + \|T_2\|_2 + \|T_3\|_2\right).
\end{equation}
We now derive high probability upper bounds for each of the terms in the above expression.
\paragraph{Bounding $\lb_{min}\left(\tX_{t,S_{t}}^T\tX_{t,S_{t}}\right)$.}
Since $S_{t}^* \subseteq S_{t}$, we have
\[
\lb_{min}\left(\tX_{t,S_{t}}^T\tX_{t,S_{t}}\right) \geq \lb_{min}\left(\tX_{t,S_{t}^*}^T\tX_{t,S_{t}^*}\right)
\]
We now use the concentration properties of the smallest eigenvalue of covariance matrix to bound the above quantity. Using Lemma~\ref{lem:ss} in Appendix~\ref{sec:std_concentration},  we obtain the following inequality, which holds with probability at least $1-\delta$
\[
\frac{1}{|S_{t}^*|}\lb_{min}\left(\tX_{t,S_{t}^*}^T\tX_{t,S_{t}^*}\right) \geq 1 - \frac{1}{2}\sqrt{\frac{p}{|S_{t}^*|}} - \sqrt{\frac{\log{\frac{2}{\delta}}}{|S_{t}^*|}}.
\]
Since $|S_{t}^*| \geq (1-\alpha)\tn$, we have
\[
\lb_{min}\left(\tX_{t,S_{t}}^T\tX_{t,S_{t}}\right) \geq (1-\alpha)\tn \left[1 - \frac{1}{2}\sqrt{\frac{p}{(1-\alpha)\tn}} - \sqrt{\frac{\log{\frac{2}{\delta}}}{(1-\alpha)\tn}}\right].
\]
\paragraph{Bounding $T_1$.}
Define set $Q_{2,j}$ and term $T_{1,j}$ as follows: 
\[
Q_{2,j} = \{i: |\bb_{t}^{*}(i)| < (j-2/9)\intl_{t}\},\ T_{1,j} = \sum_{i \in Q_{2,j}}\bb_{t}^*(i) \tx_{t, i}.
\]
Note that $Q_2 = Q_{2,j_{t}}$ and $\|T_1\|_2 \leq \sup_{j \in [\tn^{1/\gamma}]} \|T_{1,j}\|_2.$

First, note that the distribution of covariates ($\x$) and dense noise $\epsilon$ of points in $Q_{2,j}$ is the same and equal to their corresponding distributions on entire data. This follows from the fact that $Q_{2,j}$ is formed based on the magnitude of corruptions $|\bb^*_{t}(i)|$, which is chosen independent of the data. We use this observation to derive upper bound for $T_{1,j}$.
Using chi-square concentration result from Lemma~\ref{lem:chi_conc} in Appendix~\ref{sec:std_concentration}, we obtain the following upper bound for $T_{1,j}$ (w.p. $\geq 1-\delta$): 
\[
\|T_{1,j}\|_2^2 \leq \left(p + O\left(\sqrt{p\log{\frac{1}{\delta}}} + \log{\frac{1}{\delta}} \right)\right)\|[\bb^{*}_{t}]_{Q_{2,j}}\|^2.
\]
Combining this result with the upper bound on $|\bb^{*}_{t}(i)|$, we have (w.p. $\geq 1-1/\tn^8$): 
\begin{multline}
     \|T_{1,j}\|_2 = O\left(\sqrt{p+\log{\tn}}\right)\|[\bb^{*}_{t}]_{Q_{2,j}}\|_2 = O\left(\sqrt{(p+\log{\tn})|Q_{2,j}|} \right)\|[\bb^{*}_{t}]_{Q_{2,j}}\|_{\infty}\\
     = O\left(\tn^{1/2}j\sqrt{p+\log{\tn}}\right)\intl_t= O\left(\tn^{1/2+1/\gamma}\sqrt{p+\log{\tn}}\right)\intl_t,
\end{multline}
where the third equality follows from the fact that $|Q_{2,j}| \leq \tn$ and the definition of $Q_{2,j}$. Last equality follows from $j\leq n^{1/\gamma}$. 
This shows that with probability at least $1-1/\tn^7$: 
\begin{equation}
\label{eqn:t1_bound_base_real}
\|T_1\|_2 = O\left(\tn^{1/2+1/\gamma}\sqrt{p + \log{\tn}}\right)\intl_t.
\end{equation}

\paragraph{Bounding $T_2$.} We use a similar technique as above to bound $T_2$. We first upper bound $\|T_2\|$ as
\[
\|T_2\| \leq \sup_{j \in [\tn^{1/\gamma}]}\|T_{2,j}\|_2, \ \ T_{2,j} = \sum_{i \in Q_{2,j}}\epb_{t}(i) \tx_{t, i}.
\]

To bound $T_{2,j}$ we make use of the fact that $\epb_{t}$ is independent of $\tX_{t}$. Conditioned on $[\epb_{t}]_{Q_{2,j}}$, $T_{2,j}$ follows a Gaussian distribution with covariance $\|[\epb_{t}]_{Q_{2,j}}\|_2^2 I$.
Using concentration results for sum of chi-square random variables (see Lemma~\ref{lem:chi_conc}) along with upper bound on $\|[\epb_{t}]_{Q_{2,j}}\|_2$ (see Lemma~\ref{lem:sg_norm_conc}), we can show that for any  given $Q_{2,j}$, the following holds with probability at least $1-\delta$: 
\[
\frac{1}{\tn}\|T_{2,j}\|_2 = O\left(\frac{\sqrt{p+\log{\frac{1}{\delta}}}}{\tn}\right)\|[\epb_{t}]_{Q_{2,j}}\|_2=O\left(\sigma\sqrt{\frac{\alpha p\log{\frac{1}{\delta}}}{\tn}}\right).
\]
Taking a union bound over all possible choices of $j$, we obtain the following bound, which holds with probability at least $1-1/\tn^{7}$: 
\begin{equation}
\label{eqn:t2_bound_base_real}
\frac{1}{\tn}\|T_2\|_2 =  O\left(\sigma\sqrt{\frac{\alpha p\log{\tn}}{\tn}}\right).
\end{equation}

\paragraph{Bounding $T_3$.}
    Bounding $T_3$ requires more careful arguments that we present in  Section~\ref{sec:t3_bound} where by using Equation~\eqref{eq:app_t3_bound}, Lemma~\ref{lem:expectation_T3}, and Lemma~\ref{lem:t3_concentration}, we have w.p. $\geq 1-2/\tn^{9}$: 
\begin{equation}
\label{eqn:t3_bound_base_real}
     \frac{1}{\tn}\|T_3\| 
     = \displaystyle O\left(\frac{\gamma }{\log{\tn}}\right)\|\dw_{t-1}\|_2 + O\left(\tn^{1/2\gamma}\sqrt{\frac{\gamma p}{\tn}}\right)\intl_t.
\end{equation}
Combining the bounds in Equations~\eqref{eqn:t1_bound_base_real},~\eqref{eqn:t2_bound_base_real},~\eqref{eqn:t3_bound_base_real}, we get the following bound, which holds with probability at least $1-1/\tn^6$: 
\begin{equation*}
\begin{array}{c}
\|\dw_t\|_2 = O\left(\frac{\gamma}{(1-\alpha)\log{\tn}}\right)\|\dw_{t-1}\|_2 + O\left(\frac{\tn^{1/\gamma}}{1-\alpha} \sqrt{\frac{p + \log{\tn}}{\tn}}\right)\intl_t + O\left(\frac{\sigma}{1-\alpha}\sqrt{\frac{\alpha p\log{\tn}}{\tn}}\right).
\end{array}
\end{equation*}
This finishes the proof of the Lemma. What remains now is to bound $T_3$, which we do in Section~\ref{sec:t3_bound}. 
\subsubsection{Bounding $T_3$}
\label{sec:t3_bound}
We first re-write $T_3$ as: 
\begin{equation}
    \label{eqn:t3_rewritten}
    T_3 = \sum_{i \in Q_3 \cup Q_4}  \mathbb{I}\left[|\bb_{t}^{*}(i) + \epb_{t}(i) + \left\langle \tx_{t,i}, \dw_{t-1}\right\rangle | < \tau_{t,i}\right](\bb_{t}^{*}(i) + \epb_{t}(i))\tx_{t,i},
\end{equation}
where $\tau_{t,i} = \left(j_{t}-\frac{1}{2}+\eta_{i,t}\right)\intl_{t}$ and $\eta_{i,t}$ is sampled uniformly from $[-1/18,1/18]$.
First, note that $Q_3 \cup Q_4$ depends on $j_{t}$ - the interval chosen by \adaht in the first iteration - which in turn depends on $X_{t}$. This dependence of $Q_3 \cup Q_4$ on $X_{t}$ complicates the analysis. So to simplify the analysis, we bound $\|T_3\|_2$ by bounding the quantity over all intervals. The bound would then follow by union bound over all possible intervals, whose number is bounded by $n^{1/\gamma}$. 

To this end, we first define $Q_j$ and  $T_{3,j}$ for any $j \in [\tn^{1/\gamma}]$: {\small
\begin{multline}\label{eqn:t3j_rewritten}
Q_{j} = \left\lbrace i: \Big|\bb_{t}^{*}(i) -(j-1/2)\intl_{t}\Big| \leq\frac{5}{18}\intl_{t}\right\rbrace,\\ T_{3,j} \coloneqq \sum_{i \in Q_j}  \mathbb{I}\left[|\bb_{t}^{*}(i) + \epb_{t}(i) + \left\langle \tx_{t,i}, \dw_{t-1}\right\rangle | < \tau_{1,j,i}\right](\bb_{t}^{*}(i) + \epb_{t}(i))\tx_{t,i},
\end{multline}}
where $\tau_{t,j,i} = \left(j-\frac{1}{2}+\eta_{i,t}\right)\intl_{t}$. Note that if $j = j_{t}$, then $Q_{j} = Q_3 \cup Q_4$. Now by taking supremum over all $j$'s and using triangular inequality, we have: 
\begin{equation}\label{eq:app_t3_bound}
\|T_3\|_2 \leq \sup_{\begin{subarray}{c}j \in [\tn^{1/\gamma}]\\ \text{s.t. } |Q_j| \leq \frac{\gamma \tn}{j\log{\tn}} \end{subarray}} \|T_{3,j}\|_2\leq \sup_{\begin{subarray}{c}j \in [\tn^{1/\gamma}]\\ \text{s.t. } |Q_j| \leq \frac{\gamma \tn}{j\log{\tn}} \end{subarray}} \|\E_{\tX_t}\left[T_{3,j}|\dw_{t-1}, \epb_t\right]\|_2 + \|T_{3,j}-\E_{\tX_t}\left[T_{3,j}|\dw_{t-1}, \epb_t\right]\|_2.
\end{equation}
The first term above is bounded by Lemma~\ref{lem:expectation_T3} and the second term by Lemma~\ref{lem:t3_concentration}.  Finally, taking a union bound over all possible choices of $j$, we get a high probability upper bound for $\|T_3\|_2$.
\end{proof}
\paragraph{Expectation of $T_{3,j}$.}
Before bounding the expectation of $T_{3,j}$ we present two auxiliary Lemmas, the proofs of which can be found in Appendix~\ref{sec:aux_thm_real}.
\begin{lemma}
\label{lem:expectation_T3_aux1}
Let $\tx\sim \mathcal{N}(0, I_{p \times p})$ be a random vector. For any given $\uu \in \mathbb{R}^p$ and scalars $v \geq 0,b$, consider the following random vector: $\A = \mathbb{I}(|b + \left\langle \tx_, \uu \right\rangle | < v)b\tx.$ 
 Then the expected value of $\A$ satisfies: $\E[\A] = \frac{b}{\sqrt{2\pi}}\left[e^{-\frac{\left(v+b\right)^2}{2\|\uu\|^2}} - e^{-\frac{\left(v-b\right)^2}{2\|\uu\|^2}}\right]\frac{\uu}{\|\uu\|}.$
\end{lemma}
\begin{lemma}
\label{lem:expectation_T3_aux2}
Let $\tx\sim \mathcal{N}(0, I_{p \times p})$  and $v \in \mathbb{R}$ be uniformly sampled from $[s,t]$ for some $ t > s \geq 0$, and is independent of $\tx$. For any given $\uu \in \mathbb{R}^p, b\in \mathbb{R}$, consider the following random vector: $\A =  \mathbb{I}\left(|b + \left\langle \tx, \uu \right\rangle | < v\right)b\tx.$ Then the expected value of $\A$ satisfies: $\E\left[\A\right] = c\frac{b}{(t-s)}\uu, $ for some $c$ such that $|c| \leq 1$.
\end{lemma}
We are now ready to present the main result on expectation of $T_{3,j}$.
\begin{lemma}[Expectation of $T_{3,j}$]
\label{lem:expectation_T3}
Conditioned on $\dw_{t-1}$, the following holds w.p. at least $1-1/\tn^{10}$: 
\[
\|\mathbb{E}_{\tX_t}\left[T_{3,j}|\dw_{t-1}, \epb_t\right]\|_2 =   O\left(\frac{\gamma \tn}{\log{\tn}}\right)\|\dw_{t-1}\|_2.
\]
\end{lemma}
\begin{proof}
	 Using the expression for $T_{3,j}$ in Equation~\eqref{eqn:t3j_rewritten}, we have
\[
\mathbb{E}_{\tX_t}\left[T_{3,j}|\dw_{t-1},\epb_t\right] = \sum_{i \in Q_j}  \E_{\tX_t}\left[\mathbb{I}\left[|\bb_{t}^{*}(i) + \epb_{t}(i) + \left\langle \tx_{t,i}, \dw_{t-1}\right\rangle | < \tau_{t,j,i}\right](\bb_{t}^{*}(i) + \epb_{t}(i))\tx_{t,i}\Big| \dw_{t-1},\epb_t\right],
\]
where $\tau_{t,j,i} = \left(j-\frac{1}{2}+\eta_{i,t}\right)\intl_{t}$. Invoking Lemma~\ref{lem:expectation_T3_aux2} with $\A_i$ being the $i^{th}$ term of $T_{3,j}$, we have: 
\[
\mathbb{E}_{\tX_t}\left[T_{3,j}| \dw_{t-1}, \epb_t\right] = \left[\sum_{i \in Q_j} c_i \frac{\bb_{t}^{*}(i) + \epb_{t}(i)}{\intl_t/9}\right] \dw_{t-1}.
\]
Moreover, $\forall i \in Q_j$, $|\bb_{t}^{*}(i)| \leq j\intl_t$ by definition. Furthermore, using standard concentration of sub-Gaussian random variables, we have:   $\|\epb_{t}\|_{\infty}\leq 4\sigma \sqrt{\log{\tn}} \leq 2/9\intl_t,$
 with probability at least $1-1/\tn^{10}$.
Using these two observations, we get (w.p. $\geq 1-1/\tn^{10}$):
\[
\|\mathbb{E}_{\tX_t}[T_{3,j}|\dw_{t-1},\epb_t]\|_2 \leq  11j|Q_j|\|\dw_{t-1}\|_2 = \left(\frac{11\gamma \tn}{\log{\tn}}\right)\|\dw_{t-1}\|_2.
\]
\end{proof}
\paragraph{Concentration of $T_{3,j}$.}
We first present some auxiliary Lemmas which will help us derive concentration results for $T_{3,j}$. The following Lemmas help us show that $T_{3,j}$ is a sub-Gaussian random variable. The proofs of these Lemmas can be found in Appendix~\ref{sec:aux_thm_real}.
\begin{lemma}
\label{lem:concentration_T3_aux1}
Let $\x\sim \mathcal{N}(0, I_{p \times p})$ be a random vector. For any given vector $\uu \in \mathbb{R}^p$, and scalars $b,v$, the following random vector is sub-Gaussian: $\A =  b  \mathbb{I}(|b + \left\langle \x, \uu \right\rangle | < v)\x.$ Moreover, there exists a universal constant $c > 0$, such that the following  holds for any $\T \in \mathbb{R}^p$: $\E\left[e^{\left\langle \T, \A-\E[\A]\right\rangle }\right] \leq e^{\frac{cb^2\|\T\|^2}{2}}.$
\end{lemma}
\begin{lemma}
\label{lem:concentration_T3_aux2}
Let $\tx\sim \mathcal{N}(0, I_{p \times p})$  and $v \in \mathbb{R}$ be uniformly sampled from $[s,t]$ for some $t > s \geq 0$ and is independent of $\tx$. For any given vector $\uu \in \mathbb{R}^p$, and $b\in \mathbb{R}$, the following random vector is sub-Gaussian: $\A =  b  \mathbb{I}(|b + \left\langle \x, \uu \right\rangle | <  v)\x.$ Moreover, we have: $\E\left[e^{\left\langle \T, \A-\E[\A]\right\rangle }\right] \leq e^{\frac{cb^2\|\T\|^2}{2}}, \quad \forall \T.$
\end{lemma}

\begin{lemma}
\label{lem:concentration_T3_aux3}
Let $\{\tx_i\}_{i=1}^{\tn}$ be independent samples from $\mathcal{N}(0, I_{p \times p})$ and $\{v_i\}_{i=1}^{\tn}$ be  independent samples from the uniform distribution on $[s,t]$ and are independent of $\{\tx_i\}_{i=1}^{\tn}$. For any given vectors $\bb\in \mathbb{R}^{\tn}, \uu \in \mathbb{R}^p$ and set $Q \subseteq [\tn]$, the following random vector is sub-Gaussian: $\A = \sum_{i \in Q} \bb(i)  \mathbb{I}(|\bb(i) + \left\langle \tx_{i}, \uu \right\rangle | < v_i)\tx_{i}.$
Moreover, with probability at least $1 - \frac{1}{\tn^p}$ we have: $\|\A - \mathbb{E}[\A]\|^2 \leq c_1p\log{\tn}\|\bb_Q\|^2_2,$
where $c_1$ is a universal constant.
\end{lemma}
\begin{lemma}[Concentration of $T_{3,j}$]
\label{lem:t3_concentration}
Conditioned on $\dw_{t-1}$, the following holds w.p. $\geq 1-1/\tn^{10}$: 
\[
\|T_{3,j}-\mathbb{E}_{\tX_{t}}\left[T_{3,j}|\dw_{t-1},\epb_t\right]\|_2 =   O\left(\sqrt{\gamma  p\tn^{1+1/\gamma}}\right)\intl_t.
\]
\end{lemma}
\begin{proof}
The proof follows similar path to the proof of Lemma~\ref{lem:concentration_T3_aux3}. 
From Equation~\eqref{eqn:t3j_rewritten} we have the following expression for $T_{3,j}$:
\[
T_{3,j} = \sum_{i \in Q_j}  \A_i,\quad \A_i := \mathbb{I}\left[|\bb_{t}^{*}(i) + \epb_{t}(i) + \left\langle \tx_{t,i}, \dw_{t-1}\right\rangle | < \tau_{t,j,i}\right](\bb_{t}^{*}(i) + \epb_{t}(i))\tx_{t,i}.
\]
From Lemma~\ref{lem:concentration_T3_aux3} we know that conditioned on $(\dw_{t-1}, \epb_{t})$, $T_{3,j}$ is a sub-gaussian random variable. Moreover, the following bound holds with probability at least $1-\frac{2}{\tn^{10}}$: 
\begin{align*}
         \|T_{3,j}-\E_{\tX_t}\left[T_{3,j}|\dw_{t-1},\epb_t\right]\|_2 &= O(\sqrt{p\log{\tn}})\|[\bb_t^*+\epb_t]_{Q_j}\|_2 = O(\sqrt{|Q_j|p\log{\tn}})\|[\bb_t^*+\epb_t]_{Q_j}\|_{\infty} \vspace{0.1in}\\
         &=  O\left(\sqrt{\gamma  \cdot p\cdot \tn^{1+1/\gamma}}\right)\intl_t,
\end{align*}
where the third equality follows from the fact that $|Q_j|\leq \frac{\gamma \tn}{j\log{\tn}}$ and $\|[\bb_1^*+\epb_1]_{Q_j}\|_{\infty} = O(j\intl_t)$ w.p. $\geq 1-1/\tn^{10}$; recall using sub-Gaussian tail bounds, we have (w.p. $\geq 1-1/\tn^{10}$):  $\|\epb_{t}\|_{\infty}\leq 4\sigma \sqrt{\log{\tn}} \leq 2/9\intl_t$.
\end{proof}
\subsubsection{Proof of Theorem~\ref{thm:real}}
\label{sec:aux_real_induction}
We now proceed to the proof of Theorem~\ref{thm:real}. 
To prove Theorem~\ref{thm:real}, we prove the following bound on $\|\dw_t\|_{\Sigma}$, which is slightly stronger than the bound in Theorem~\ref{thm:real}: 
\begin{equation}
    \label{eqn:descent_tighter}
         \displaystyle \|\dw_t\|_2 \leq \beta^{t}\|\dw_0\|_{2} + \left(\sum_{i=0}^{t-1}\zeta^i\right)O\left(\frac{\mu\sigma \tn^{1/\gamma}}{(1-\alpha)} \sqrt{\frac{p\log{\tn} + \log^2{\tn}}{\tn}}\right)
\end{equation}
Theorem~\ref{thm:real} directly follows from the above result by observing that $\sum_{i=1}^t\zeta^i \leq \frac{1}{1-\zeta}$, when $\zeta < 1$ and $\beta \geq \zeta$. 
We use induction on iteration $t$ to prove this result. 
\paragraph{Base Case ($t=1$).} 
First note that by definition of the interval length $\intl_{1}$ in Equation~\eqref{eqn:dist_ub}, we have
\begin{equation*}
 \intl_{1} = 18\sqrt{(2\hat{\sigma}^2 + 2\hat{d}_0^2)\log{\tn}} \geq  18\sqrt{(\sigma^2 + \|\dw_0\|^2)\log{\tn}}.
 \end{equation*}
So from Lemma~\ref{lem:real_intd_lemma} we have the following bound on $\|\dw_1\|_2$, which holds with probability at least $1-1/\tn^6$
$$\|\dw_1\|_2 = O\left(\frac{\gamma}{(1-\alpha)\log{\tn}}\right)\|\dw_{0}\|_2 + O\left(\frac{\tn^{1/\gamma}}{1-\alpha} \sqrt{\frac{p + \log{\tn}}{\tn}}\right)\intl_1 + O\left(\frac{\sigma}{1-\alpha}\sqrt{\frac{\alpha p\log{\tn}}{\tn}}\right).$$
Using definitions of $\intl_1, \beta$ in Theorem~\ref{thm:real}, and $\hat{\sigma} \leq \mu \sigma$,  $\hat{d}_0 \leq \nu \|\dw_0\|_2$ , and  $\tn^{1-2/\gamma}= \tilde{\Omega}(\nu^2 p/\gamma^2)$, we get:
\[
\|\dw_1\|_2 \leq \beta \|\dw_0\|_2 + O\left(\frac{\mu\sigma  \tn^{1/\gamma}}{1-\alpha} \sqrt{\frac{ p\log{\tn} + (\log{\tn})^2}{\tn}}\right).
\]
\paragraph{Induction Step.}  
 Suppose the Theorem holds for $t\leq \tT$, we show that it also holds for $t=\tT+1$, with high probability. We first show that $\intl_{\tT+1} \geq 18\sqrt{(\sigma^2+\|\dw_{\tT}\|_2^2) \log{\tn}}$. Consider the difference $\left(2\hat{\sigma}^2 + 2\beta^{2\tT}\hat{d}_0^2\right) - \left(\sigma^2 + \|\dw_{\tT}\|^2\right)$
\begin{equation*}
    \begin{array}{lll}
          2\hat{\sigma}^2 + 2\beta^{2\tT}\hat{d}_0^2 -   \left(\sigma^2 + \|\dw_{\tT}\|^2\right)
         &\geq&  \hat{\sigma}^2 + 2\beta^{2\tT}\hat{d}_0^2 - \|\dw_{\tT}\|^2 \\
         &\geq&  \hat{\sigma}^2 + 2\beta^{2\tT}\hat{d}_0^2 - 2\beta^{2\tT} \|\dw_0\|^2 - O\left(\frac{\mu\sigma  \tn^{1/\gamma}}{(1-\zeta)(1-\alpha)} \sqrt{\frac{ p\log{\tn} + (\log{\tn})^2}{\tn}}\right)^2\\
         &\geq& \hat{\sigma}^2 - O\left(\frac{\mu\sigma  \tn^{1/\gamma}}{(1-\zeta)(1-\alpha)} \sqrt{\frac{ p\log{\tn} + (\log{\tn})^2}{\tn}}\right)^2\\
         &\geq&0,
    \end{array}
\end{equation*}
where the first inequality follows from the fact that $\hat{\sigma} \geq \sigma$, the second inequality uses the bound on $\|\dw_{\tT}\|_2$ in inequality~\eqref{eqn:descent_tighter}, and the last inequality holds whenever $\tn^{1-2/\gamma} \geq c\frac{\mu^2(p\log{\tn}+(\log{\tn})^2)}{(1-\zeta)^2(1-\alpha)^2}$, for some universal constant $c$. This shows that $\intl_{\tT+1} \geq 18\sqrt{(\sigma^2+\|\dw_{\tT}\|_2^2) \log{\tn}}$, with probability at least $1-\tT/\tn^{6}$. 
We now use Lemma~\ref{lem:real_intd_lemma} to get the following bound on $\|\dw_{\tT+1}\|_2$
$$\|\dw_{\tT+1}\|_2 = O\left(\frac{\gamma}{(1-\alpha)\log{\tn}}\right)\|\dw_{\tT}\|_2 + O\left(\frac{\tn^{1/\gamma}}{1-\alpha} \sqrt{\frac{p + \log{\tn}}{\tn}}\right)\intl_{\tT+1} + O\left(\frac{\sigma}{1-\alpha}\sqrt{\frac{\alpha p\log{\tn}}{\tn}}\right).$$
Using definitions of $\intl_t, \beta$ in Theorem~\ref{thm:real}, and $\hat{\sigma} \leq \mu \sigma$,  $\hat{d}_0 \leq \nu \|\dw_0\|_2$ , and  the bound on $\|\dw_{\tT}\|_2$ from Equation~\eqref{eqn:descent_tighter}, we get:
\[
\|\dw_{\tT+1}\|_2 \leq \beta^{\tT+1} \|\dw_{0}\|_2 + \left(\sum_{i = 0}^{\tT}\zeta^i\right)O\left(\frac{\mu\sigma  \tn^{1/\gamma}}{1-\alpha} \sqrt{\frac{ p\log{\tn} + (\log{\tn})^2}{\tn}}\right).
\]
\section{Proof of Theorem~\ref{thm:real_hd}}
\label{sec:aux_proof_hd}
\begin{minipage}[t]{0.5\textwidth}
\begin{algorithm}[H]
\caption{\alg-HD}
\label{alg:torrent_hd}
\begin{algorithmic}[1]
  \small
  \STATE \textbf{Input:} Training data $(X, \y)$, iterations $T$, sparsity $k$
  \STATE Randomly split $(X, \y)$ into $T$ sets $\{(X_t, \y_t)\}_{t = 0}^T$ of size $\tn = \lfloor\frac{n}{T+1}\rfloor$ each
 \STATE $\w_0 \leftarrow 0$
  \STATE $t \leftarrow 1$ 
  \WHILE{$t \leq T$}
        \STATE \textbf{Get new set of samples} $(X_{t}, \y_{t})$
        \STATE $S_{t} \leftarrow \text{AdaHT}\left(\y_{t} - X_{t}\w_{t-1}\right)$
        \STATE $\w_t \leftarrow \textsf{IHT}\left(X_{t,S_t}, \y_{t,S_t}, k\right)$
        \STATE $t \leftarrow t + 1$
  \ENDWHILE
\end{algorithmic}
\end{algorithm}\vspace*{5pt}
\end{minipage}
\hspace*{10pt}
\begin{minipage}[t]{.48\textwidth}
\begin{algorithm}[H]
\caption{Iterative Hard Thresholding}
\label{alg:iht}
\begin{algorithmic}[1]
\small
    \STATE \textbf{Input:} $X, \y$, desired sparsity $k$, step size $\eta$.
    \STATE $\w_1 \leftarrow 0, t = 1$
    \WHILE{not converged}
    \STATE $\Tilde{\w}_{t} \leftarrow \w_t - \eta X^T(X\w_t - \y)$.
    \STATE $\w_{t+1} \leftarrow \displaystyle \argmin_{\w:\|\w\|_0 \leq k} \|\w - \Tilde{\w}_{t}\|_2$.
    \STATE $t \leftarrow t+1$.
    \ENDWHILE
    \STATE \textbf{return} $\w_t$
\end{algorithmic}
\end{algorithm}
\end{minipage}
 The \alg-HD algorithm for consistent robust sparse regression is given in Algorithm~\ref{alg:torrent_hd}. Before we present the proof of Theorem~\ref{thm:real_hd}, we first recall some notation and introduce some additional ingredients which we require for the proof of the Theorem.

\paragraph{Notation.} Recall that $\hat{\sigma}, \hat{d}_0$ are approximate upper bounds of $\sigma, \|\dw_0\|_2$ which satisfy the following inequalities
\[
\sigma \leq \hat{\sigma} \leq \mu \sigma, \quad \|\dw_0\|_2 \leq \hat{d}_0 \leq \nu \|\dw_0\|_2.
\]
The interval length we choose in this setting is given by: $ \intl_{t+1} = 18\sqrt{(2\hat{\sigma}^2 + 2\beta^{2t}\hat{d}_0^2)\log{p}}.$
Let $k^*$ be the sparsity of $\w^*$. The rest of the notation is same as in Theorems~\ref{thm:real}, which we recall here for convenience. Let $\tX = \Sigma^{-1/2}X$ and $j_t$ be the interval chosen in $t^{th}$ iteration. Let $\tau_t = \left(j_t-\frac{1}{2}\right)\intl_{t}$ be the midpoint of $j_{t}^{th}$ interval. Let $S_t^*$ be the set of un-corrupted points in $(X_t, \y_t)$ and $S_t$ be the output of \adaht in the $t^{th}$ iteration of \alg-HD. Let $\zeta = \frac{c\gamma}{(1-\alpha)\log{\tn}}$, for some universal constant $c> 0$.
\subsection{Background on Iterative Hard Thresholding (IHT)}
The IHT algorithm for solving the following sparse regression problem is given in Algorithm~\ref{alg:iht}
\[
\min_{\|\w\|_0 \leq k^*} \|\y - X\w\|_2^2.
\]
\citet{jain2014iterative} show that if the design matrix $X$ satisfies the  Restricted Strong Convexity (RSC) and Restricted Strong Smoothness (RSS) properties (defined below), then IHT can efficiently solve the above optimization problem.
\begin{defn}
\label{def:rsc}
\textbf{(RSC Property)}. A  matrix $X \in \mathbb{R}^{n \times p}$ is said to satisfy Restricted Strong Convexity (RSC) at sparsity level $k$ with strong convexity constraint $\alpha_{k}$ if the following holds for all $\w_1$, $\w_2$ s.t. $\|\w_1\|_0 \leq k_1$ and $\|\w_2\|_0 \leq k_2$, $k = k_1 + k_2$:
\[
\frac{1}{2n}\|X(\w_2-\w_1)\|_2^2 \geq \alpha_k\|\w_2-\w_1\|_2^2.
\]
\end{defn}
\begin{defn}
\label{def:rss}
\textbf{(RSS Property)}. A  matrix $X \in \mathbb{R}^{n \times p}$ is said to satisfy Restricted Strong Smoothness (RSS) at sparsity level $k$ with strong smoothness constraint $L_{k}$ if the following holds for all $\w_1$, $\w_2$ s.t. $\|\w_1\|_0 \leq k_1$ and $\|\w_2\|_0 \leq k_2$, $k = k_1 + k_2$:
\[
\frac{1}{2n}\|X(\w_2-\w_1)\|_2^2 \leq L_k\|\w_2-\w_1\|_2^2.
\]
\end{defn}
\begin{theorem}[Theorem 1, \citet{jain2014iterative}]
\label{thm:iht_guarantees}
 Let $X$ have RSS, RSC parameters given by $L_{2k+k^*} = L, \alpha_{2k+k^*} = \alpha$. Suppose IHT is run with $k \geq 32\left(\frac{L}{\alpha}\right)^2k^*$ and $\eta  = \frac{2}{3L}$. Then the $t^{th}$ iterate of the IHT algorithm, for $t = O\left(\frac{L}{\alpha} \log{\frac{\|\y\|_2}{\epsilon}}\right)$ satisfies
\[
\frac{1}{2n}\|\y - X\w_t\|_2^2 \leq \min_{\w:\|\w\|_0 \leq k^*}\frac{1}{2n}\|\y - X\w\|_2^2 + \epsilon,
\]
\end{theorem}
When the covariates $\{\x_i\}_{i = 1}^n$ are sampled from a Gaussian distribution, the following result of \citet{agarwal2010fast} implies that $X$ satisfies RSC and RSS properties with high probability.
\begin{theorem}[\citet{agarwal2010fast}]
Suppose the rows of $X \in \mathbb{R}^{n \times p}$ are sampled from a Gaussian distribution with covariance $\Sigma$. Then the following statements hold with probability at least $1-e^{-cn}$
\[
\forall \uu, \quad \frac{1}{n}\|X\uu\|_2^2 \geq \frac{1}{2}\|\Sigma^{1/2}\uu\|_2^2 - c_1\rho(\Sigma)\frac{\log{p}}{n}\|\uu\|_1^2,
\]
\[
\forall \uu, \quad \frac{1}{n}\|X\uu\|_2^2 \leq 4\|\Sigma^{1/2}\uu\|_2^2 + c_1\rho(\Sigma)\frac{\log{p}}{n}\|\uu\|_1^2,
\]
where $\rho(\Sigma) = \max_{i}\Sigma_{ii}$.
\end{theorem}
The following Corollary follows immediately from the above Theorem.
\begin{corollary}
\label{cor:gaussian_rsc}
Suppose the rows of $X \in \mathbb{R}^{n \times p}$ are sampled from a Gaussian distribution with covariance $\Sigma$. Then with probability at least $1-e^{-cn}$, $X$ satisfies RSC, RSS properties at sparsity level $k$, with strong convexity constraint $\alpha_k$ and strong smoothness constraint $L_k$ given by
\[
\alpha_k = \frac{1}{4}\lambda_{min}(\Sigma) - c_2\rho(\Sigma)\frac{k\log{p}}{n},
\]
\[
L_k = 2\lambda_{max}(\Sigma) + c_2\rho(\Sigma)\frac{k\log{p}}{n},
\]
where $\rho(\Sigma) = \max_{i}\Sigma_{ii}$.
\end{corollary}
\subsection{Main Argument}
We first prove the following more general result for any $\gamma \in (1, \log{\tn})$ and any $T > 0$. Theorem~\ref{thm:real_hd} then readily follows by substituting $\gamma = \frac{2\log{\tn}}{\log\log{\tn}}$ and $T = O\left(\log\left(\frac{\tn}{k}\frac{\|\w_0-\w^*\|_{\Sigma}}{\sigma}\right)\right)$ in the following Theorem.
\begin{theorem}
\label{thm:real_hd_intd}
Let $\{\x_i, y_i\}_{i = 1}^n$ be $n$ observations generated from the oblivious adversary model and let $\w^*$ be such that $\|\w^*\|_0 \leq k^*$. Suppose \alg-HD  is run for $T$ iterations.
Suppose the sparsity $k$ in the call to IHT is such that $k= \frac{\Omega(1)}{(1-\alpha)^4}\frac{\lambda_{max}(\Sigma)^2}{\lambda_{min}(\Sigma)^2}k^*$ and the hyperparameters of \adaht are set as follows:
\[
 a= 1/18, \quad \gamma \in (1, \log{\tn}),\quad  \intl_t = 18\sqrt{(2\hat{\sigma}^2 + 2\beta^{2(t-1)}\hat{d}_0^2)\log{p}},
 \]
  with $\beta \geq \zeta$, where $\tn = \frac{n}{T+1}$.
For any $\alpha < 1 - \frac{c\gamma}{\log{\tn}}$, let $\tn$
 be such that
$$\tn^{1-2/\gamma} \geq c_1\max\left\lbrace\frac{\mu^2}{(1-\zeta)^2(1-\alpha)^2}, \frac{\nu^2\log^2{\tn}}{\gamma^2}\right\rbrace k\log^2{p},$$
for some universal constant $c_1 > 0$.
   Then the iterates $\{\w_{t}\}_{t = 1}^T$ produced by \alg-HD satisfy
\begin{equation*}
         \displaystyle \|\w_{t}-\w^*\|_{\Sigma} \leq \beta^{t}\|\w_0-\w^*\|_{\Sigma}+ \left(\sum_{i = 0}^{t-1}\zeta^i\right)O\left(\frac{\mu\sigma \tn^{1/\gamma}}{(1-\alpha)} \sqrt{\frac{k\log^2{p}}{\tn}}\right).
\end{equation*}
with probability greater than $1-T/p^6$.
\end{theorem}
The proof of Theorem~\ref{thm:real_hd_intd} uses the same arguments as in the proof of Theorem~\ref{thm:real}. So we only focus on the  key differences in the proof.
That is, we use the same induction based argument for getting the final bound. So, we only need to prove the following lemma: 
\begin{lemma}
	Consider the setting of Theorem~\ref{thm:real_hd_intd}. Let $\intl_t  \geq  18\sqrt{(\sigma^2 + \|\dw_{t-1}\|^2)\log{p}}.$ Then, $\forall t\in [T]$, w.p. $\geq 1-1/p^6$: 
	 $$\|\dw_t\|_2 = O\left(\frac{\gamma}{(1-\alpha)\log{\tn}}\right)\|\dw_{t-1}\|_2 + O\left(\frac{\tn^{1/\gamma}}{1-\alpha} \sqrt{\frac{k\log{p}}{\tn}}\right)\intl_t + O\left(\frac{\sigma}{1-\alpha}\sqrt{\frac{\alpha k\log^2{p}}{\tn}}\right),$$
	 where $\dw_t=\Sigma^{1/2}(\w_t-\w^*)$. 
\end{lemma}
\begin{proof}
	Similar to proof of Theorem~\ref{thm:real}, we 
   divide $(X_{t}, \y_{t})$ into mutually exclusive sets $Q_1, Q_2, Q_3, Q_4$
  $$Q_1 = \left\lbrace i: |\bb_{t}^{*}(i)| >  \tau_{t} + \frac{5}{18}\intl_{t}\right\rbrace, \quad Q_2 = \left\lbrace i: |\bb_{t}^{*}(i)| <  \tau_{t} - \frac{5}{18}\intl_{t}\right\rbrace,$$ $$Q_3 = \left\lbrace i: |\bb_{t}^{*}(i) -\tau_{t}| \leq \frac{5}{18}\intl_{t}, \text{ and } |\y_{t}(i) - \left\langle \x_{t,i}, \w_{t-1}\right\rangle | \geq \tau_{t} + \eta_{i,t}\intl_{t}\right\rbrace,$$ $$Q_4 = \left\lbrace i:
|\bb_{t}^{*}(i) -\tau_{t}| \leq \frac{5}{18}\intl_{t}, \text{ and } |\y_{t}(i) - \left\langle \x_{t,i}, \w_{t-1}\right\rangle | < \tau_{t} + \eta_{i,t}\intl_{t}\right\rbrace.$$
Since $\intl_{t} \geq  18\sqrt{(\sigma^2 + \|\dw_{t-1}\|_2^2)\log{p}}$, using similar argument as in Lemma~\ref{prop:HT_properties}, it is easy to verify that the sets $Q_1, Q_2, Q_3, Q_4$ satisfy the following key properties, with probability at least $1-1/p^{10}$:
 \begin{enumerate}
     \item $Q_1 \cap S_{t}^* = \{\}, \quad Q_1 \cap S_{t} = \{\}, \quad S_{t}^* \subseteq Q_2 \subseteq S_{t}.$
 \item $S_{t} = Q_2 \cup Q_4$.
 \item $ |S_{t}| \geq |Q_2| \geq (1-\alpha)\tn$,
 \end{enumerate}
 We use the above properties to first show that the input $(X_{t,S_{t}},\y_{t,S_{t}})$ to IHT satisfies RSC, RSS properties. Since $S_{t}^* \subseteq S_{t}$ we have: 
 \begin{equation*}
          \displaystyle\frac{1}{|S_{t}|} \|X_{t,S_{t}}\uu\|_2^2 \geq \displaystyle\frac{|S_{t}^*|}{|S_{t}|} \left( \frac{1}{|S_{t}^*|} \|X_{t,S_{t}^*}\uu\|_2^2\right)  \geq  \displaystyle\frac{1-\alpha}{|S_{t}^*|} \|X_{t,S_{t}^*}\uu\|_2^2.
 \end{equation*}
 Since $S_{t}^*$ is chosen by an oblivious adversary, the rows of $X_{t,S_{t}^*}$ are i.i.d and follow a Gaussian distribution with covariance $\Sigma$. Using Corollary~\ref{cor:gaussian_rsc} on the RHS of the above equation we obtain
 \[
   \inf_{\uu:\|\uu\|_0 \leq k}\frac{1}{|S_{t}|} \frac{\|X_{t,S_{t}}\uu\|_2^2}{\|\uu\|_2^2} \geq (1-\alpha) \left(\frac{1}{4}\lambda_{min}(\Sigma) - c_2\rho(\Sigma)\frac{k\log{p}}{(1-\alpha)\tn}\right).
 \]
This shows that $X_{t,S_{t}}$ satisfies RSC property with $\alpha_k=(1-\alpha) \left(\frac{1}{4}\lambda_{min}(\Sigma) - c_2\rho(\Sigma)\frac{k\log{p}}{(1-\alpha)\tn}\right)$. A similar argument shows that $X_{t,S_{t}}$ satisfies RSS property
 \begin{equation*}
          \displaystyle\frac{1}{|S_{t}|} \|X_{t,S_{t}}\uu\|_2^2 \leq \displaystyle\frac{\tn}{|S_{t}|} \left( \frac{1}{\tn} \|X_{t}\uu\|_2^2\right)
          \leq \displaystyle\frac{\|X_{t}\uu\|_2^2}{(1-\alpha)\tn}
          \leq \displaystyle \frac{1}{1-\alpha}\left(2\lambda_{max}(\Sigma) + c_2\rho(\Sigma)\frac{k\log{p}}{(1-\alpha)\tn}\right).
 \end{equation*}
We now use the convergence properties of IHT presented in Theorem~\ref{thm:iht_guarantees} to obtain a bound on $\|\dw_{t}\|_2$.
Suppose the sparsity $k$ in the call to IHT in Algorithm~\ref{alg:torrent} is such that $k \geq \frac{c}{(1-\alpha)^4}\kappa^2(\Sigma)k^*$, where $\kappa(\Sigma)\coloneqq \frac{\lambda_{max}(\Sigma)}{\lambda_{min}(\Sigma)}$. Then,
from Theorem~\ref{thm:iht_guarantees} we know that $\w_{t}$, the output of IHT, satisfies
 \begin{equation*}
    \displaystyle\frac{1}{2\tn}\|\y_{t,S_{t}}-X_{t,S_{t}}\w_{t}\|^2_2\leq  \displaystyle\min_{\|\w\|_0 \leq k^*}\frac{1}{2\tn}\|\y_{t,S_{t}} - X_{t,S_{t}}\w\|^2_2 \leq \displaystyle\frac{1}{2\tn}\|\y_{t,S_{t}} - X_{t,S_{t}}\w^*\|^2_2.
 \end{equation*}
 Rearranging terms in the above expression gives us
 \begin{equation}
 \label{eqn:hd_descent_intd1}
     \begin{array}{lll}
 \|\tX_{t,S_{t}}\dw_{t}\|_2^2 \leq 2\left\langle [\bb_{t}^*+\epb_{t}]_{S_{t}}, \tX_{t,S_{t}}\dw_{t}\right\rangle.
     \end{array}
 \end{equation}
We now use the Restricted Eigenvalue property of $\tX_{t,S_{t}^*}$ to lower bound the LHS of the above equation. Using Lemma~\ref{lem:re_tight} in Appendix~\ref{sec:std_concentration} we obtain the following lower bound on $\|\tX_{t,S_{t}}\dw_{t}\|_2^2$, which holds with probability at least $1-1/p^{10}$
 \[
 \frac{1}{|S_{t}^*|}\|\tX_{t,S_{t}}\dw_{t}\|_2^2 \geq  \frac{1}{|S_{t}^*|} \|\tX_{t,S_{t}^*}\dw_{t}\|_2^2 \geq \|\dw_{t}\|_2^2\left(1 - \sqrt{\frac{c(\Sigma)(k+k^*)\log{p}}{(1-\alpha)\tn}}\right)^2,
 \]
 for some constant $c(\Sigma)$, which depends on $\Sigma$. We now bound RHS of \eqref{eqn:hd_descent_intd1}. Let $K_{t}$ denote the set of non-zero indices of $\w_{t}-\w^*$ and $\w_{t-1}-\w^*$: 
\[
K_{t} = \text{supp}(\w_{t}-\w^*) \cup \text{supp}(\w_{t-1}-\w^*).
\]
Note that $|K_{t}| \leq 2(k+k^*)$.
 Let $X_{t,S_{t},K_{t}} \in \mathbb{R}^{|S_{t}|\times |K_{t}|}$ be a sub-matrix obtained from $X_{t}$ by selecting rows corresponding to $X_{t}$ and columns corresponding to $K_{t}$. And let $\Sigma_{K_{t}}$ be the sub-matrix of $\Sigma$ restricted to rows and columns corresponding to $K_{t}$.
 \begin{multline*}
 \left\langle \bb_{t}^*+\epb_{t}, \tX_{t,S_{t}}\dw_{t}\right\rangle = \left\langle X_{t,S_{t}}^T[\bb_{t}^*+\epb_{t}]_{S_{t}}, \w_{t}-\w^*\right\rangle  = \left\langle X_{t,S_{t},K_{t}}^T[\bb_{t}^*+\epb_{t}]_{S_{t}}, [\w_{t}-\w^*]_{K_{t}}\right\rangle\\
 \leq \|\tX_{t,S_{t},K_{t}}^T[\bb_{t}^*+\epb_{t}]_{S_{t}}\|_2\|\Sigma_{K_{t}}^{1/2}[\w_{t}-\w^*]_{K_{t}}\|_2 
 = \|\tX_{t,S_{t},K_{t}}^T[\bb_{t}^*+\epb_{t}]_{S_{t}}\|_2 \|\dw_{t}\|_2.
 \end{multline*}
 Plugging the previous two results in \eqref{eqn:hd_descent_intd1}, and using assumption about $\tn$, we get (w.p. $\geq 1-2/p^{10}$):
\[
\|\dw_{t}\|_2 \leq \frac{c}{(1-\alpha)\tn} \|\tX_{t,S_{t},K_{t}}^T[\bb_{t}^*+\epb_{t}]_{S_{t}}\|_2,
\]
for some universal constant $c$.
The rest of the proof focuses on bounding $\|\tX_{t,S_{t},K_{t}}^T[\bb_{t}^*+\epb_{t}]_{S_{t}}\|_2$ and is similar to the proofs of Theorem~\ref{thm:real}.
Using sets $Q_{1}, Q_2, Q_3, Q_4$, we rewrite $\tX_{t,S_{t},K_{t}}^T[\bb_{t}^*+\epb_{t}]_{S_{t}}$ as
\begin{equation*}
        \tX_{t,S_{t},K_{t}}^T[\bb_{t}^*+\epb_{t}]_{S_{t}} =  -\underbrace{\left(\sum_{i \in Q_2}\bb_{t}^*(i) \tx_{t, i, K_{t}}\right)}_{T_{1}}
         -\underbrace{\left(\sum_{i \in Q_2}\epb_{t}(i) \tx_{t, i, K_{t}}\right)}_{T_2} -\underbrace{\left(\sum_{i \in Q_4}\left(\bb_{t}^*(i) + \epb_{t}(i) \right)\tx_{t, i, K_{t}}\right)}_{T_3},
\end{equation*}
So we have: 
\[
\|\dw_{t}\|_2 \leq \frac{c}{(1-\alpha)\tn}\left(\|T_1\|_2 + \|T_2\|_2 + \|T_3\|_2\right).
\]
We now bound each of the terms in the RHS of the above equation. Note that since $K_{t}$ is a random quantity, we take a union bound over all possible ${p \choose |K_{t}|}$ choices of $K_{t}$ while bounding these terms. So we have an additional $\log{p}$ term in our bounds.

Using similar techniques as in the proof of Theorem~\ref{thm:real} for bounding $\|T_1\|_2, \|T_2\|_2, \|T_3\|_2$, we can show that the following hold with probability at least $1-1/p^7$: 	
\begin{equation}
\label{eqn:t1_bound_base_real_hd}
\frac{1}{\tn}\|T_1\|_2 = O\left(\tn^{1/\gamma}\sqrt{\frac{k\log{p}}{\tn}}\right)\intl_t\,, \quad
\frac{1}{\tn}\|T_2\|_2 =  O\left(\sigma\sqrt{\frac{\alpha k\log^2{p}}{\tn}}\right).
\end{equation}
\begin{equation}
\label{eqn:t3_bound_base_real_hd}
\begin{array}{lll}
     \frac{1}{\tn}\|T_3\| &\leq& \displaystyle \frac{1}{\tn}\|\mathbb{E}[T_3|\dw_{t-1}, \epb_{t}]\|_2 + \frac{1}{\tn}\|T_3 - \mathbb{E}[T_3|\dw_{t-1}, \epb_{t}]\|_2  \vspace{0.1in}\\
     &=&  \displaystyle O\left(\frac{\gamma }{\log{\tn}}\right)\|\dw_{t-1}\|_2 + O\left(\tn^{1/2\gamma}\sqrt{\frac{\gamma k\log{p}}{\tn}}\right)\intl_t.
\end{array}
\end{equation}
Combining the bounds in Equations~\eqref{eqn:t1_bound_base_real_hd} and~\eqref{eqn:t3_bound_base_real_hd}, we get the following bound, which holds with probability at least $1-1/p^6$: 
\[
\|\dw_{t}\|_2 = O\left(\frac{\gamma}{(1-\alpha)\log{\tn}}\right)\|\dw_{t-1}\|_2 + O\left(\frac{\tn^{1/\gamma}}{1-\alpha} \sqrt{\frac{k\log{p}}{\tn}}\right)\intl_t + O\left(\frac{\sigma}{1-\alpha}\sqrt{\frac{\alpha k\log^2{p}}{\tn}}\right)
\]
\end{proof}

\section{Proof of Theorem~\ref{thm:heavy_tails}}\label{sec:app_heavy}
We first introduce some notation that we use in the proof.
\paragraph{Notation.} Recall that $\hat{d}_0$ is an approximate upper bound of $\|\dw_0\|_2$ which satisfies the following inequality: $\|\dw_0\|_2 \leq \hat{d}_0 \leq \nu \|\dw_0\|_2.$
Define the tail probability of the noise distribution as: $\alpha_{\rho} = \mathbb{P}(|\epsilon| > \rho).$
Note that the key idea behind adapting Algorithm~\ref{alg:torrent} to heavy-tailed setting is to consider all the points with noise magnitude less than $\rho$ as un-corrupted points and try to use only these points to estimate the parameter vector. Accordingly, we define $S_t^*$, the set of ``un-corrupted'' points in $(X_t, \y_t)$  as $S_t^* = \{i: |\epb_t(i)| \leq \rho\},$
where $\epb_t$ is the noise vector corresponding to the points in $(X_t, \y_t)$.
Let $S_t$ be the output of Algorithm~\ref{alg:ht_oracle} in the $t^{th}$ iteration of Algorithm~\ref{alg:torrent}.
The rest of the notation is same as in Theorem~\ref{thm:real}, which we recall next. Finally, we define $\zeta := \frac{c\gamma}{(1-\alpha_{\rho})\log{\tn}}$, for some constant $c > 0$.
\subsection{Intermediate Results}
The following Lemma, which is similar to Lemma~\ref{prop:HT_properties}, provides condition on $\intl_t$ which ensures all the ``uncorrupted'' points lie to the left of the $j_t^{th}$ interval.
\begin{lemma}[Interval Length]
\label{prop:HT_properties_heavy}
Consider the $t^{th}$ iteration of \alg-FC. Suppose \adaht is run with the interval length $\intl_t$ such that $
\intl_t \geq 18\left(\frac{\rho}{4} + \|\dw_{t-1}\|\sqrt{\log{\tn}}\right),$
and $a = 1/18$ and $\gamma \in (1, \log{\tn})$.
Define sets $Q_1, Q_2$, which are subsets of points in $(X_{t}, \y_{t})$, as follows:
\[
Q_1 = \left\lbrace i: |\epb_{t}(i)| >  (j_t - 2/9)\intl_t \right\rbrace
\quad \text{and} \quad
Q_2 = \left\lbrace i: |\epb_{t}(i)| < (j_t - 7/9)\intl_t  \right\rbrace.
\]
 Then the following statements hold with probability at least $1-1/\tn^7$:
 \[
 Q_1 \cap S_t = \{\}, \quad S_t^* \subseteq Q_2 \subseteq S_t.
 \]
 Moreover, all the points in $(Q_1\cup Q_2)^c$ fall in the $j_t^{th}$ interval.
\end{lemma}
\begin{proof}
The proof of the Lemma uses the exact same arguments as in the proof of Lemma~\ref{prop:HT_properties} and relies on the following bound on the residual of uncorrupted points
\[
\|[\y_t-X_t\w_{t-1}]_{S_t^*}\|_{\infty} = \|[\epb_{t}+\tX_{t}\dw_{t-1}]_{S_1^*}\|_{\infty} \leq \rho + \|\tX_{t}\dw_{t-1}\|_{\infty} \leq \rho + 4\|\dw_{t-1}\|_2\sqrt{\log{\tn}},
\]
where the last inequality holds with probability at least $1-1/\tn^7$ and follows from the concentration properties of sub-Gaussian random variables.
\end{proof}

\subsection{Main Argument}
We first prove the following more general result for any $\gamma \in (1, \log{\tn})$ and any $T > 0$. Theorem~\ref{thm:heavy_tails} then readily follows from this by substituting $\gamma = \frac{2\log{\tn}}{\log\log{\tn}}$ and $T = O\left(\log\left(\frac{\tn}{p}\frac{\|\w_0-\w^*\|_{\Sigma}}{\rho}\right)\right)$.
\begin{theorem}
\label{thm:heavy_intd}
Let $\{\x_i, y_i\}_{i = 1}^n$ be $n$ observations generated from a linear model with heavy tailed noise. Suppose \alg-FC is run for $T$ iterations.
Suppose that the hyperparameters of \adaht are set as follows:
\[
 a= 1/18, \quad \gamma \in (1, \log{\tn}),\quad  \intl_t = 18\left(\frac{\rho}{\sqrt{8}} + \beta^{t-1}\hat{d}_0\sqrt{\log{\tn}}\right),
 \]
  with $\beta \geq \zeta$, where $\tn = \frac{n}{T+1}$.
For any $\alpha_{\rho} < 1 - \frac{c\gamma}{\log{\tn}}$, let $\tn$
 be such that
$$\tn^{1-2/\gamma} \geq c_1\max\left\lbrace\frac{1}{(1-\zeta)^2(1-\alpha_{\rho})^2}, \frac{\nu^2\log^2{\tn}}{\gamma^2}\right\rbrace (p\log{\tn} + \log^2{\tn}),$$
for some universal constant $c_1 > 0$.
   Then the iterates $\{\w_{t}\}_{t = 1}^T$ produced by \alg-FC satisfy:
\begin{equation}
\label{eqn:heavy_descent}
         \displaystyle \|\w_{t}-\w^*\|_{\Sigma} \leq \beta^{t}\|\w_0-\w^*\|_{\Sigma}+ \left(\sum_{i = 0}^{t-1}\zeta^i\right)O\left(\frac{\rho \tn^{1/\gamma}}{(1-\alpha_{\rho})} \sqrt{\frac{p + \log{\tn}}{\tn}}\right).
\end{equation}
with probability greater than $1-T/\tn^6$.
\end{theorem}
The proof of the Theorem uses the same arguments as in the proof of Theorem~\ref{thm:real}. Specifically, we use the same induction based argument for getting the final bound. So, we only need to prove the following lemma.
\begin{lemma}
	Consider the setting of Theorem~\ref{thm:heavy_intd}. Let $\intl_t \geq 18\left(\frac{\rho}{4} + \|\dw_{t-1}\|\sqrt{\log{\tn}}\right).$ Then, $\forall t\in [T]$, w.p. $\geq 1-1/\tn^6$: 
	 $$\|\dw_t\|_2 = O\left(\frac{\gamma}{(1-\alpha_{\rho})\log{\tn}}\right)\|\dw_{t-1}\|_2 + O\left(\frac{\tn^{1/\gamma}}{1-\alpha_{\rho}} \sqrt{\frac{p+\log{\tn}}{\tn}}\right)\intl_t,$$
	 where $\dw_t=\Sigma^{1/2}(\w_t-\w^*)$. 
\end{lemma}
\begin{proof}
Consider the $t^{th}$ iteration of \alg-FC. We divide $(X_{t}, \y_{t})$ into mutually exclusive sets $Q_1, Q_2, Q_3, Q_4$
$$Q_1 = \left\lbrace i: |\epb_{t}(i)| >  \tau_{t} + \frac{5}{18}\intl_{t}\right\rbrace, \quad Q_2 = \left\lbrace i: |\epb_{t}(i)| <  \tau_{t} - \frac{5}{18}\intl_{t}\right\rbrace,$$ $$Q_3 = \left\lbrace i: |\epb_{t}(i) -\tau_{t}| \leq  \frac{5}{18}\intl_{t}, \text{ and } |\y_{t}(i) - \left\langle \x_{t,i}, \w_{t-1}\right\rangle | \geq \tau_{t} + \eta_{i,t}\intl_t\right\rbrace,$$ $$Q_4 = \left\lbrace i:
|\epb_{t}(i) -\tau_{t}| \leq \frac{5}{18}\intl_{t}, \text{ and } |\y_{t}(i) - \left\langle \x_{t,i}, \w_{t-1}\right\rangle | < \tau_{t} + \eta_{i,t}\intl_t\right\rbrace.$$
Since $\intl_1 \geq 18\left(\frac{\rho}{4} + \|\dw_{t-1}\|\sqrt{\log{\tn}}\right)$, the sets defined above satisfy the following properties (this follows from Lemma~\ref{prop:HT_properties_heavy})
     $$Q_1 \cap S_{1}^* = \{\}, \quad Q_1 \cap S_{1} = \{\}, \quad S_{1}^* \subseteq Q_2 \subseteq S_{t}, \quad S_1 = Q_2 \cup Q_4.$$
     Using sets $Q_1, Q_2, Q_3, Q_4$, we now rewrite $\dw_{t}$ as
\begin{equation*}
    \begin{array}{lll}
        \dw_{t} &=&  -(\tX_{t,S_{t}}^T\tX_{t,S_{t}})^{-1}\underbrace{\left(\sum_{i \in Q_2}\epb_{t}(i) \tx_{t, i}\right)}_{T_1}
         -(\tX_{t,S_{t}}^T\tX_{t,S_{t}})^{-1}\underbrace{\left(\sum_{i \in Q_4}\epb_{t}(i) \tx_{t, i}\right)}_{T_2}.
    \end{array}
\end{equation*}
So we have
\[
\|\dw_t\|_2 \leq \frac{1}{\lb_{min}\left(\tX_{t,S_{t}}^T\tX_{t,S_{t}}\right)}\left(\|T_1\|_2 + \|T_2\|_2\right).
\]
Note that $T_1$ above corresponds to $T_1$ that appears in the proof of Lemma~\ref{lem:real_intd_lemma}. We now use similar techniques as in the proof  of Lemma~\ref{lem:real_intd_lemma} for bounding $\|T_1\|_2$.
\paragraph{Bounding $T_1$.}
Define $Q_{2,j}, T_{1,j}$ as follows
$$
Q_{2,j} = \{i: |\epb_{t}(i)| < (j-2/9)\intl_{t}\}, \quad T_{1,j} = \sum_{i \in Q_{2,j}}\epb_{t}(i) \tx_{1, i}.
$$
Note that $Q_2 = Q_{2,j_{t}}$, where $j_t$ is the bucket chosen by \alg-FC in $t^{th}$ iteration of Algorithm~\ref{alg:torrent}. 
 $\|T_1\|_2$ can be upper bounded as
\[
\|T_1\|_2 \leq \sup_{j \in [\tn^{1/\gamma}]} \|T_{1,j}\|_2.
\]
Note that the covariates ($\x$) in $Q_{2,j}$ are still distributed according to Gaussian distribution. This follows from the fact that $Q_{2,j}$ is formed based on the noise magnitude $|\epb_{t}(i)|$, which is independent of the covariates. We use this observation to derive an upper bound for $\|T_{1,j}\|$.
Using chi-square concentration result from Lemma~\ref{lem:chi_conc} in Appendix~\ref{sec:std_concentration}, we obtain the following upper bound for $T_{1,j}$, which holds with probability at least $1-\delta$
\[
\|T_{1,j}\|_2^2 \leq \left(p + O\left(\sqrt{p\log{\frac{1}{\delta}}} + \log{\frac{1}{\delta}} \right)\right)\|[\epb_{t}]_{Q_{2,j}}\|^2.
\]
Combining this result with the upper bound on $|\epb_{t}(i)|$, we obtain the following bound on $T_{1,j}$, which holds with probability at least $1-1/\tn^8$
\begin{multline}
     \|T_{1,j}\|_2 = O\left(\sqrt{p+\log{\tn}}\right)\|[\epb_{t}]_{Q_{2,j}}\|_2 = O\left(\sqrt{(p+\log{\tn})|Q_{2,j}|} \right)\|[\epb_{t}]_{Q_{2,j}}\|_{\infty}\\
     = O\left(\tn^{1/2}j\sqrt{p+\log{\tn}}\right)\intl_t=  O\left(\tn^{1/2+1/\gamma}\sqrt{p+\log{\tn}}\right)\intl_t
\end{multline}
where the third equality follows from the fact that $|Q_{2,j}| \leq \tn$.
This shows that with probability at least $1-1/\tn^7$
\[
\|T_1\|_2 =O\left(\tn^{1/2+1/\gamma}\sqrt{p+\log{\tn}}\right)\intl_t.
\]
\paragraph{Bounding $T_2$.}
$T_2$ corresponds to the term $T_3$ that appears in the proof of Lemma~\ref{lem:real_intd_lemma}. Using similar techniques as in Lemma~\ref{lem:real_intd_lemma}, $T_2$ can be bounded as
\[
\|T_2\|_2 = O\left(\frac{\gamma }{\log{\tn}}\right)\|\dw_{t-1}\|_2 + O\left(\tn^{1/2\gamma}\sqrt{\frac{\gamma (p+\log{\tn})}{\tn}}\right)\intl_t,
\]
which holds with probability at least $1-1/\tn^{8}$.
Combining the above bounds for $\|T_1\|, \|T_2\|$, we get the following bound on $\|\dw_{t}\|_2$, which holds with probability at least $1-1/\tn^6$
\[
\|\dw_t\|_2 = O\left(\frac{\gamma}{(1-\alpha_{\rho})\log{\tn}}\right)\|\dw_{t-1}\|_2 + O\left(\frac{\tn^{1/\gamma}}{1-\alpha_{\rho}} \sqrt{\frac{p + \log{\tn}}{\tn}}\right)\intl_t.
\]
\end{proof}
\section{Proof of Theorem~\ref{thm:ht}}
 Let $\mathcal{E}$ be the even that $|\epsilon| > \rho$. From the definition of $\alpha_{\rho}$, we know that $\mathbb{P}(\mathcal{E}) = \alpha_{\rho}$. We now lower bound $\mathbb{E}[|\epsilon|^{\delta}]$ as follows: 
\begin{equation*}
\mathbb{E}[|\epsilon|^{\delta}] = \mathbb{P}(\mathcal{E})\mathbb{E}\left[|\epsilon|^{\delta} | \mathcal{E}\right] +  \left(1-\mathbb{P}(\mathcal{E})\right)\mathbb{E}\left[|\epsilon|^{\delta}| \mathcal{E}^c\right] 
\geq \mathbb{P}(\mathcal{E})\mathbb{E}\left[|\epsilon|^{\delta} | \mathcal{E}\right]
\geq \alpha_{\rho}\rho^{\delta}.
\end{equation*}
Since $\mathbb{E}[|\epsilon|^{\delta}] = C$, we have: $\rho \leq \left(\frac{C}{\alpha_{\rho}}\right)^{1/\delta}.$
Finally, to prove the Theorem, we substitute the above bound on $\rho$ in Theorem~\ref{thm:heavy_tails}.
\section{Proof of Corollary~\ref{cor:heavy_tails_cauchy}}
The CDF of a Cauchy random variable $z$ with location parameter $0$ and scale parameter $\sigma$ is given by: 
$\mathbb{P}(z < \rho) = \frac{1}{\pi}\tan^{-1}\left(\frac{\rho}{\sigma}\right) + \frac{1}{2}.$  So, the tail probability $\alpha_{\rho}$ is given by
$\alpha_{\rho} = 1 - \frac{2}{\pi}\tan^{-1}\left(\frac{\rho}{\sigma}\right).$
This can equivalently be written as: $
\rho = \sigma\tan\left(\frac{(1-\alpha_{\rho})\pi}{2}\right)$. 
Replacing $\rho$ with $\sigma\tan\left(\frac{(1-\alpha_{\rho})\pi}{2}\right)$ and setting $\alpha_{\rho} = 1/2$ in the result of Theorem~\ref{thm:heavy_tails} proves the Corollary.


\section{1-d Heavy-tailed Mean Estimation}\label{sec:app_mean}
In this section we consider the 1d heavy-tailed mean estimation problem and  provide intuition for why \alg-FC estimates the mean of any distribution which is symmetric around its mean as remarked in Section~\ref{sec:heavy}. Let $\{x_i, y_i\}_{i=1}^n$, where $x_i=1$, be the input to \alg-FC. Note that $(x_i,y_i)$ are related as
\[
y_i = x_iw^* + \epsilon_i,
\]
where $\epsilon_i$ is independent of $x_i$ and $\mathbb{E}[\epsilon_i] = 0$.
Next, observe that \alg-FC is ``sign invariant''; \emph{i.e.,} changing $(x_i, y_i)$ to $(-x_i, -y_i)$ doesn't change the course of the algorithm and the iterates produced by the algorithm will exactly be the same in both cases. So we can randomly choose $n/2$ data points and flip their sign without effecting the outcome of the algorithm. Let $\{x_i',y_i'\}_{i=1}^n$ be the resulting data points, where $(x_i', y_i')$ are related as follows
\[
y_i' = x_i'w^* + \epsilon_i'.
\]
Note that the distribution of $x_i'$ is given by
\[
\mathbb{P}(x_i' = 1) = \mathbb{P}(x_i' = -1) = \frac{1}{2}.
\]
Since the distribution of $\epsilon$ is symmetric around origin, $\epsilon_i'$ has the same distribution as $\epsilon_i$.
Moreover, $\epsilon_i'$ remains independent of $x_i'$ even after the transformation; a simple calculation shows that $P(x_i' = a, \epsilon_i' = b) = P(x_i' = a)P(\epsilon_i' = b)$. This shows that $\{x_i', y_i'\}_{i=1}^n$ can be viewed as being generated from a different linear model in which the covariates are sampled from a sub-Gaussian distribution with mean $0$ and variance $1$ and the noise is independent of covariates and is sampled from a heavy tailed distribution. So, we can use the exact same arguments as in the proof of Theorem~\ref{thm:heavy_tails}, where we assumed the covariates are Gaussian, to show that \alg-FC estimates the mean of any symmetric heavy tailed distribution, with bounded first moment, at sub-gaussian rates.

\section{Additional Results}
\label{sec:aux_additional}

\subsection{Estimating $\|\dw_0\|_2$}
In this section we present techniques to estimate $\|\dw_0\|_2 = \|\w_0-\w^*\|_{\Sigma}$ in both $n > p$ and $n < p$ settings.

\subsubsection{Low dimensional regression ($n > p$)}
We begin with low dimensional setting where $n > p$ and provide techniques to estimate $\|\dw_0\|_2$ for two different initializations of $\w$.
In the following Proposition we provide a constant factor upper bound for $\|\dw_0\|_2$, when $\w$ is intialized at the OLS solution.
\begin{proposition}
\label{prop:dw0_estimate_ols}
Suppose we start Algorithm~\ref{alg:torrent} at $\w_0 = (X_{0}^TX_{0})^{-1}X_{0}^T\y_{0}$. Consider the following estimate for $\|\dw_0\|_2$:
\[
\hat{d}_0 = \frac{2c_{p,n}\sqrt{p}}{\tn}\left(\|\y_0-X_0\w_0\|_2\right),
\]
where $c_{p,n} = 4\sqrt{\frac{p\log{\tn}}{\tn}}$. If $\tn$ is such that $\tn < e^{cp}$, for some universal constant $c$, then $\hat{d}_0$ satisfies the following inequalities with probability at least $1-2/\tn^{10}$: $$\|\dw_0\|_2 \leq \hat{d}_0 \leq 4\left(1+2c_{p,n}\right)\|\dw_0\|_2.$$
\end{proposition}

\begin{proof}
Recall that we start our algorithm at $\w_0 = (X_{0}^TX_{0})^{-1}X_{0}^T\y_{0}$. So $\dw_0$ is given by
\[
\dw_0 = \Sigma^{1/2}\left(\w_0 - \w^*\right) = \left(\tX_0^T\tX_0\right)^{-1}\tX_0^T(\bb^*_0 + \epb_0),
\]
where $\tX_0 = X_0\Sigma^{-1/2}$.
We first study $\|\dw_0\|_2$ and understand its lower and upper bounds. We use concentration properties of Gaussian random variables to derive these bounds. $\|\dw_0\|$ can be written as
\begin{equation}
\label{eqn:prop_dw0_intd1}
\begin{array}{lll}
    \|\dw_0\|_2 &=& \Big|\Big|\left(\tX_0^T\tX_0\right)^{-1}\tX_0^T(\bb^*_0 + \epb_0)\Big|\Big|_2  \vspace{0.1in}\\
     &=&  \Big|\Big|\underbrace{\frac{1}{\tn}\tX_0^T(\bb^*_0 + \epb_0)}_{T_1} + \underbrace{\left(\left(\tX_0^T\tX_0\right)^{-1}-\frac{1}{\tn}I\right)\tX_0^T(\bb^*_0 + \epb_0)}_{T_2}\Big|\Big|_2
\end{array}
\end{equation}
So we have
\[
\|T_1\|_2 - \|T_2\|_2 \leq \|\dw_0\|_2 \leq \|T_1\|_2 + \|T_2\|_2.
\]
Note that, conditioned on $\bb^*_0,\epb_0$, the elements of $T_1$ are i.i.d Gaussian random variables with mean $0$ and variance $\tn^{-2}\|\bb^*_0+\epb_0\|_2^2$.
Using Chi-squared concentration results (see Lemma~\ref{lem:chi_conc}) we get the following lower and upper bounds for $\|T_1\|_2$, which hold with probability at least $1-1/\tn^{10}$, for any $\bb^*_0, \epb_0$
\[
\frac{p-8\sqrt{p\log{\tn}}}{\tn^2}\|\bb^*_0+\epb_0\|_2^2 \leq \|T_1\|_2^2 \leq \frac{p+8\sqrt{p\log{\tn}}}{\tn^2}\|\bb^*_0+\epb_0\|_2^2.
\]
Note that $T_2$ acts as a remainder term and is of a smaller order than $T_1$. To bound $T_2$, we first consider the term $\left(\frac{1}{\tn}\tX_0^T\tX_0\right)^{-1}-I$. Let $\hat{\Sigma}_0 = \frac{1}{\tn}\tX_0^T\tX_0$. Then,
\[
\| \hat{\Sigma}_0^{-1}-I \|_2 \leq \max\left\lbrace |1-\lambda_{max}(\hat{\Sigma}_0^{-1})|, |1-\lambda_{min}(\hat{\Sigma}_0^{-1})| \right\rbrace.
\]
Using concentration properties of eigenvalues of covariance matrix (see Lemma~\ref{lem:ss}), we get the following bound on $\| \hat{\Sigma}_0^{-1}-I \|_2$, which holds with probability at least $1-e^{p/4}$:
\[
\| \hat{\Sigma}_0^{-1}-I \|_2 \leq \sqrt{\frac{p}{\tn}}.
\]
This gives us the following upper bound for $\|T_2\|$, which holds with probability at least $1-1/\tn^{10}$:
\[
\|T_2\|_2^2 \leq \frac{p}{\tn}\|T_1\|_2^2 \leq \left(\frac{p^2+8p\sqrt{p\log{\tn}}}{\tn^3}\|\bb^*_0+\epb_0\|_2^2\right).
\]
We now substitute the above bounds on $T_1, T_2$ in Equation~\eqref{eqn:prop_dw0_intd1}. Simplifying the resulting terms and using the assumption that $\tn < e^{cp}$, we get the following bound for $\|\dw_0\|_2$, which holds with probability at least $1-1/\tn^{10}$:
\begin{equation}
    \label{eqn:prop_dw0_intd2}
    \frac{1}{2}\left(1-2\sqrt{\frac{p}{\tn}}\right)\frac{\sqrt{p}}{\tn}\|\bb^*_0 + \epb_0\|_2 \leq \|\dw_0\|_2 \leq 2\left(1+\sqrt{\frac{p}{\tn}}\right)\frac{\sqrt{p}}{\tn}\|\bb^*_0 + \epb_0\|_2.
\end{equation}
The above bound shows that to estimate $\|\dw_0\|_2$, it suffices to have a good estimate of $\|\bb^*_0 + \epb_0\|_2$.  We now show that $\|\y_0 - X_0\w_0\|_2$ is a good estimate of $\|\bb^*_0 + \epb_0\|_2$. Note that $\y_0 - X_0\w_0$ can written as:
\[
\y_0 - X_0\w_0 = (I-P_{X_0})(\bb^*_0 + \epb_0),
\]
where $P_{X_0}$ is the projection matrix onto the column span of $X_0$ given by $P_{X_0}=X_0\left(X_0^TX_0\right)^{-1}X_0^T$. Note that $P_{X_0}$ is a random projection matrix which projects any given $n$-dimensional vector onto a random $p$-dimensional subspace, with every subspace equally likely. So for any fixed vector $\uu$, $\|P_{X_0}\uu\|_2$ has the same distribution as $\|\uu\|_2\|\vb(1:p)\|_2$, where $\vb$ is sampled uniformly from the unit sphere in $\mathbb{R}^n$ and $\vb(1:p)$ is the subvector of $\vb$ corresponding  the first $p$ coordinates.  We now use Lemma 2.2 of \citet{dasgupta2003elementary} which provides a high probability bound on $\|\vb(1:p)\|_2$, to get the following bound on $\|P_{X_0}(\bb^*_0 + \epb_0)\|_2$, which holds with probability at least $1-1/\tn^{p/2}$:
\[
\|P_{X_0}(\bb^*_0 + \epb_0)\|_2 \leq \sqrt{\frac{p\log{\tn}}{\tn}}\|\bb^*_0 + \epb_0\|_2.
\]
This gives us the following bound on the norm of the residual vector $\y_0 - X_0\w_0$:
\[
\left(1-\sqrt{\frac{p\log{\tn}}{\tn}}\right)\|\bb^*_0 + \epb_0\|_2 \leq \|\y_0 - X_0\w_0\|_2 \leq \left(1+\sqrt{\frac{p\log{\tn}}{\tn}}\right)\|\bb^*_0 + \epb_0\|_2.
\]
Substituting this in Equation~\eqref{eqn:prop_dw0_intd2} and simplifying the resulting terms, we get the following bounds for $\|\dw_0\|_2$ in terms of the residual vector $y_0 - X_0\w_0$
\[
\frac{1}{2}\left(1-4\sqrt{\frac{p\log{\tn}}{\tn}}\right)\frac{\sqrt{p}}{\tn}\|\y_0 - X_0\w_0\| \leq \|\dw_0\|_2 \leq 2\left(1+4\sqrt{\frac{p\log{\tn}}{\tn}}\right)\frac{\sqrt{p}}{\tn}\|\y_0 - X_0\w_0\|.
\]
\end{proof}
 We now consider the case where $\w_0 = 0$. In this case, $\|\dw_0\|_2 = \|\w^*\|_{\Sigma}$. So the problem of estimating $\|\dw_0\|$ is equivalent to the problem of estimating the signal strength $\|\w^*\|_{\Sigma}$. This is a well studied problem in statistics and a number of estimators have been proposed, which work under fairly mild conditions on the distributions of $\x,\epsilon$~\citep[see][and references therein]{dicker2014variance}. Here we consider the following estimator for $\|\w^*\|_{\Sigma}$:
\[
\hat{d}_0^2 = \frac{1}{\tn}\left(\|\y_0\|^2 - \|\y_0-X_0\w_{\text{OLS}}\|^2\right).
\]
The following Proposition shows that $\hat{d}_0$ provides a good approximation of $\|\w^*\|_{\Sigma}$, when the noise is not too strong compared to the signal strength.
\begin{proposition}
\label{prop:dw0_estimate_origin}
Suppose the noise and corruptions are such that
\begin{equation}
\label{eqn:aux_dw0_condition}
\left(\frac{\|\bb_0^*\|}{\sqrt{\tn}} + \sigma \right) \leq \frac{\epsilon}{2}\left(\frac{\tn}{p\log{\tn}}\right)^{1/2}\|\w^*\|_{\Sigma},
\end{equation}
for some $\epsilon \in (0,1)$. Then $\hat{d}_0$ satisfies the following inequality with probability at least $1-2/\tn^{p}$:
\[
\left(1 - \epsilon - 2\left(\frac{p\log{\tn}}{\tn}\right)^{1/4}\right)\|\w^*\|_{\Sigma}\leq \hat{d}_0 \leq \left(1 + \epsilon + 2\left(\frac{p\log{\tn}}{\tn}\right)^{1/4}\right)\|\w^*\|_{\Sigma}.
\]
\end{proposition}
\begin{proof}
First note that $\hat{d}_0^2$ can be rewritten as
\[
\hat{d}_0^2 = \frac{1}{\tn}\|P_{X_0}\y_0\|_2^2 = \frac{1}{\tn}\|X_0\w^* + P_{X_0}(\bb_0^* + \epb_0)\|_2^2,
\]
where $P_{X_0}$ is the projection matrix onto the column span of $X_0$ given by $P_{X_0}=X_0\left(X_0^TX_0\right)^{-1}X_0^T$.
This gives us the following upper and lower bound for $\hat{d}_0$:
\[
\frac{1}{\sqrt{\tn}}\left(\|X_0\w^*\|_2 - \|P_{X_0}(\bb_0^* + \epb_0)\|_2\right) \leq \hat{d}_0 \leq \frac{1}{\sqrt{\tn}}\left(\|X_0\w^*\|_2 + \|P_{X_0}(\bb_0^* + \epb_0)\|_2\right).
\]
We now bound each of the terms in the upper and lower bounds of $\hat{d}_0$
\begin{itemize}
\item Recall that in Proposition~\ref{prop:dw0_estimate_ols} we showed the following bound for $\|P_{X_0}(\bb^*_0 + \epb_0)\|_2$ which holds with probability at least $1-1/\tn^{p/2}$:
\[
\|P_{X_0}(\bb^*_0 + \epb_0)\|_2 \leq \sqrt{\frac{p\log{\tn}}{\tn}}\|\bb^*_0 + \epb_0\|_2 \leq 2\sqrt{\frac{p\log{\tn}}{\tn}}\left(\|\bb_0^*\|_2 + \sigma\sqrt{\tn}\right),
\]
where the last inequality follows from concentration properties of chi-square random variables (see Lemma~\ref{lem:chi_conc}) and holds with probability at least $1-1/e^{\tn}$.
\item Note that each entry of $X_0\w^*$ is a Gaussian random variable with mean $0$ and variance $\|\w^*\|_{\Sigma}$. Using concentration properties of chi-square random variables (Lemma~\ref{lem:chi_conc}), we get the following bounds for $\frac{1}{\sqrt{\tn}}\|X_0\w^*\|_2$,
which hold with probability at least $1-1/\tn^{p/2}$:
\[
\left(\sqrt{1 - \sqrt{\frac{2p\log{\tn}}{\tn}}}\right)\|\w^*\|_{\Sigma} \leq \frac{1}{\sqrt{\tn}}\|X_0\w^*\|_2 \leq \left(\sqrt{1 + \sqrt{\frac{2p\log{\tn}}{\tn}}}\right)\|\w^*\|_{\Sigma}.
\]
\end{itemize}
Substituting these bounds in the previous equation gives us the following upper and lower bounds on $\hat{d}_0$, which hold with probability at least $1-2/\tn^{p/2}$:
\[
\hat{d}_0 \leq \|\w^*\|_{\Sigma} + 2\left(\frac{p\log{\tn}}{\tn}\right)^{1/4}\left(\|\w^*\|_{\Sigma} + \left(\frac{p\log{\tn}}{\tn}\right)^{1/4}\left(\frac{\|\bb_0^*\|}{\sqrt{\tn}} + \sigma \right) \right),
\]
\[
\hat{d}_0 \geq \|\w^*\|_{\Sigma} - 2\left(\frac{p\log{\tn}}{\tn}\right)^{1/4}\left(\|\w^*\|_{\Sigma} + \left(\frac{p\log{\tn}}{\tn}\right)^{1/4}\left(\frac{\|\bb_0^*\|}{\sqrt{\tn}} + \sigma \right)\right).
\]
Suppose the noise and corruptions are such that:
\[
\left(\frac{\|\bb_0^*\|}{\sqrt{\tn}} + \sigma \right) \leq \frac{\epsilon}{2}\left(\frac{\tn}{p\log{\tn}}\right)^{1/2}\|\w^*\|_{\Sigma},
\]
for some $\epsilon \in (0,1)$.
Then we get the following upper and lower bounds for $\hat{d}_0$:
\[
\left(1-\epsilon - 2\left(\frac{p\log{\tn}}{\tn}\right)^{1/4}\right)\|\w^*\|_{\Sigma} \leq \hat{d}_0 \leq \left(1 + \epsilon + 2\left(\frac{p\log{\tn}}{\tn}\right)^{1/4}\right)\|\w^*\|_{\Sigma}
\]
\end{proof}
\paragraph{Discussion.} When $\tn \to \infty$, we can see that $\frac{1}{1-\epsilon}\hat{d}_0$ is a good upper bound of $\|\w^*\|_{\Sigma}$
\[
\|\w^*\|_{\Sigma}\leq \frac{1}{1-\epsilon}\hat{d}_0 \leq \frac{1+\epsilon}{1-\epsilon}\|\w^*\|_{\Sigma}.
\]
Condition~\eqref{eqn:aux_dw0_condition} holds when the noise and corruptions aren't too strong compared to  the signal strength. It imposes a bound on the norm of corruptions and requires $\|\bb_0^*\|_2$ to be bounded by $\tn \|\w^*\|_{\Sigma}$. This is a very mild assumption and holds even if the adversary adds $O(\sqrt{\tn})$ corruptions to each data point.

\subsubsection{Sparse regression ($n < p$)}
In this setting we consider the case where \alg-HD is initialized at $0$ and estimate $\|\dw_0\|_2$ which is equal to the signal strength $\|\w^*\|_{\Sigma}$. The problem of estimating signal strength and noise variance in sparse regression setting is well studied~\citep[see][and references therein]{sun2012scaled, fan2012variance}. In this work we use the estimator of \citet{sun2012scaled} to first estimate the variance of $\epsilon+\bb^*$, \emph{i.e.,} $\left(\sigma^2 + \frac{\|\bb_0^*\|^2_2}{\tn}\right)$, and then use it to estimate the signal strength $\|\w^*\|_{\Sigma}$. The estimator of \citet{sun2012scaled} solves the following scaled sparse linear regression problem to estimate the noise variance
\[
(\hat{\w}_{\lambda}, \hat{\sigma}_{\lambda})\in \argmin_{\w, \sigma'} \frac{1}{2\tn\sigma'}\|\y_0-X_0\w\|_2^2 + \frac{\sigma'}{2} + \lambda\|\w\|_1.
\]
For $\lambda \geq \sqrt{\frac{2\log{p}}{\tn}}$, Theorem 1 of \citet{sun2018adaptive} shows that $\hat{\sigma}^2_{\lambda}$ is a good  estimate of variance $\left(\sigma^2 + \frac{\|\bb_0^*\|_2^2}{\tn}\right)$. To be more precise, $\hat{\sigma}_{\lambda}$ satisfies the following bound:
\begin{equation}
\label{eqn:scaled_lasso_error}
\max\left(1-\frac{\hat{\sigma}_{\lambda}}{\sigma^*}, 1- \frac{\sigma^*}{\hat{\sigma}_{\lambda}}\right) \leq c\sqrt{k^*}\lambda,
\end{equation}
where $\sigma^* = \sqrt{\sigma^2 + \frac{\|\bb_0^*\|_2^2}{\tn}}$ and $c$ is a universal constant. Using $\hat{\sigma}_{\lambda}$ we estimate $\|\w^*\|_{\Sigma}$ as:  $$\hat{d}_0^2 = \frac{1}{\tn}\|\y_0\|_2^2 - \hat{\sigma}^2_{\lambda}.$$
The following Proposition, which is similar to Proposition~\ref{prop:dw0_estimate_origin}, shows that $\hat{d}_0$ is a good estimate of $\|\w^*\|_{\Sigma}$.
\begin{proposition}
Let the noise $\epb_0$ be sampled from $\mathcal{N}(0, \sigma^2I_{\tn \times \tn})$. For any $\lambda \geq \sqrt{\frac{2\log{p}}{\tn}}$, suppose the Gaussian noise and corruptions are such that,
\begin{equation*}
\sigma^2 + \frac{\|\bb_0^*\|^2_2}{\tn} \leq  \frac{\epsilon}{\sqrt{k^*}\lambda}\|\w^*\|^2_{\Sigma},
\end{equation*}
for some $\epsilon \in (0,1)$. Then $\hat{d}_0$ satisfies the following inequality w.p $\geq 1-1/p^{10}$:
\[
\left(1-c\epsilon-c\sqrt{\frac{\log{p}}{\tn}}\right)\|\w^*\|_{\Sigma}^2\leq \hat{d}_0^2 \leq \left(1+c\epsilon + c\sqrt{\frac{\log{p}}{\tn}}\right)\|\w^*\|_{\Sigma}^2,
\]
for some universal constant $c > 0$.
\end{proposition}
\begin{proof}
We begin by deriving lower and upper bounds for $\|\y_0\|_2^2$. Let $\vb = X_0\w^* + \epb_0$. Then $\y_0 = \bb_0^* + \vb$ and $\|\y_0\|_2^2$ can be written as:
\[
\|\y_0\|_2^2 = \|\bb_0^*\|_2^2 + \|\vb\|_2^2 + 2\vb^T\bb_0^*.
\]
Note that the elements of $\vb$ are independent and $\vb(i) \in \mathcal{N}(0, \sigma^2 + \|\w^*\|^2_{\Sigma})$. Moreover, $\vb^T\bb_0^* \in \mathcal{N}(0, \|\bb_0^*\|_2^2(\sigma^2 + \|\w^*\|^2_{\Sigma}))$. Using concentration properties of Gaussian and chi-squared random variables (see Lemma~\ref{lem:chi_conc}), we get the following bounds which hold with probability at least $1-1/p^{10}$:
\begin{equation*}
\begin{array}{c}
\|\y_0\|_2^2 \leq \left(1+c\sqrt{\frac{\log{p}}{\tn}}\right)\left(\|\bb_0^*\|_2^2 +  \tn(\sigma^2 + \|\w^*\|^2_{\Sigma})\right),\\
\|\y_0\|_2^2 \geq \left(1-c\sqrt{\frac{\log{p}}{\tn}}\right)\left(\|\bb_0^*\|_2^2 +  \tn(\sigma^2 + \|\w^*\|^2_{\Sigma})\right)
\end{array}
\end{equation*}
for some universal constant $c > 0$. Using the above bounds together with Equation~\eqref{eqn:scaled_lasso_error} gives us the following upper and lower bounds for $\hat{d}_0^2$:
\begin{equation*}
\begin{array}{c}
\hat{d}_0^2 \leq \left(1+c\sqrt{\frac{\log{p}}{\tn}}\right)\|\w^*\|_{\Sigma}^2 + c\sqrt{k^*}\lambda\sigma^{*2},\\
\hat{d}_0^2 \geq \left(1-c\sqrt{\frac{\log{p}}{\tn}}\right)\|\w^*\|_{\Sigma}^2 - c\sqrt{k^*}\lambda\sigma^{*2}.
\end{array}
\end{equation*}
The Theorem follows by observing that $\sqrt{k^*}\lambda\sigma^{*2} \leq \epsilon \|\w^*\|_{\Sigma}^2$.
\end{proof}
\subsection{Analysis of \alg-GD}
In this section we provide an informal argument for why \alg-GD obtains similar guarantees as \alg-FC. Here, we work in the oblivious adversary model described in Section~\ref{sec:setup}. Consider the update step of \alg-GD described in Algorithm~\ref{alg:torrent_gd}:
\[
\w_{t} = \w_{t-1} - \eta X_{t,S_{t}}^T(X_{t,S_{t}}\w_{t-1} - \y_{t,S_{t}}).
\]
Letting $\dw_t = \Sigma^{1/2}(\w_{t} - \w^*)$ and $\tX = X\Sigma^{-1/2}$, we rewrite the above equation as:
\[
\dw_{t} = \left(I - \eta \Sigma \tX_{t,S_{t}}^T\tX_{t,S_{t}}\right)\dw_{t-1} + \eta \Sigma \tX_{t,S_{t}}^T(\bb^*_{t,S_t} + \epb_{t,S_{t}}).
\]
\paragraph{One step progress.} We now use the same proof technique as used in the proof of Theorem~\ref{thm:real} (see Lemma~\ref{lem:real_intd_lemma}) to obtain a bound on the one-step progress made by \alg-GD.
We first bound $\|\dw_t\|$ follows:
\[
\|\dw_{t}\|_2 \leq \|I - \eta \Sigma \tX_{t,S_{t}}^T\tX_{t,S_{t}}\|_2\|\dw_{t-1}\|_2 + \|\eta\Sigma\|_2 \|\tX_{t,S_{t}}^T(\bb^*_{t,S_t} + \epb_{t,S_{t}})\|_2.
\]
Supose $\eta = \frac{1}{\tn \lambda_{max}(\Sigma)}$. Then from Lemma~\ref{lem:ss} on concentration of largest eigenvalue of covariance matrix, and the fact that $S_{t}^* \subseteq S_t$, we get the following bound (w.p $\geq 1-e^{-p}$): $$\|I - \eta \Sigma \tX_{t,S_{t}}^T\tX_{t,S_{t}}\|_2 \leq \left(1- \frac{1-\alpha}{\kappa(\Sigma)}\right),$$
where $\kappa(\Sigma) = \frac{\lambda_{max}(\Sigma)}{\lambda_{min}(\Sigma)}$.

Lets suppose the interval length we choose in \adaht is such that $\intl_{t} \geq 18\sqrt{(\sigma^2 + \|\dw_{t-1}\|_2^2)\log{\tn}}$. Then from Lemma~\ref{prop:HT_properties} we know that none of the uncorrupted points get thresholded by \adaht, in the $t^{th}$ iteration of \alg-GD. Now, consider the sets $Q_1, Q_2, Q_3, Q_4$ defined in the proof of  Lemma~\ref{lem:real_intd_lemma}. From Lemma~\ref{prop:HT_properties} we also know that $S_{t} = Q_2 \cup Q_4$. We now decompose $\|\tX_{t,S_{t}}^T(\bb^*_{t,S_t} + \epb_{t,S_{t}})\|_2$ in terms of $Q_2, Q_4$ as follows:
\[
\tX_{t,S_{t}}^T(\bb^*_{t,S_t} + \epb_{t,S_{t}}) = \underbrace{\left(\sum_{i \in Q_2}\bb_{t}^*(i) \tx_{t, i}\right)}_{T_1}
         -\underbrace{\left(\sum_{i \in Q_2}\epb_{t}(i) \tx_{t, i}\right)}_{T_2}
         -\underbrace{\left(\sum_{i \in Q_4}\left(\bb_{t}^*(i) + \epb_{t}(i) \right)\tx_{t, i}\right)}_{T_3}.
\]
Note that these are the exactly the same terms which appear in the proof of Lemma~\ref{lem:real_intd_lemma}. Using the results from Lemma~\ref{lem:real_intd_lemma}, we get the following bound on $\|\dw_t\|_2$ (w.p.$\geq 1-1/\tn^6$):
\begin{equation*}
\begin{array}{c}
\|\dw_{t}\|_2 \leq \left(1- \frac{1-\alpha}{\kappa(\Sigma)} + O\left(\frac{\gamma}{\log{\tn}}\right)\right)\|\dw_{t-1}\|_2 +  O\left(\tn^{1/\gamma} \sqrt{\frac{p + \log{\tn}}{\tn}}\right)\intl_t + O\left(\sigma\sqrt{\frac{\alpha p\log{\tn}}{\tn}}\right).
\end{array}
\end{equation*}
\paragraph{Bound on $\|\dw_t\|_2$. }
Now, suppose the interval length $\intl_t$ is chosen as in Equation~\eqref{eqn:dist_ub}, with $\beta \geq 1- \frac{1-\alpha}{\kappa(\Sigma)} +\frac{c\gamma}{\log{\tn}}$, for some universal constant $c > 0$. Then, using the same induction argument as in the proof of Theorem~\ref{thm:real}, we get the following bound on $\|\w_t\|_2$:
\begin{equation*}
\begin{array}{c}
\|\w_t-\w^*\|_{\Sigma} \leq \beta^{t}\|\w_0-\w^*\|_{\Sigma} + O\left(\frac{\mu\sigma \tn^{1/\gamma}}{(1-\beta)(1-\alpha)}\sqrt{\frac{p\log{\tn} + \log^2{\tn}}{\tn}}\right)
\end{array}
\end{equation*}
\paragraph{Discussion.} The above bound shows that \alg-GD achieves similar error guarantees as \alg-FC in oblivious adversary model. However, the fraction of corruptions that \alg-GD can tolerate depends on the condition number $\kappa(\Sigma)$
\[
\alpha \leq 1 - \frac{c\kappa(\Sigma)\gamma}{\log{\tn}}.
\]

\section{Lemmas for Theorem~\ref{thm:real}}
\label{sec:aux_thm_real}
\subsection{Proof of Lemma~\ref{lem:update_rewrite}}
From the definition of $\rb_t$ we know that: $\rb_{t} = \y_{t} - X_{t}\w_{t-1}.$ 
Substituting $\y_t$ with $X_t\w^* + \bb_t^* + \epb_t$ we get: 
\[
\rb_{t} = X_t(\w^* - \w_{t-1}) + \bb^*_t + \epb_t = \bb_t^* + \tX_t \Sigma^{1/2}(\w^* - \w_{t-1}) + \epb_t = \bb_t^* + \tX_t \dw_{t-1} + \epb_t.
\]

\subsection{Proof of Lemma~\ref{lem:bucket_number}}
The proof is based on a simple counting argument. Suppose $j_t$ is greater than $\tn^{1/\gamma}$. Let $n_j$ be the number of points in bucket $j$. We know that $\forall j \leq \tn^{1/\gamma}$, $n_j > \frac{\gamma\tn }{j\log{\tn}}$. The number of points in the first $\tn^{1/\gamma}$ buckets can be lower bounded as: 
\[
\left(\sum_{j=1}^{\tn^{1/\gamma}}n_j\right) >  \sum_{j = 1}^{\tn^{1/\gamma}} \frac{\gamma\tn }{j\log{\tn}} = \frac{\gamma\tn }{\log{\tn}}\sum_{j = 1}^{\tn^{1/\gamma}}\frac{1}{j} \geq  \frac{\gamma\tn }{\log{\tn}}\log{\tn^{1/\gamma}} = \tn.
\]
However, the total number of points in $(X_t,\y_t)$ is only $\tn$. This shows that $j_t$ can't be greater than $\tn^{1/\gamma}$.

\subsection{Proof of Lemma~\ref{prop:HT_properties}}
Let $\rb_{t}$ be the input to \adaht in $t^{th}$ iteration of \alg-FC: $\rb_{t} = \y_{t} - X_{t}\w_{t-1} = \bb^{*}_{t} + \tX_{t}\dw_{t-1} + \epb_{t}.$
Let $\rb_t^*= \tX_{t}\dw_{t-1} + \epb_{t}$. Since $\tX_t$ is independent of $\epb_t$ and since sum of two independent sub-gaussian random variables is a sub-gaussian random variable, it is easy to see that $\rb_t^*(i)$ is sub-gaussian and satisfies the following tail bounds: 
\[
\mathbb{P}(|\rb_t^*(i)| \geq t) \leq 2e^{-\frac{t^2}{2\left(\sigma^2 + \|\dw_{t-1}\|_2^2\right)}}.
\]
So we have the bound on $\|\rb_t^*\|_{\infty}$, which holds with probability at least $1-1/\tn^7$: 
\[
\|\rb_t^*\|_{\infty}= \|\tX_{t}\dw_{t-1} + \epb_{t}\|_{\infty} \leq 4\sqrt{(\sigma^2+\|\dw_{t-1}\|^2) \log{\tn}}
\]

\paragraph{Set $Q_1$.} Now, consider the residual of points in $Q_1$
\begin{multline*}
         \min_{i\in Q_1}|\rb_t(i)| = \min_{i \in Q_1}|[\bb^{*}_{t}+\tX_{t}\dw_{t-1} + \epb_{t}](i)|
         \geq \min_{i \in Q_1}|\bb^{*}_{t}(i)| - \|\tX_{t}\dw_{t-1} + \epb_{t}\|_{\infty}\\
         > (j_t-2/9)\intl_t-4\sqrt{(\sigma^2+\|\dw_{t-1}\|^2) \log{\tn}}
         \geq (j_t-4/9)\intl_t,
\end{multline*}
where the second inequality follows from the definition of $Q_1$ and the above concentration bound on $\|\rb_t^*\|_{\infty}$, and the last inequality follows from the fact that $\intl_t \geq 18\sqrt{(\sigma^2+\|\dw_{t-1}\|^2) \log{\tn}}$. We now show that any point with a residual larger than $(j_t-4/9)\intl_t$ will never be thresholded; that is, the point will never be added to $S_t$. Any point with residual larger than $(j_t-4/9)\intl_t$ can either lie to the right or inside the $j_t^{th}$ interval. If it is to the right, then it will not be added to $S_t$. If it lies in the interval, we uniformly sample $\eta \in [-1/18,1/18]$ and add the point to $S_t$ only if: $|\rb_t(i)| < (j_t-1/2 + \eta)\intl_t.$ Clearly, this can never hold for the points in $Q_1$.
This shows that $Q_1 \cap S_t = \{\}$.

\paragraph{Set $Q_2$.} Now, consider the residual of points in $Q_2$: 
\begin{multline*}
         \|[\rb_t]_{Q_2}\|_{\infty} =\|[\bb^{*}_{t}+\tX_{t}\dw_{t-1} + \epb_{t}]_{Q_2}\|_{\infty} 
         \leq \|[\bb^{*}_{t}]_{Q_2}\|_{\infty} + \|\tX_{t}\dw_{t-1} + \epb_{t}\|_{\infty} \\
         < (j_t-7/9)\intl_t + 4\sqrt{(\sigma^2+\|\dw_{t-1}\|^2) \log{\tn}} 
         \leq (j_t-5/9)\intl_t.
\end{multline*}
Using a similar argument as above, we can show that any point with residual smaller than $(j_t-5/9)\intl_t$ will always be added to $S_t$. This shows that
$Q_2 \subseteq S_t$. Since $S_t^*$ is a subset of $Q_2$, we get: $S_t^* \subseteq Q_2 \subseteq S_t.$ 
\paragraph{Set $(Q_1 \cup Q_2)^c$.} Let $\tau_t$ be the center of $j_t^{th}$ interval, which is given by: $\tau_t = (j_t - 1/2)\intl_t.$ 
Note that $(Q_1 \cup Q_2)^c$ is given by: $(Q_1 \cup Q_2)^c = \left\lbrace i: |\bb_t^*(i) - \tau_t| \leq \frac{5}{18}\intl_t\right\rbrace.$
We first obtain an upper bound for the residual of points in $(Q_1 \cup Q_2)^c$: 
\begin{multline*}
\|[\rb_t]_{(Q_1 \cup Q_2)^c}\|_{\infty} = \|[\bb_t^* +\tX_{t}\dw_{t-1} + \epb_{t} ]_{(Q_1 \cup Q_2)^c}\|_{\infty}   
\leq \|[\bb_t^*]_{(Q_1 \cup Q_2)^c}\|_{\infty} + \|\tX_{t}\dw_{t-1} + \epb_{t}\|_{\infty} \\
\leq \left(\tau_t + \frac{5}{18}\intl_t\right) + 4\sqrt{(\sigma^2+\|\dw_{t-1}\|^2) \log{\tn}} 
\leq j_t\intl_t,
\end{multline*}
Next we obtain a lower bound for the residual of points in $(Q_1 \cup Q_2)^c$: 
\begin{multline*}
\|[\rb_t]_{(Q_1 \cup Q_2)^c}\|_{\infty} = \|[\bb_t^* +\tX_{t}\dw_{t-1} + \epb_{t} ]_{(Q_1 \cup Q_2)^c}\|_{\infty}   
\geq \min_{i \in (Q_1 \cup Q_2)^c}|\bb_t^*(i)|- \|\tX_{t}\dw_{t-1} + \epb_{t}\|_{\infty} \\
\geq \left(\tau_t - \frac{5}{18}\intl_t\right) - 4\sqrt{(\sigma^2+\|\dw_{t-1}\|^2) \log{\tn}} 
\geq (j_t-1)\intl_t.
\end{multline*}
This shows that all the points in $(Q_1 \cup Q_2)^c$ fall in the $j_t^{th}$ interval.
This finishes the proof of the Proposition.

\subsection{Proof of Lemma~\ref{lem:expectation_T3_aux1}}
First note that $\tx$ can be rewritten as a sum of two independent random variables: 
\[
\tx = z\frac{\uu}{\|\uu\|} + \z_{\bot},
\]
where $z \sim \mathcal{N}(0,1), \z_{\bot} \sim \mathcal{N}(0, I - \frac{\uu\uu^T}{\|\uu\|^2_2})$ and $z \indep z_{\bot}$. So $\A$ can be rewritten as: 
\[
\A = \underbrace{\mathbb{I}(|b + z\|\uu\| | < v)bz\frac{\uu}{\|\uu\|}}_{T_5} +\underbrace{  \mathbb{I}(|b + z\|\uu\| | < v)b\z_{\bot}}_{T_6}.
\]
Since $\mathbb{E}[\z_{\bot}] = 0$ and $z \indep z_{\bot}$, it is easy to see that $\mathbb{E}[T_6] = 0$. So we have: 
\begin{equation*}
    \begin{array}{lll}
         \E[\A] &=& \mathbb{E}[T_5] =\mathbb{E}\left[\mathbb{I}(|b + z\|\uu\| | < v)bz\right]\frac{\uu}{\|\uu\|}.
    \end{array}
\end{equation*}
Since $z$ is a standard normal random variable, we get the following closed form expression for $\E[\A]$: 
\[
\E[\A] = \frac{b}{\sqrt{2\pi}}\left[e^{-\frac{\left(v+b\right)^2}{2\|\uu\|^2}} - e^{-\frac{\left(v-b\right)^2}{2\|\uu\|^2}}\right]\frac{\uu}{\|\uu\|}.
\]

\subsection{Proof of Lemma~\ref{lem:expectation_T3_aux2}}
From Lemma~\ref{lem:expectation_T3_aux1} we know that: 
\[
\E\left[\A|v\right] = \frac{b}{\sqrt{2\pi}}\left[e^{-\frac{\left( v+b\right)^2}{2\|\uu\|^2}} - e^{-\frac{\left( v-b\right)^2}{2\|\uu\|^2}}\right]\frac{\uu}{\|\uu\|}.
\]
Using this expression we now compute $\mathbb{E}[\A]$: 
\begin{equation*}
    \begin{array}{lll}
        \E\left[\A\right] &=&  \frac{b}{\sqrt{2\pi}}\mathbb{E}\left[e^{-\frac{\left( v+b\right)^2}{2\|\uu\|^2}} - e^{-\frac{\left( v-b\right)^2}{2\|\uu\|^2}}\right]\frac{\uu}{\|\uu\|}  \vspace{0.1in}\\
        &=& \displaystyle \frac{1}{\sqrt{2\pi}}\left[\int_{z=\frac{s}{\|\uu\|}}^{\frac{t}{\|\uu\|}} e^{-\frac{\left(z+b/\|\uu\|\right)^2}{2}} - e^{-\frac{\left(z-b/\|\uu\|\right)^2}{2}} dz\right]\frac{b}{t-s} \uu= c\frac{b}{t-s} \uu,
    \end{array}
\end{equation*}
for some $c$ such that $|c| \leq 1$.

\subsection{Proof of Lemma~\ref{lem:concentration_T3_aux1}}
First note that $\x$ can be rewritten as: 
$\x = z\frac{\uu}{\|\uu\|} + \z_{\bot},$ 
where $z \sim \mathcal{N}(0,1), \z_{\bot} \sim \mathcal{N}(0, I - \frac{\uu\uu^T}{\|\uu\|^2_2})$ and $z \indep z_{\bot}$. So $\A$ can be rewritten as: 
\[
\A = \underbrace{b  \mathbb{I}(|b + z\|\uu\| | < v)z\frac{\uu}{\|\uu\|}}_{T_1} +\underbrace{ b  \mathbb{I}(|b + z\|\uu\| | < v)\z_{\bot}}_{T_3}.
\]
We now show that $T_1$, $T_3$ are both sub-Gaussian random vectors.
\paragraph{$\mathbf{T_1}$.} To show that $T_1$ is a sub-Gaussian random vector, it suffices to show that $\mathbb{I}(|b + z\|\uu\| | < v)z$ is a sub-Gaussian random variable. Let $T_4 = \mathbb{I}(|b + z\|\uu\| | < v)z$.
Note that, 
\[
\mathbb{P}\left(|T_4| \geq s\right) \leq \mathbb{P}\left(|z| \geq s\right) \leq e^{-s^2/2}.
\]
Some computation shows that we get the following tail bound for $T_4$, $\mathbb{P}\left(|T_4 - \mathbb{E}[T_4]| \geq s\right) \leq 2e^{-s^2/8}.$ 
This shows that for any $\T \in \mathbb{R}^p$, we have: 
\[
\mathbb{P}\left(|\left\langle \T, T_1 - \mathbb{E}[T_1]\right\rangle| \geq s\right)  = \mathbb{P}\left(|T_4 - \mathbb{E}[T_4]| \geq \frac{\|\uu\|}{|b||\left\langle \T, \uu \right\rangle|}s\right) \leq 2e^{-\frac{s^2}{8}\frac{\|\uu\|^2}{b^2|\left\langle \T, \uu \right\rangle|^2}}
\leq 2e^{-\frac{s^2}{8}\frac{1}{b^2\|\T\|^2}}.
\]
This shows that $T_1$ is a sub-Gaussian random vector\footnote{\url{http://lear.inrialpes.fr/people/harchaoui/teaching/2013-2014/ensl/m2/lecture6.pdf}}.
\paragraph{$\mathbf{T_3}$.} It is easy to see that $\mathbb{E}[T_3] = 0$. We now bound its MGF. Let $\T \in \mathbb{R}^p$ be any vector
\begin{equation*}
\begin{array}{lll}
     \E\left[e^{\left\langle \T, T_3\right\rangle}\right] &= &  \E\left[\mathbb{I}(|b + z\|\uu\| | < v) \left(e^{b\left\langle \T, \z_{\bot}\right\rangle} - 1\right) + 1\right]\vspace{0.1in}\\
     & =& 1 + \E\left[\mathbb{I}(|b + z\|\uu\| | < v) \right]\E\left[\left(e^{b\left\langle \T, \z_{\bot}\right\rangle} - 1\right)\right] \vspace{0.1in}\\
     &=& \left(1 - \E\left[\mathbb{I}(|b + z\|\uu\| | < v) \right]\right) + \E\left[\mathbb{I}(|b + z\|\uu\| | < v) \right]\E\left[e^{b\left\langle \T, \z_{\bot}\right\rangle}\right]
\end{array}
\end{equation*}
Since $\E\left[\mathbb{I}(|b + z\|\uu\| | < v) \right] < 1$, we have: 
\[
\E\left[e^{\left\langle \T, T_3\right\rangle}\right]  \leq \E\left[e^{b\left\langle \T, \z_{\bot}\right\rangle} \right] \leq  e^{\frac{b^2}{2}\left(\|\T\|^2 - \frac{\left\langle \T, \uu \right\rangle^2}{\|\uu\|^2}\right)} \leq e^{\frac{b^2\|\T\|^2}{2}}.
\]
%
This shows that $T_3$ is a sub-Gaussian random vector.

We now use the result that the sum of two dependent sub-Gaussian random vectors is also a sub-Gaussian random vector. As a result $T_1 + T_3$ is also a sub-Gaussian random vector, which satisfies the following for any $\T \in \mathbb{R}^p$: 	
\[
\E\left[e^{\left\langle \T, \A-\E[\A]\right\rangle}\right] \leq e^{\frac{cb^2\|\T\|^2}{2}},
\]
for some universal constant $c > 0$.
This shows that $\A$ is a sub-Gaussian random vector.

\subsection{Proof of Lemma~\ref{lem:concentration_T3_aux2}}
First note that $\E\left[e^{\left\langle \T, \A-\E[\A]\right\rangle }\right]$ can be written as: 
\[
\E\left[e^{\left\langle \T, \A-\E[\A]\right\rangle }\right] = \E_{v}\left[\E\left[e^{\left\langle \T, \A-\E[\A|v]\right\rangle }\Big| v\right]e^{\left\langle \T, \E[\A|v]-\E[\A]\right\rangle }\right].
\]
Using Lemma~\ref{lem:concentration_T3_aux1} to bound $\E\left[e^{\left\langle \T, \A-\E[\A|v]\right\rangle }\Big| v\right]$, we get: 
\begin{equation}
\label{eqn:aux_lem_t3_conc}
\E\left[e^{\left\langle \T, \A-\E[\A]\right\rangle }\right] \leq e^{\frac{cb^2\|\T\|^2}{2}} \E_{v}\left[e^{\left\langle \T, \E[\A|v]-\E[\A]\right\rangle }\right]
\end{equation}
We now bound the expectation in the RHS of the above equation.
From Lemma~\ref{lem:expectation_T3_aux1} we know that $\mathbb{E}\left[\A|v\right] = \frac{bq}{\sqrt{2\pi}} \frac{\uu}{\|\uu\|}$, for some random variable $q \in [-1,1]$.
This shows that $\left\langle \T, \E[\A|\eta]-\E[\A]\right\rangle$ is a bounded random variable which satisfies: 	
\[
|\left\langle \T, \E[\A|\eta]-\E[\A]\right\rangle| = \Big|\frac{\left\langle \T, \uu\right\rangle }{\|\uu\|}bq'\Big|,
\]
for some random variable $q' \in [-2,2]$. Since every  bounded random variable is sub-Gaussian we have: 
$\E_{v}\left[e^{\left\langle \T, \E[\A|v]-\E[\A]\right\rangle }\right] \leq e^{\frac{c_3b^2\|\T\|^2}{2}}.$ Substituting this in Equation~\eqref{eqn:aux_lem_t3_conc} we get: $\E\left[e^{\left\langle \T, \A-\E[\A]\right\rangle }\right] \leq e^{\frac{c_4b^2\|\T\|^2}{2}}.$ This shows that $\A$ is a sub-Gaussian random variable.

\subsection{Proof of Lemma~\ref{lem:concentration_T3_aux3}}
Define random variable $\A_i$ as: $\A_i =  \bb(i)  \mathbb{I}(|\bb(i) + \left\langle \tx_{i}, \uu \right\rangle | < v_i)\tx_{i}.$ 
Note that $\A = \sum_{i \in Q}\A_i$. Since $\{\A_i\}_{i \in Q}$ are independent sub-Gaussian random variable, $\A$ is also a sub-Gaussian random vector which satisfies the following inequality for any $\T \in \mathbb{R}^p$: 
\[
\E\left[e^{\left\langle \T, \A-\E[\A]\right\rangle}\right] \leq e^{\frac{c\|\bb_Q\|_2^2\|\T\|^2}{2}},
\]
for some universal constant $c$.
To get the tail bound for the norm of $\A$, we the following result from \citet{hsu2012tail}. Suppose $\A$ is a sub-Gaussian vector which satisfies the following for any $\T \in \mathbb{R}^p$: $\E\left[e^{\left\langle \T, \A - \mathbb{E}[\A]\right\rangle}\right] \leq e^{\|\T\|^2\sigma^2/2}.$ 
Then $\|\A\|$ satisfies the following tail bound for any $s > 0$: 
\[
\mathbb{P}\left(\|\A\|^2 \geq \sigma^2(p + 2\sqrt{ps} + 2s) + \|\E[\A]\|^2\left(1 + 4\sqrt{\frac{s}{p}} + \frac{4s}{p}\right)^{1/2}\right) \leq e^{-s}.
\]
Setting $s = p\log{\tn}$ and $\sigma = \|\bb_Q\|$ gives us the required bound.


\section{Some Concentration Results}
\label{sec:std_concentration}
In this section we state some well-known concentration results of Gaussian and sub-gaussian random variables.
\begin{lemma}[\citet{laurent2000adaptive}]
\label{lem:chi_conc}
Let $Y_1, \dots Y_m$ be $i.i.d$ standard normal random variables. Let $a_1 \dots a_m$ be non-negative. Let $Z = \sum_{i \in [m]}a_i\left(Y_i^2 - 1\right)$. Then, the following holds for any positive $t$: 
\[
\mathbb{P}(Z \geq 2\|a\|_2\sqrt{t} + 2\|a\|_{\infty}t) \leq e^{-t},\quad 
\mathbb{P}(Z \leq -2\|a\|_2\sqrt{t}) \leq e^{-t}.
\]
\end{lemma}
\begin{lemma}[\citet{hsu2012tail}]
\label{lem:sg_norm_conc}
Suppose $\A$ is a sub-Gaussian vector which satisfies the following for any $\T \in \mathbb{R}^p$: $\E\left[e^{\left\langle \T, \A - \mathbb{E}[\A]\right\rangle}\right] \leq e^{\|\T\|^2\sigma^2/2}.$ 
Then $\|\A\|$ satisfes the following tail bound for any $s > 0$: 
\[
\mathbb{P}\left(\|\A\|^2 \geq \sigma^2(p + 2\sqrt{ps} + 2s) + \|\E[\A]\|^2\left(1 + 4\sqrt{\frac{s}{p}} + \frac{4s}{p}\right)^{1/2}\right) \leq e^{-s}.
\]
\end{lemma}

\begin{lemma}
\label{lem:ss}
Let $X \in \mathbb{R}^{n \times p}$ be the matrix of covariates with columns sampled from $\mathcal{N}(0, I)$. Then, for any $\delta > 0$, with probability at least $1-\delta$, the following statements are true: 
\[
\lb_{max}\left(\frac{1}{n}X^TX\right) \leq 1 + \frac{1}{2}\sqrt{\frac{p}{n}} + \sqrt{\frac{\log{\frac{2}{\delta}}}{n}}, \quad 
\lb_{min}\left(\frac{1}{n}X^TX\right) \geq 1 - \frac{1}{2}\sqrt{\frac{p}{n}} - \sqrt{\frac{\log{\frac{2}{\delta}}}{n}}.
\]
\end{lemma}
\begin{lemma}[\citet{bhatia2015robust}]
\label{lem:sss}
Let $X \in \mathbb{R}^{n \times p}$ be the matrix of covariates with columns sampled from $\mathcal{N}(0, I)$. Then, for any $k > 0$, with probability at least $1-\delta$, the following statement is true: 
\[
\max_{|S| = k}\ \lb_{max}\left(X_S^TX_S\right) \leq k\left( 1 + 3e\sqrt{6\log{\frac{en}{k}}}\right) +  O\left(\sqrt{np + n\log{\frac{1}{\delta}}}\right),
\]
\[
\min_{|S| = k}\ \lb_{min}\left(X_S^TX_S\right) \geq n - (n-k)\left( 1 + 3e\sqrt{6\log{\frac{en}{n-k}}}\right) -  \Omega\left(\sqrt{np + n\log{\frac{1}{\delta}}}\right).
\]
\end{lemma}
\begin{lemma}[Proposition 2.1, \citet{vershynin2012close}]
\label{lem:cov_vershynin}
Consider independent random vectors $X_1, \dots X_n$ in $\mathbb{R}^p$, which have sub-gaussian distribution with parameter $L$: $\mathbb{P}\left(|\left\langle X, u\right\rangle| \geq t\right) \leq e^{-t^2/L^2}, \text{ for } t > 0, u \in \mathbb{S}^{p-1}.$ 
Then for any $\delta > 0$ with probability at least $1-\delta$ one has: 
\[
\Big|\Big|\frac{1}{n}\sum_{i = 1}^nX_iX_i^T - \mathbb{E}[X_iX_i^T]\Big|\Big|_2 \leq 4L\sqrt{\frac{p\log{\frac{2}{\delta}}}{n}}.
\]
\end{lemma}
\begin{lemma}[\citet{raskutti2010restricted, zhou2009restricted}]
\label{lem:re_tight}
Let $X\in\mathbb{R}^{n \times p}$ be a matrix with rows $\{\x_i\}_{i = 1}^n$ sampled from a Gaussian distribution with mean $0$ and covariance $\Sigma$. Then $X$ satisfies the following restricted eigenvalue (RE) property with probability at least $1-1/p^{10}$: 
\[
\forall \uu, \quad \frac{\|X\uu\|}{\sqrt{n}} \geq \|\Sigma^{1/2}\uu\|_2 - \sqrt{\frac{\rho(\Sigma)\log{p}}{n}}\|\uu\|_1,
\]
where $\rho(\Sigma) = \max_{i}\Sigma_{ii}$ is a constant that depends on $\Sigma$.
\end{lemma}


\section{Experiments}
\label{sec:exps}
\begin{figure}[t]
    \centering
    \includegraphics[width=0.245\textwidth]{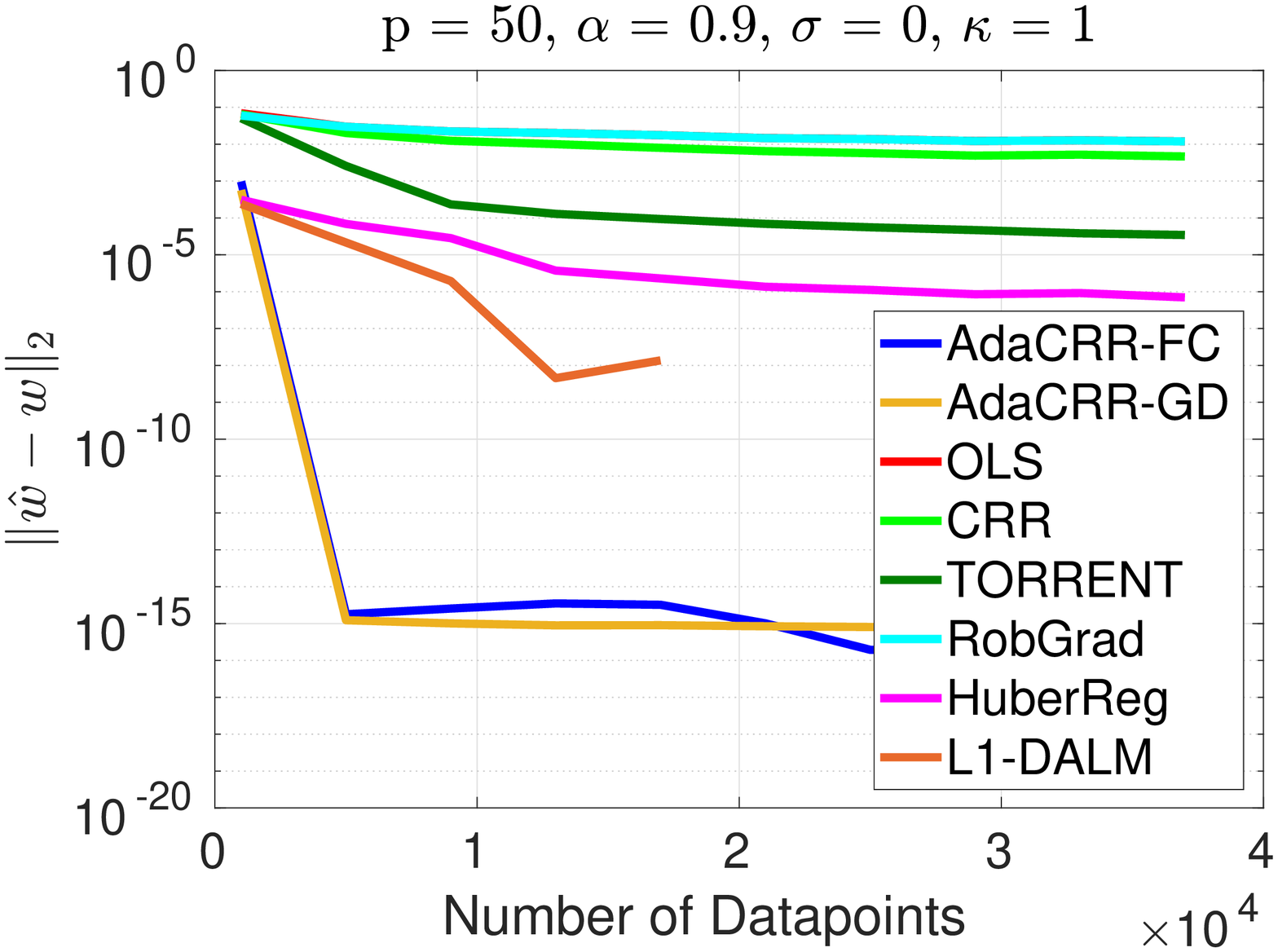}
    \includegraphics[width=0.245\textwidth]{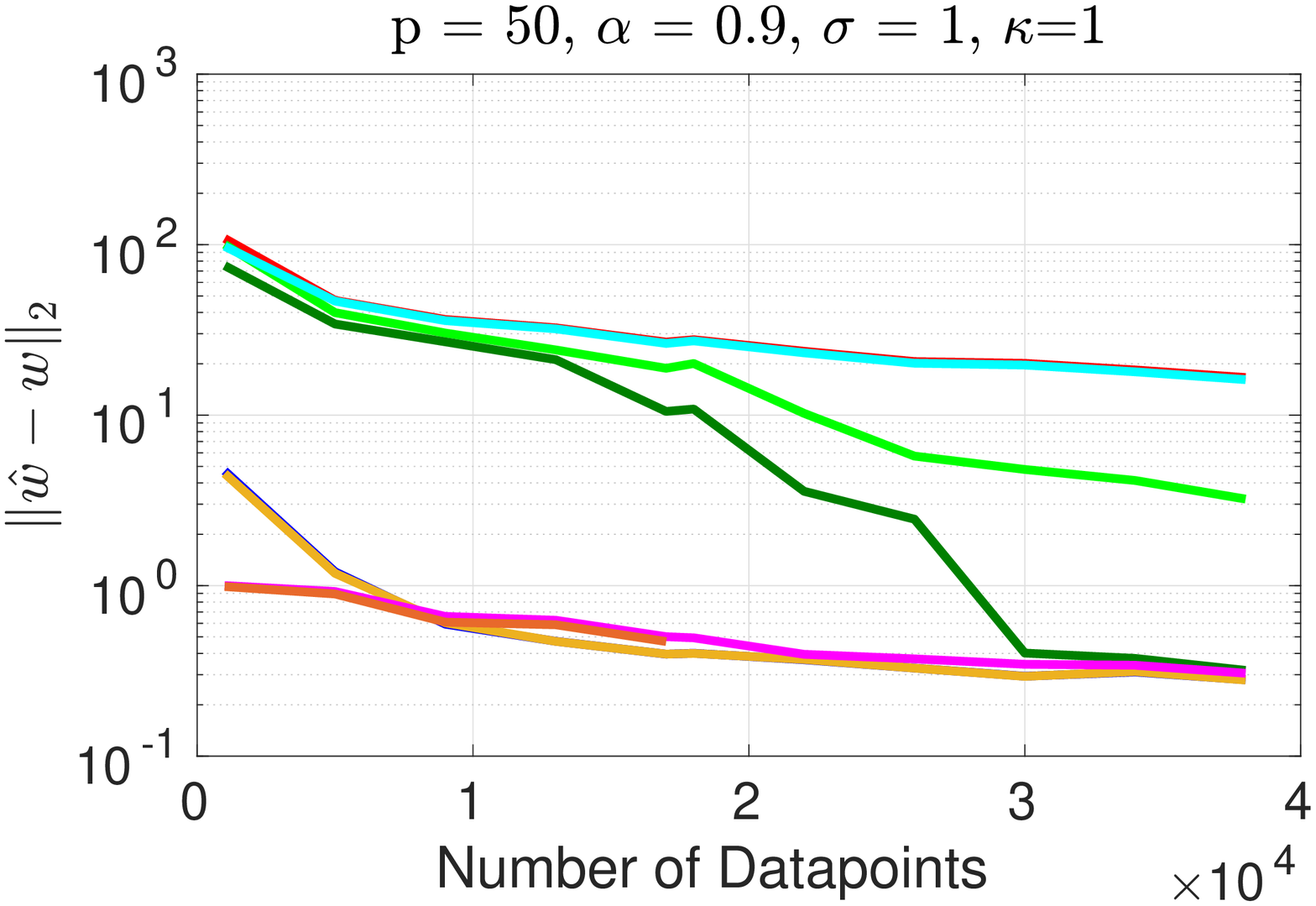}
    \includegraphics[width=0.245\textwidth]{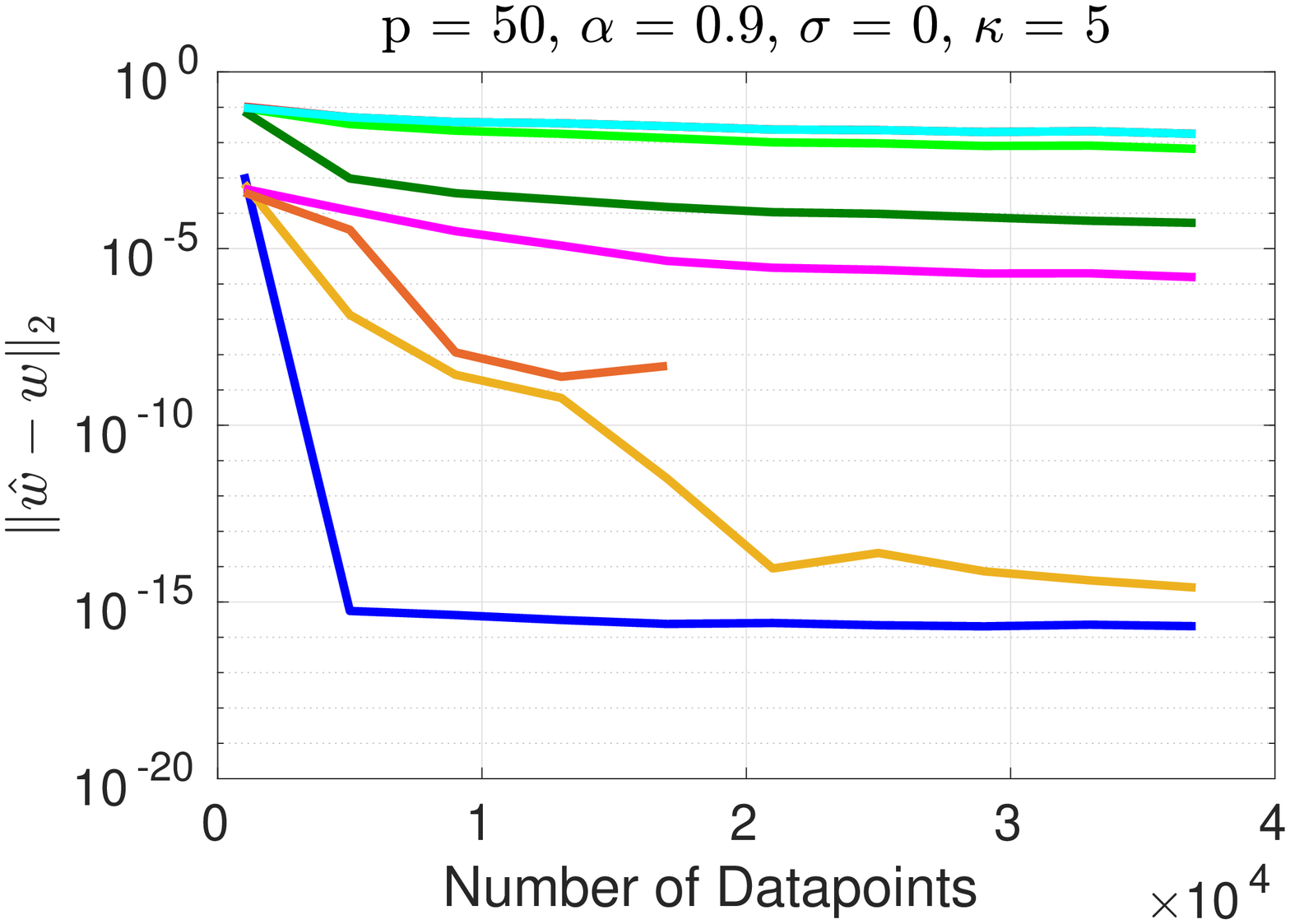}
    \includegraphics[width=0.245\textwidth]{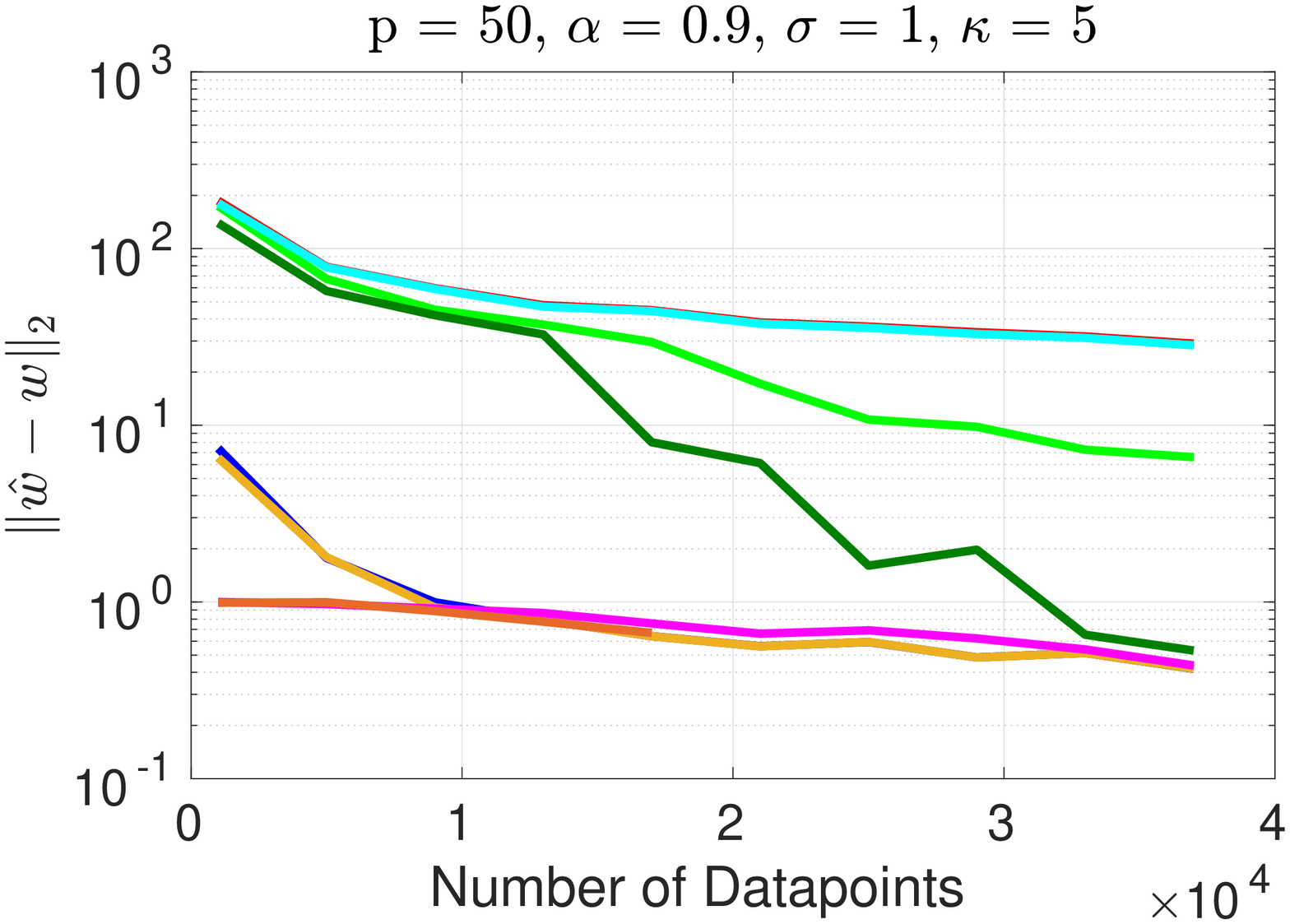}
    \caption{Performance of various estimators as we vary $n$ in the robust regression setup. In the noiseless setting $\sigma = 0$, \alg consistently outperforms baseline methods while in the setting with $\sigma =1$, Huber regression is able to compete with \alg in terms of recovery guarantees.}
    \label{fig:n_exp}
\end{figure}

\begin{figure}[t]
    \centering
     \includegraphics[width=0.245\textwidth]{plots/alpha_exp_noiseless.eps}
    \includegraphics[width=0.245\textwidth]{plots/alpha_exp_noisy.eps}
     \includegraphics[width=0.245\textwidth]{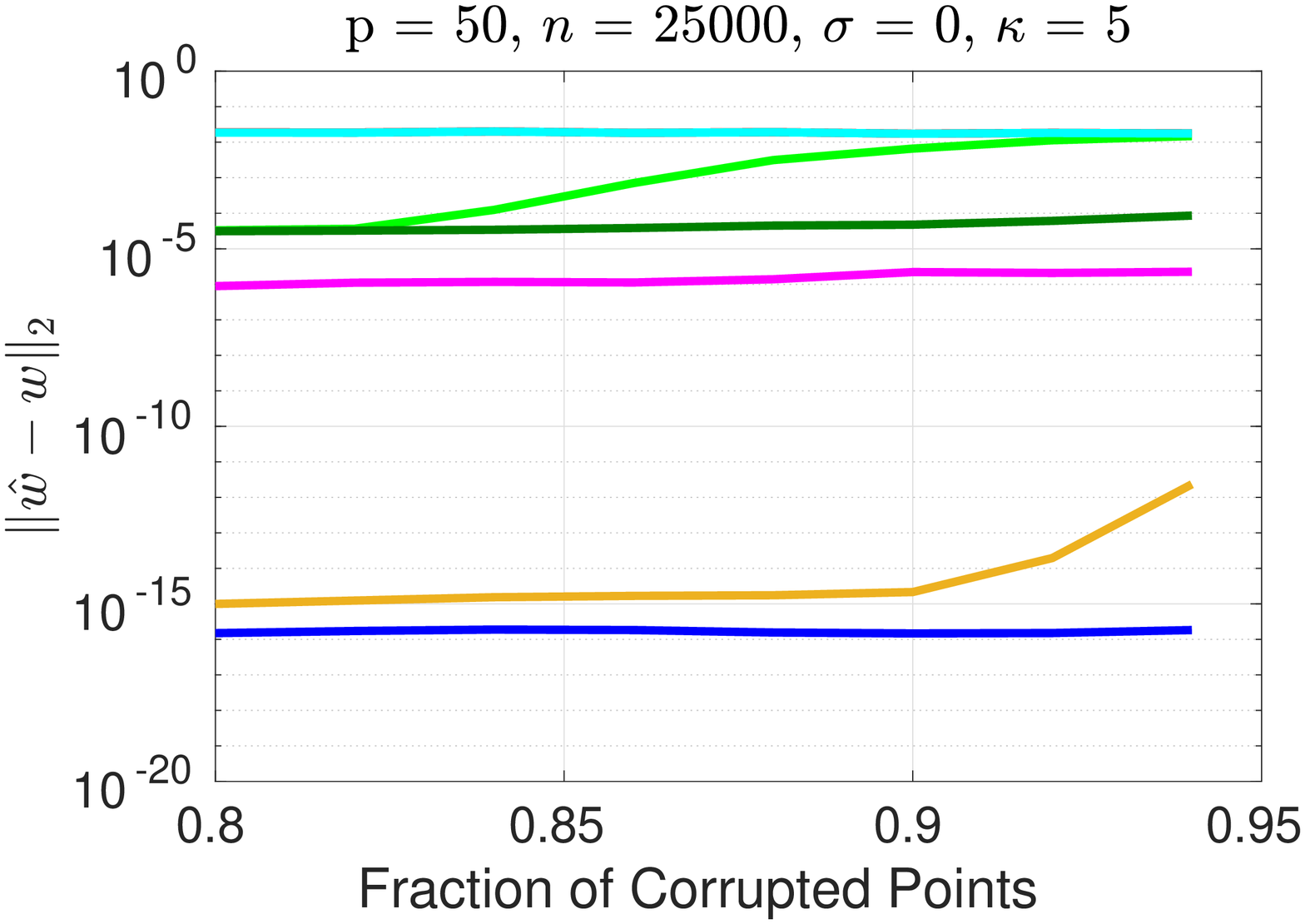}
    \includegraphics[width=0.245\textwidth]{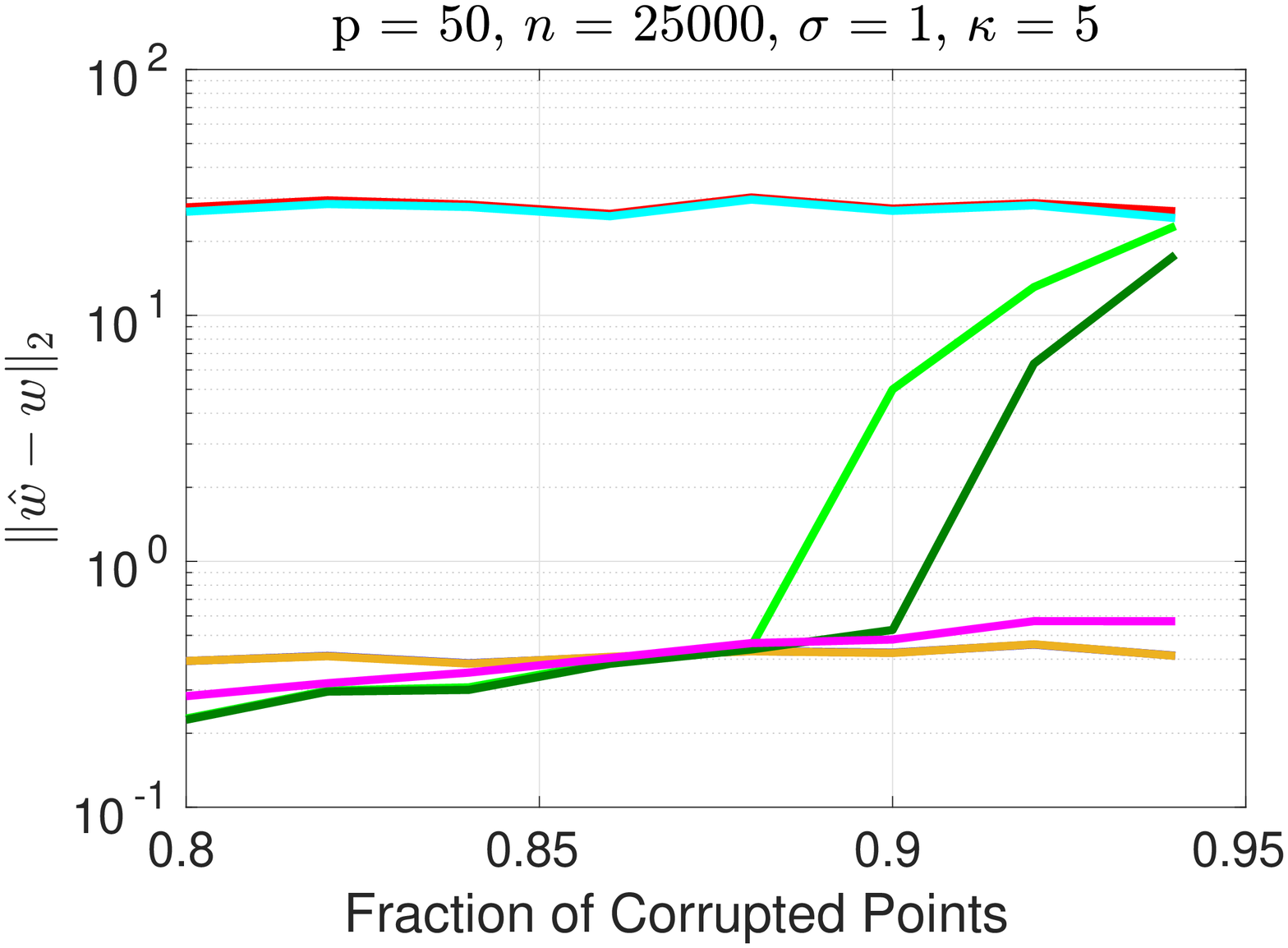}
    \caption{Performance of various estimators as we vary fraction of corruptions $\alpha$ in the robust regression setting. As stated in Theorem~\ref{thm:rr}, \alg is able to tolerate $1-o(1)$ fraction of outliers while most other competing methods tend to have a smaller breakdown point as compared to \alg.}
    \label{fig:alpha_exp}
\end{figure}
In this section, we compare the performance of \alg with baselines, for the problems of robust regression  as well as heavy-tailed regression.

\subsection{Robust Regression}\label{sec:exp_rr}
In this section, we focus on experiments related to the robust regression problem described in Section~\ref{sec:setup}. We begin by describing the problem setup as well as the baseline algorithms before proceeding to a discussion of the experimental findings.

\paragraph{Experiment Setup.} For all the experiments, the feature vectors $\x \sim \mathcal{N}(0, \Sigma)$. For experiments with condition number $\kappa(\Sigma) = 1$, $\Sigma$ was set to identity; whenever the condition number $\kappa(\Sigma) > 1$, $\Sigma$ was chosen to be a diagonal matrix with entries uniformly sampled from $(0, 1)$ and rescaled to ensure that the maximum and the minimum entry are $1$ and $1/\kappa$ respectively. Further for any value of $\alpha > 0.5$, the corruption vector when $\sigma > 0$ was set as follows: a random set of $n/4$ corruptions were set to $1000$, another $n/4$ were set to $\sqrt{1000}$ and the remaining were uniformly sampled from $(0, 10)$. When $\sigma = 0$, we set the corruption vector as follows: a random set of $n/4$ corruptions were set to $1$, another $n/4$ were set to $1/\sqrt{n}$ and the remaining were set to $1/n$. Each experiment was repeated for 5 runs and we report the mean performance of the methods across these runs.

\paragraph{\alg.} We use gradient descent to solve the least squares optimization problem that arises in the update step of $\w$ in \alg-FC, and take $5$ gradient descent steps. Since the least squares loss is strongly convex and smooth, a small number of GD steps is enough to take us close to the optimal solution.

\paragraph{Baseline Algorithms.} We compare the two variants of our proposed estimator \alg-FC and \alg-GD with various baseline alogrithms including a) Ordinary Least Squares (OLS) estimator, b) \tor estimator \citep{bhatia2015robust}, c) CRR estimator \citep{bhatia2017consistent}, d) Robust Gradient estimator (\rob) \citep{prasad2018robust} using the mean estimator of \citet{lai2016agnostic} to robustly estimate the gradients, e) Huber regression and f) $\ell_1$-regression (L1-DALM). For \tor, CRR and \rob, we use the implementations provided by the authors. For Huber Regression we compared the matlab in-built implementation, which uses an iteratively rewieghted least squares solver (IRLS), with the library by \citet{msmatlab}, which has various solvers for Huber Regression (\emph{e.g.,} sub-gradient solver, L-BFGS solver), and found the L-BFGS solver to be more stable in its recovery properties than the other techniques. For $\ell_1$ regression, we use the DALM solver by \cite{yang2013fast}; for a detailed evaluation of various solvers $\ell_1$-regression  we refer the reader to \cite{bhatia2015robust}. The hyperparameters of the the baseline algorithms were tuned using a \emph{fine grid-search} while those for \alg were \emph{fixed} to a default setting: a) the hyper-parameters $a, \gamma$ in \adaht are set  as $1/18$, $4$ respectively, b) $\beta$ in the interval length computation is set to $0.98$, d) $\hat{d}_0$ - the upper bound of $\|\dw_0\|_2$ - is estimated using the techniques described in Appendix~\ref{sec:aux_additional}, c) the upper bound of $\sigma$ is set as $\hat{\sigma} = 2\sigma$,  and d) the interval length was computed as $\hat{I}_t = 3\sqrt{2\hat{\sigma}^2 + 2\beta^{2(t-2)}\hat{d}_0^2}$. Note that this expression differs from the expression in Equation~\eqref{eqn:dist_ub} in the leading constants. This is because the expression in Equation~\eqref{eqn:dist_ub} is an optimistic estimate of the ``optimal'' interval length, which ensures none of the un-corrupted points are left out. In practice, one can use smaller interval lengths and still ensure this property holds.

\paragraph{Recovery guarantees.} Figures~\ref{fig:n_exp} and~\ref{fig:alpha_exp} show the performance of various estimators for the above designed setup where we measure the performance in terms of the estimation error $\|\hat{w}-w^* \|_2$.\\
Figure~\ref{fig:n_exp} exhibits this performance as we vary the total number of datapoints, $n$, keeping other parameters fixed. a) In the noiseless setting, $\sigma = 0$, we obser that both \alg-FC and \alg-GD are consistently able to achieve much lower error rate as compared to other baseline methods, even when they were run to convergence. The L1-DALM method could not be scaled to more than 20000 datapoints since it required the computation of a $n\times n$ matrix which lead to memory errors. b) In the noisy setting, $\sigma = 1$, the recovery error of Huber regression method was competitive with those of \alg while all other methods typically had worse performance in comparison. \\
Figure~\ref{fig:alpha_exp} shows the recover error with variation in the level of corruption $\alpha$, keeping other parameters fixed. Both in the noisy and noiseless settings, \alg is able to tolerate much higher levels of corruption as compared to other baselines, with huber regression becoming competitive in the setting with $\sigma=1$. This is in accordance with the theoretical guarantees of \alg which show that it can tolerate a level of corruption $1-o(1)$.

\begin{figure}[t]
    \centering
    \includegraphics[scale=0.22]{plots/time_noisy.eps}
    \includegraphics[scale=0.22]{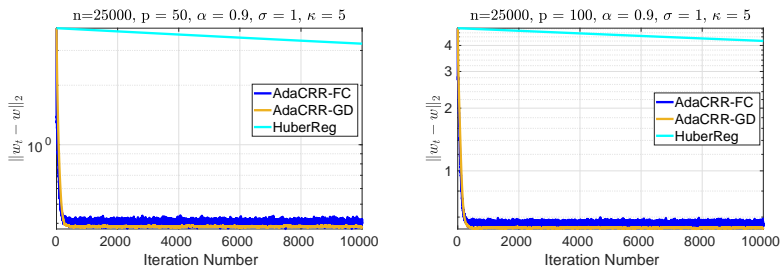}
    \caption{Error vs. iteration number for robust regression. Corruptions are sampled uniformly from $[0, 100]$. Both the algorithms are intialized at the same point. \alg can be upto 100x faster as compared to Huber regression while achieving much better recovery errors. }
    \label{fig:time}
\end{figure}

\paragraph{Computational Efficiency.} Figure~\ref{fig:time} compares the computational efficiency of \alg with Huber regression, which had competitive recovery error in a few settings. It shows the variation in recovery error $\|\hat{w} -w^* \|_2$ with the number of iteration of the algorithm for $d=50,\ 100$. Each iteration of \alg-GD is a simple matrix-vector multiplication with cost $O(nd)$ which is atleast that incurred by the LBFGS solver for huber regression. \alg-GD and \alg-FC can be upto 100x faster as compared to stable solvers for huber regression and are in accordance with the theoretically proven linear convergence for the optimization procedure.

\begin{figure}[t]
    \centering
    \includegraphics[scale=0.2]{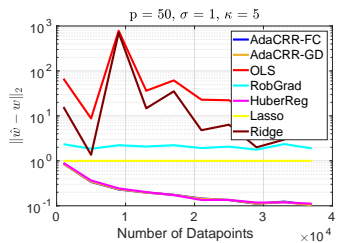}
    \caption{Performance of various estimators on linear regression with cauchy noise. Both \alg and Huber regression are able to offer consistent estimators for this regression problem where the noise variables do not have a well-defined first moment. As before, \alg exhibits linear convergence to the optima and can be orders of magnitude faster than Huber regression.}
    \label{fig:cauchy_noise}
\end{figure}

\subsection{Heavy-tailed Regression}\label{sec:exp_ht}
In this section, we focus on experiments related to the heavy-tailed regression problem described in Section~\ref{sec:setup} with an emphasis on Cauchy regression, wherein the noise variables $\epsilon_i$ are sampled i.i.d. from a Cauchy distribution.

\paragraph{Experiment Setup.} For all the experiments, the feature vectors $x \sim \mathcal{N}(0, \Sigma)$ where the $\Sigma$ matrix was set similar to the robust regression experiments in Section~\ref{sec:exp_rr}. We set the  scale parameter ($\sigma$) of the Cauchy distribution to $ 1$. Each experiment was repeated for 5 runs and we report the mean performance of the methods across these runs.

\paragraph{Baseline Algorithms.} We compare the two variants of our proposed estimator \alg-FC and \alg-GD with various baseline alogrithms including a) Ordinary Least Squares (OLS) estimator, b) Robust Gradient estimator (\rob) \citep{prasad2018robust} which uses the median-of-means estimator for estimating mean gradient, c) Huber regression d) Sparse regression (\textsc{Lasso})~\citep{tibshirani1996regression}, e) Ridge Regression and f) $\ell_1$-regression (L1-DALM). For \rob, we use the implementations provided by the authors. The hyperparameters of the the baseline algorithms were tuned using a fine grid-search while those for \alg were fixed to a default setting described above, with $\rho = 0.3$.

\paragraph{Recovery guarantees.} Figure~\ref{fig:cauchy_noise} shows the variation of parameter recovery error $\|\hat{w}-w^*\|_2$ as the number of data points $n$ are varied for the above setup. \alg-FC, \alg-GD and the huber regression estimator are able to achieve consistent estimation of the parameter vector with the error exhibiting a decreasing trend with increasing datapoints even when the Cauchy noise does not have a well-defined first moment. Other baseline estimators do not exhibit this consistency and tend to flatten out in terms of recovery error. The high fluctuations in the OLS and Ridge estimators can be attributed to the heavy-tailed Cauchy distribution.

\end{document}